\documentclass[11pt,a4paper]{article}
\usepackage{authblk}
\usepackage{graphicx}
\usepackage{amsmath}
\usepackage{amssymb}
\usepackage{algorithm}
\usepackage{algorithmic}
\usepackage{amsthm}
\usepackage{epstopdf}
\usepackage{caption}
\usepackage{url}
\usepackage{subcaption}
\usepackage{natbib}
\usepackage{cases}
\newtheorem{theorem}{Theorem}
\newtheorem{lemma}{Lemma}

\newtheorem{corollary}{Corollary}

\theoremstyle{definition}

\newtheorem{definition}{Definition}
\title{\bf Linear Convergence of SVRG　 in Statistical Estimation }
\date{}
\author[1]{Chao Qu}
\author[2]{Yan Li}
\author[2]{Huan Xu}
\affil[1]{Department of Mechanical Engineering, National University of Singapore}
\affil[2]{H. Milton Stewart School of Industrial and Systems Engineering, Georgia Institute of Technology}
\begin{document}

\maketitle

\begin{abstract}
The last several years has witness the huge success on the stochastic variance reduction method  in the finite sum problem.  However assumption on strong convexity to have linear rate limits its applicability. In particular, it does not include several important formulations such as Lasso, group Lasso, logistic regression,  and some non-convex models including corrected Lasso and SCAD. In this paper, we prove that, for a class of statistical M-estimators covering examples  mentioned above,  SVRG solves the formulation with {\em a linear convergence rate}  without strong convexity or even convexity.
Our analysis makes use of {\em restricted strong convexity}, under which we show that SVRG converges linearly to   the fundamental statistical precision of the model, i.e., the difference between true unknown parameter $\theta^*$ and the optimal solution $\hat{\theta}$ of the model.
\end{abstract}

\section{Introduction}
In this paper we establish fast convergence rate of stochastic variance reduction gradient (SVRG)  for a class of  problems motivated by applications in high dimensional statistics where {\em the problems are not strongly convex, or even non-convex}. 	
High-dimensional statistics has achieved remarkable success  in the last decade,
including   results on consistency and rates for various estimator under non-asymptotic high-dimensional scaling, especially when the problem dimension $p$ is larger than the number of data $n$ \citep[e.g.,][and many others \citep{candes2006stable,wainwright2006sharp,chen2011robust}]{negahban2009unified, candes2009exact}  . It is now well known  that while this setup appears   ill-posed, the estimation or recovery is indeed possible by exploiting the {\em underlying structure} of the parameter space~--~notable examples include sparse vectors, low-rank matrices,  and structured regression functions, among others.  Recently, estimators leading to non-convex optimizations have gained fast growing attention. Not only it typically has better statistical properties in the high dimensional regime, but also in contrast to common belief, under many cases there exist efficient algorithms that provably find   near-optimal solutions \cite{loh2011high,zhang2012general,loh2013regularized} .

Computation challenges of statistical estimators and machine learning algorithms have been an active area of study, thanks to countless applications involving {\em big data}~--~datasets where both $p$ and $n$ are large. In particular, there are renewed interests in first order methods to solves the following class of optimization problems:
\begin{equation}\label{obj1}
\begin{split}
\underset{\theta\in \Omega}{\mbox{Minimize:}}\quad G(\theta) &\triangleq F(\theta)+\lambda\psi(\theta)\\
&=\frac{1}{n}\sum_{i=1}^{n} f_i(\theta)+\lambda\psi(\theta).
\end{split}
\end{equation}
Problem~\eqref{obj1} naturally arises in statistics and machine learning. In supervised learning, we are
given a sample of $n$ training data $ (x_1,y_1), (x_2,y_2),..., (x_n,y_n)$, and thus  $f_i (\theta)$ is the corresponding loss, e.g., the squared loss $f_i(\theta)=(y_i-\theta^Tx_i)^2$, $\Omega$ is a convex set corresponding to the class of hypothesis, and $\psi(\theta)$ is the (possibly non-convex) regularization. 
Many widely applied statistical formulations are examples of Problem~\eqref{obj1}. A partial list includes:
\begin{itemize}  
	\item Lasso: $f_i(\theta)=\frac{1}{2} (\langle \theta, x_i \rangle-y_i)^2$ and $\psi(\theta)=\|\theta\|_1$.
	\item Group Lasso: $f_i(\theta)=\frac{1}{2} (\langle \theta, x_i \rangle-y_i)^2$,  $\psi(\theta)=\|\theta\|_{1,2}$.
	\item Logistic Regression with $l_1$ regularization: $f_i(\theta)=\log (1+\exp(-y_i \langle x_i,\theta \rangle)$ and $\psi(\theta)=\|\theta\|_1$.
	\item Corrected Lasso \cite{loh2011high}: $G(\theta)=\sum_{i=1}^{n}\frac{1}{2n} (\langle \theta, x_i \rangle-y_i)^2-\frac{1}{2}\theta^T\Sigma \theta+\lambda \|\theta\|_1,$ where $\Sigma$ is some positive definite matrix.
	\item Regression with SCAD  regularizer \cite{fan2001variable}:
	$G(\theta)=\sum_{i=1}^{n}\frac{1}{2n} (\langle \theta, x_i \rangle-y_i)^2+SCAD(\theta)$.
\end{itemize}
In the first three examples, the 
objective functions $G(\theta)$ are {\em not strongly convex} when $p> n$. Example 4 is {\em non-convex} when $p>n$, and the last example is  {\em non-convex} due to the SCAD regularizer.

Projected gradient method, proximal gradient method,  dual averaging method \cite{nesterov2009primal} and several variants of them have been proposed to solve Problem~\eqref{obj1}.  However, at each step, these  batched gradient descent methods  need to evaluate all $n$ derivatives, corresponding to each $f_i(\cdot)$, which can be expensive for large $n$.  Accordingly,   stochastic gradient descent (SGD) methods have gained attention, because of its significantly lighter computation load for each iteration: at iteration $t$, only one data point~--~sampled from $\{1,...,n\}$ and indexed by $i_t$~--~is used to update the parameter $\theta$ according to
$$\theta_{t+1}:=\theta_t-\beta\nabla f_{i_t} (\theta_{t}),$$
or its proximal counterpart (w.r.t.\ the regularization function $\psi(\cdot)$)
$$\theta_{t+1}:=Prox_{\beta\lambda\psi}(\theta_t-\beta\nabla f_{i_t} (\theta_{t})).$$

Although the computational cost in each step is low, SGD often suffers from slow convergence, i.e., sub-linear convergence rate even with strong assumptions (strong convexity and smoothness).  Recently, one state-of-art technique to improve the convergence of SGD called the \textit{variance-reduction-gradient} has been proposed \cite{johnson2013accelerating,xiao2014proximal}. As the name suggests,  it devises a better unbiased estimator of stochastic gradient $v_t$ such that the variance $ E \|v_t-\nabla F(\theta_t)\|_2 $ diminishes when $t\rightarrow \infty$. In particular, in SVRG and its variants, the algorithm keeps a snapshot $\tilde{\theta}$ after every $m$ SGD iterations and calculate the full gradient $F(\tilde{\theta}) $  just for this snapshot, then the variance reduced gradient is computed by
$$  v_t=\nabla f_{i_t}(\theta_t)-\nabla f_{i_t}(\tilde{\theta})+\nabla F(\tilde{\theta}). $$
It is shown in \citet{johnson2013accelerating} that when $G(\theta)$ is strongly convex and $F(\theta)$ is smooth, SVRG and its variants  enjoy  linear convergence, i.e.,  $O(\log(1/\epsilon))$ steps suffices to obtain  an $\epsilon$-optimal solution.  Equivalently, the gradient complexity (i.e., the number of gradient evaluation needed) is $O( (n+\frac{L}{\mu})\log (1/\epsilon))$, where $L$ is the smoothness of $F(\theta)$ and $\mu$ is the strong convexity of $F(\theta)$.

What if $G$ is not strongly convex or even not convex? As we discussed above, many popular machine learning models belongs to this. When $G$ is not strongly convex,   existing theory only  guarantees that SVRG will converge {\em sub-linearly}.
A folklore method is to add a dummy strongly convex term $ \frac{\sigma}{2}\|\theta\|_2^2$  to the objective function and then apply the algorithm \cite{shalev2014accelerated,allen2015improved}. This  undermines the performance of the model, particularly its ability to recover {\em sparse} solutions. One may attempt to reduce $\sigma$ to zero in the hope of reproducing the optimal solution of the original formulation, but the convergence will still be sub-linear via this approach. As for the non-convex case , to the best of our knowledge, no work provides linear convergence guarantees for the above mentioned examples using SVRG. 
 \subsubsection*{Contribution of the paper} 
 We show that for a class of problems, SVRG achieves linear convergence  \textit{without} strong convexity or  convexity assumption. In particular we prove the gradient complexity of SVRG is $O ( (n+\frac{L}{\bar{\sigma}}) \log\frac{1}{\epsilon})$  when $\epsilon$ is larger than the statistical tolerance, where $\bar{\sigma}$ is the modified restricted stronly convex parameter defined in Theorem \ref{main_theorem} and Theorem \ref{main_theorem_non_convex}. Notice If we replace modified restricted stronly convex parameter by the strong convexity, above result becomes standard result of SVRG.   Indeed, in the proof, our effort is to replace strong convexity by Restricted strong convexity.  Our analysis is general and covers many formulations of interest, including all examples mentioned above.  Notice  that RSC is known to hold with high probability for a broad class of statistical models including sparse linear model,  group sparsity model  and low rank matrix model. Further more, the batched gradient method with RSC assumption by \cite{loh2013regularized} has the gradient complexity $ O\big( nL/\bar{\sigma}\log\frac{1}{\epsilon}   \big) $ ($\epsilon>statistical~tolerance $). Thus our result is  better than the batched one, especially when the problem is ill-conditioned ($L/\bar{\sigma} \gg 1$).

  We also remark that while we present analsyis  for the vanilla SVRG \cite{xiao2014proximal} , the analysis for   variants of SVRG \cite{nitanda2014stochastic,harikandeh2015stopwasting,nitanda2014stochastic} is similar and indeed such extension is straightforward.
 \subsubsection*{Related work}
 There is a line of work establishing fast convergence rate without strong convexity assumptions for {\em batch gradient methods}:
 \cite{xiao2013proximal} proposed a homotopy method to solve Lasso with  RIP condition.      \citet{agarwal2010fast} analyzed the convergence rate of batched composite gradient method on several models, such as Lasso, logistic regression with $\ell_1$ regularization and noisy matrix decomposition,  showed that the convergence is linear under mild condition (sparse or low rank). \citet{loh2011high,loh2013regularized} extended above work to the non-convex case. Conceptually, our work can be thought as the stochastic counterpart of it, albeit with more involved analysis due to the stochastic nature of SVRG.

 In general, when the function is not strongly convex,  stochastic variance-reduction type method has shown to converge with a sub-linear rate: SVRG\cite{johnson2013accelerating}, SAG \cite{mairal2013optimization}, MISO \cite{mairal2015incremental}, and SAGA \cite{defazio2014saga} are shown to converge with gradient complexity for non-strongly convex functions with a  {\em sub-linear rate} of $ O(\frac{n+L}{\epsilon})$.  \citet{allen2015improved} propose $SVRG^{++}$ which solves the non-strongly convex problem with gradient complexity $O(n\log \frac{1}{\epsilon}+\frac{L}{\epsilon}) $.  \citet{shalev2016sdca} analyzed SDCA~--~another stochastic gradient type algorithm with variance reduction~--~and established similar results. He allowed each $f_i(\theta)$ to be non-convex  but  $F(\theta)$ needs to be strongly convex for linear convergence to hold.
 Neither work establishes  linear convergence of the above mentioned examples, especially when $G(\theta)$ is non-convex.
 
 Recently, several papers revisit an old idea called Polyak-Lojasiewica inequality and use it to replace the strongly convex assumption \cite{karimi2016linear,reddi2016fast,gong2014linear}, to establish fast rates. They established  linear convergence of SVRG without strong convexity for Lasso and Logistic regression. The contributions of our work  differs  from theirs in two aspects. First, the linear convergence rate they established does not depend on sparsity $r$, which does not agree with the empirical observation. We report simulation results  on solving Lasso using SVRG  in the Appendix, which shows  a phase transition on rate: when $\theta^*$ is dense enough, the rate becomes sub-linear. A careful analysis of their result shows that that  the convergence result using P-L inequality depends on a so-called {\em Hoffman parameter}. Unfortunately it is not clear how to characterize or bound the Hoffman parameter, although from the simulation results it is conceivable that such parameter must correlated with the sparsity level. In contrast, our results   state that the algorithm converges faster with sparser $\theta^*$ and a phase transition happens when   $\theta^*$ is dense enough, which clearly fits better with the empirical observation. Second,  their results require the epigraph of $\psi(\theta)$ to be a polyhedral set, thus are not applicable to popular models such as group Lasso.

 \citet{li2016stochastic} consider the sparse linear problem with $\ell_0$ ``norm'' constraint and solve it using {\em stochastic variance reduced gradient hard thresholding} algorithm (SVR-GHT), where the proof also uses the idea of RSC.  In contrast,
 we establish a unified framework that  provides more general result which covers not only sparse linear regression, but also group sparsity, corrupted data model (corrected Lasso), SCAD we mentioned above but not limited to these.
 \section{Problem Setup and Notations}
 
 In this paper, we consider two setups, namely the convex but not strongly convex case, and the non-convex case. For the first one we consider the following form:
 \begin{equation}\label{obj2}
 \begin{split}
 \hat{\theta} &=\arg\min_{\psi(\theta)\leq \rho} G(\theta)\quad   \mbox{where} \\
 &G(\theta) \triangleq       F(\theta)+    \lambda\psi(\theta)           =  \frac{1}{n}\sum_{i=1}^{n} f_i(\theta)+\lambda\psi(\theta),
 \end{split}
 \end{equation}
 where $\rho>0$ is a pre-defined radius, and the regularization function $\psi(\cdot)$ is a norm. The functions $f_i(\cdot)$, and consequently  $G(\cdot)$, are convex. Yet, neither $f_i(\cdot)$ nor  $G(\cdot)$ are necessarily strongly convex.  We remark that   the side-constraint $\psi(\theta)\leq\rho$ in (\ref{obj2}) is included without loss of generality: it is easy to see that for the unconstrained case, the optimal solution $\hat{\theta}'$ satisfies $\psi(\hat{\theta}')\leq \rho'\triangleq ( \sum_{i=1}^n f_i(0) - nK)/n\lambda$, where $K\in \mathbb{R}$ lower bounds $f_i(\theta)$ for all $i$.

 For the second case we consider the following non-convex estimator.
 \begin{equation}\label{obj3}
 \begin{split}
 \hat{\theta} &=\arg\min_{g_\lambda(\theta)\leq \rho} G(\theta) \quad\mbox{where} \\
 &  G(\theta) \triangleq F(\theta)+ g_{\lambda,\mu}(\theta) =   \frac{1}{n}\sum_{i=1}^{n} f_i(\theta)+g_{\lambda,\mu} (\theta),
 \end{split}
 \end{equation}
 where $f_i(\cdot)$ is convex, $g_{\lambda,\mu}（(\cdot)$ is a non-convex regularizer depending on a tuning parameter $\lambda$ and a parameter $\mu$ explained in section \ref{section:nonconvex_regularizer}. This M-estimator also includes a side constraint depending on $g_\lambda(\theta) $ , which needed to be a convex function and have a lower bound $g_\lambda(\theta)\geq \|\theta\|_1 $.  This $g_\lambda (\theta)$ is close related to $g_{\lambda,\mu} (\theta)$, for more details we defer to section \ref{section:nonconvex_regularizer}. Similarly as the first case, the side constraint is added without loss of generality.
 
 \subsection{RSC}
 A central concept we use in this paper is {\em Restricted strong convexity} (RSC), initially proposed in \citet{negahban2009unified} and explored in  \citet{agarwal2010fast,loh2013regularized}.
 A  function $F(\theta)$ satisfies restricted strong convexity with respect to $ \mathcal{\psi} $ and with parameter $(\sigma,\tau_{\sigma})$ over the set $\Omega$ if for all $\theta_2,\theta_1 \in \Omega,$
 \begin{equation}\label{RSC}
 \begin{split}
 & F(\theta_2)-F(\theta_1)-\langle \nabla F(\theta_1),\theta_2-\theta_1 \rangle\\
 \geq &\frac{\sigma}{2} \|\theta_2-\theta_1\|_2^2-\tau_{\sigma}  \mathcal{\psi}^2 (\theta_2-\theta_1),
 \end{split}
 \end{equation}
 where the second term on the right hand side  is called the tolerance, which essentially measures how far $F(\cdot)$ deviates from being strongly convex. Clearly, when $ \tau_{\sigma}=0$, the RSC condition reduces to strong convexity. However, strong convexity can be restrictive in some cases. For example, it is well known that strong convexity does not hold for Lasso or logistic regression in the {\em high-dimensional regime} where the dimension $p$  is larger than the number of data $n$. In contrast, in many of such problems, RSC holds with relatively small tolerance.  Recall $F(\cdot)$ is convex,     which implies $ \frac{\mathcal{\psi}^2 (\theta_2-\theta_1) }{\|\theta_2-\theta_1\|_2^2}\leq \frac{\sigma}{2\tau_{\sigma} }.$ We remark that in our analysis, we only require RSC to  hold for $F(\theta)=\frac{1}{n}\sum_{i=1}^{n} f_i(\theta)$, rather than on individual loss functions $f_i(\theta)$. This agrees with the case in practices, where  RSC does not hold on $f_i(\theta)$ in general.
 
 \subsection{Assumptions on $\psi(\theta)$}
 RSC is a useful property because for many formulations, the tolerance is small {\em along some directions}. To this end, we need the concept of {\em decomposable regularizers}. 
 Given a pair of subspaces $ M\subseteq \bar{M} $ in $ \mathbb{R}^p$, the orthogonal complement of $\bar{M}$ is 
 $$ \bar{M}^{\perp}=\{ v\in \mathbb{R}^p | \langle u,v \rangle=0 \text{~for all~} u\in \bar{M} \}.$$ $M$ is known as the model subspace, where $\bar{M}^{\perp}$ is called the {\em perturbation subspace}, representing the deviation from the model subspace. A  regularizer $\psi $ is {\em decomposable} w.r.t.\ $(M,\bar{M}^{\perp})$ if 
 $$ \psi(\theta+\beta)=\psi(\theta)+\psi(\beta) $$ 
 for all $ \theta\in M $ and $\beta\in \bar{M}^{\perp}.$
 Given the regularizer $\psi(\cdot)$, the subspace compatibility $H(\bar{M})$ is given by
 $$ H(\bar{M})=\sup_{\theta\in \bar{M} \backslash \{0\}} \frac{\psi(\theta)}{\|\theta\|_2}.$$
 For more discussions and intuitions on decomposable regularizer, we refer reader to \citet{negahban2009unified}. 
 Some examples of decomposable regularizers are in order.
 
 \subsubsection*{$\ell_1$ norm regularization }
 $\ell_1$ norm are widely used as a regularizer  to encourage sparse solutions. As such,  the subspace $M$ is chosen according to  the $ r$-sparse vector in $p$ dimension space. Specifically,  given a subset $S\subset \{1,2,...,p\}$ with cardinality $r$, we let
 $$ M(S):=\{ \theta\in \mathbb{R}^p | \theta_j=0 \text{~for all ~} j\notin S   \}. $$ 
 In this case,  we let $ \bar{M}(S)=M(S)$ and it is easy to see that 
 $$\|\theta+\beta\|_1=\|\theta\|_1+\|\beta\|_1,\quad \forall \theta\in M,\,\beta\in \bar{M}^{\bot},$$ 
 which implies that $\|\cdot\|_1$ is decomposable with $M(S)$ and $\bar{M}(S)$.
 \subsubsection*{Group sparsity regularization }
 Group sparsity extends the concept of sparsity, and has found a wide variety of   applications \cite{yuan2006model}. For simplicity, we consider the case of non-overlapping groups.
 Suppose all features are grouped into disjoint blocks, say, $\mathcal{G}=\{G_1,G_2,...,G_{N_{\mathcal{G}}}\}$. The $(1,\bar{\gamma})$ grouped norm is defined as
 $$ \|\theta\|_{\mathcal{G},\bar{\gamma}}=\sum_{i=1}^{N_{\mathcal{G}}}\|\theta_{G_i}\|_{\gamma_i}, $$
 where $\bar{\gamma}=(\gamma_1,\gamma_2,...,\gamma_{N_{\mathcal{G}}})$. Notice that  group Lasso is thus a special case where $\bar{\gamma}=(2,2,...,2)$. Since blocks are disjoint, we can define the subspace in the following way. For a subset $S_{\mathcal{G}}\subset \{1,..,N_{\mathcal{G}} \} $ with cardinality $s_{\mathcal{G}}=|S_{\mathcal{G}}|$, we define the subspace
 $$ M(S_{\mathcal{G}})=\{\theta|\theta_{G_i}=0 \text{~for all~} i\notin S_{\mathcal{G}}\}. $$
 Similar to Lasso we have $\bar{M}(S_{\mathcal{G}})=M(S_{\mathcal{G}})$.
 The orthogonal complement is 
 $$ M^\perp (S_{\mathcal{G}})=\{\theta |\theta_{G_i}=0 \text{~for all~} i\in S_{\mathcal{G}} \}. $$
 It is not hard to see that
 $$\|\alpha+\beta\|_{\mathcal{G},\gamma}=\|\alpha\|_{\mathcal{G},\gamma}+\|\beta\|_{\mathcal{G},\gamma}, $$
 for any $\alpha
 \in M(S_{\mathcal{G}})$ and $\beta\in \bar{M}^{\perp}(S_\mathcal{G})$.

 \subsection{Assumptions on Nonconvex regularizer $g_{\lambda,\mu} (\theta)$}\label{section:nonconvex_regularizer}
 In the non-convex case, we consider regularizers that  are separable across coordinates, i.e., $g_{\lambda,\mu}(\theta)=\sum_{j=1}^{p} \bar{g}_{\lambda,\mu} (\theta_j)$ . 
 Besides the separability, we have additional assumptions on $g_{\lambda,\mu}(\cdot)$. For the univariate function $\bar{g}_{\lambda,\mu} (t)$, we assume
 \begin{enumerate}
 	\item  $\bar{g}_{\lambda,\mu}(\cdot)$ satisfies $\bar{g}_{\lambda,\mu}(0)=0$ and is symmetric around zero (i.e.,~$\bar{g}_{\lambda,\mu}(t)=\bar{g}_{\lambda,\mu}(-t)$).
 	\item On the nonnegative real line, $\bar{g}_{\lambda,\mu} $ is nondecreasing.
 	\item For $t>0$,  $\frac{\bar{g}_{\lambda,\mu} (t)}{t}$ is nonincreasing in t.
 	\item $\bar{g}_{\lambda,\mu}(t)$ is differentiable at all $t\neq0$ and subdifferentiable at $t=0$, with $\lim_{t\rightarrow0^+} g_{\lambda,\mu}'(t)=\lambda L_g$ for a constant $L_g$.
 	\item  $ \bar{g}_{\lambda}(t):= (\bar{g}_{\lambda,\mu}(t)+\frac{\mu}{2}t^2)/\lambda$ is convex.	
 \end{enumerate}
 
 We provide two examples satisfying above assumptions.
 
 $(1)\quad \mathbf{SCAD_{\lambda,\zeta} (t) \triangleq }$ 
 $$
 \begin{cases}
 \lambda |t|,   & \quad \text{for}~ |t|\leq \lambda,\\
 -(t^2-2\zeta\lambda|t|+\lambda^2)/(2(\zeta-1)), & \quad \text{for} ~\lambda<|t|\leq \zeta\lambda,\\
 (\zeta+1)\lambda^2/2, &\text{for} ~|t|>\zeta\lambda,
 \end{cases}
 $$
 where $\zeta>2$ is a fixed parameter. It satisfies the assumption with $L_g=1$ and $\mu=\frac{1}{\zeta-1}$ \cite{loh2013regularized}.\\
 $$(2)\quad \mathbf{ MCP_{\lambda,b}(t)}\triangleq \mbox{sign}(t)\lambda \int_{0}^{|t|} (1-\frac{z}{\lambda b})_{+} dz ,$$
 where $b>0$ is a fixed parameter. MCP satisfies the assumption with $L_g=1$ and $\mu=\frac{1}{b}$~\cite{loh2013regularized}.
 
 \section{Main Result}
 
 In this section, we present our main theorems, which asserts linear convergence of SVRG under RSC, for both the convex and non-convex setups. We then instantiate it on the sparsity model, group sparsity model, linear regression with corrupted covariate and linear regression with SCAD regularizer. All proofs are deferred in Appendix.
 
 We analyze the (vanilla) SVRG (See Algorithm~\ref{alg:convex}) proposed in \citet{xiao2014proximal} 
 to solve Problem (\ref{obj2}). We remark that our proof can  easily be adapted to other accelerated versions of SVRG, e.g., non-uniform sampling. The algorithm contains an inner loop and an outer loop. We use the superscript $s$ to denote the step in the outer iteration and subscript $k$ to denote the step in the inner iteration throughout the paper.  For the non-convex problem~ (\ref{obj3}), we adapt SVRG to Algorithm \ref{alg:non-convex}. The idea of Algorithm \ref{alg:non-convex} is to solve 
 
 $$ \min_{g_\lambda (\theta)\leq \rho} \big(F(\theta)-\frac{\mu}{2} \|\theta\|_2^2 \big)+\lambda g_\lambda (\theta).  $$
 
 Since $g_\lambda (\theta)$ is convex, the proximal step in the algorithm is well defined. Also notice $\theta^s$ is randomly picked from $\theta_1 $ to $\theta_m$ rather than average.
 
 \begin{algorithm}[t]
 	\caption{Convex Proximal SVRG }
 	\label{alg:convex}
 	\begin{algorithmic}
 		\STATE {\bfseries Input:}  update frequency $m$, stepsize $\beta$, initialization $\theta_0$ 
 		\FOR{$s=1,2,...$}
 		\STATE{ $\tilde{\theta}={\theta}^{s-1}$, $\tilde{v}=\nabla F(\tilde{\theta})$, $\theta_0=\tilde{\theta}$}
 		\FOR {$k=1$ \bfseries to $m$}
 		\STATE {Pick $i_k$ uniformly random from $\{1,...,n\}$}
 		\STATE  $v_{k-1}=\nabla f_{i_k}(\theta_{k-1})-\nabla f_{i_k}(\tilde{\theta} )+\tilde{v}$
 		\STATE $\theta_k= \arg\min_{\psi(\theta)\leq \rho}\frac{1}{2}\|\theta -(\theta_{k-1}-\beta v_{k-1})\|_2^2+\beta \lambda\psi(\theta)$ 
 		\ENDFOR    
 		\STATE ${\theta}^s=\frac{1}{m}\sum_{k=1}^{m}\theta_k$
 		\ENDFOR
 	\end{algorithmic}
 \end{algorithm}
 
 \begin{algorithm}[t]
 	\caption{Non-Convex Proximal SVRG }
 	\label{alg:non-convex}
 	\begin{algorithmic}
 		\STATE {\bfseries Input:}  update frequency $m$, stepsize $\beta$, initialization $\theta_0$ 
 		\FOR{$s=1,2,...$}
 		\STATE{ $\tilde{\theta}={\theta}^{s-1}$, $\tilde{v}=\nabla F(\tilde{\theta})-\mu \tilde{\theta}$, $\theta_0=\tilde{\theta}$}
 		\FOR {$k=1$ \bfseries to $m$}
 		\STATE {Pick $i_k$ uniformly random from $\{1,...,n\}$}
 		\STATE  $v_{k-1}=\nabla f_{i_k}(\theta_{k-1})-\mu \theta_{k-1}-\nabla f_{i_k}(\tilde{\theta} )+\mu \tilde{\theta} +\tilde{v}$
 		\STATE $\theta_k= \arg\min_{g_{\lambda}(\theta)\leq \rho}\frac{1}{2}\|\theta -(\theta_{k-1}-\beta v_{k-1})\|_2^2+\beta \lambda g_{\lambda} (\theta)$ 
 		\ENDFOR    
 		\STATE ${\theta}^s=\theta_t$ for random chosen $t\in \{1,2,...,m \}$. 
 		\ENDFOR
 	\end{algorithmic}
 \end{algorithm}
 
 \subsection{Results for convex $G(\theta)$}
 To avoid notation clutter, we  define the following terms that appear  frequently in our theorem and corollaries. 
 
 \begin{definition}[List of notations]
 	\begin{itemize}
 		\item[]
 		\item Dual norm of $\psi(\theta)$: $\psi^*(\theta)$.
 		\item Unknown true parameter: $\theta^*$.
 		\item Optimal solution of Problem \eqref{obj2}: $ \hat{\theta}$.
 		\item Modified restricted strongly convex parameter:
 		$$\bar{\sigma}=\sigma-64\tau_{\sigma} H^2(\bar{M}).$$
 		\item Contraction factor: $ \alpha=(\frac{1}{Q(\beta,\bar{\sigma},L,m)}+\frac{4L\beta (m+1)}{(1-4L\beta)m}),$
 		where $Q(\beta,\bar{\sigma},L,M)=\bar{\sigma} \beta(1-4L\beta)m.$
 		\item
 		Statistical tolerance: 
 		$$ e^2=\frac{8\tau_\sigma}{Q(\beta,\bar{\sigma},L,m)}(8H(\bar{M})\|\hat{\theta}-\theta^*\|_2+8\psi(\theta^*_{M^\perp}))^2. $$
 	\end{itemize}
 \end{definition}
 
 The main theorem bounds the optimality gap  $  G(\theta^s)-G(\hat{\theta})$.
 \begin{theorem}\label{main_theorem}
 	In Problem (\ref{obj2}), suppose each $f_i(\theta)$ is $L$ smooth, $\theta^*$ is feasible, i.e., $\psi(\theta^*)\leq \rho$ , $F(\theta)$ satisfies RSC with parameter $(\sigma, \tau_\sigma)$,
 	the regularizer $\psi$ is decomposable w.r.t.\
 	$(M,\bar{M}^\perp)$, such that  $ \bar{\sigma}>0, \alpha\in [0,1) $ and suppose $n>c\rho^2\log p$ for some constant c. Consider any  the regularization parameter $\lambda$ satisfies $\lambda\geq 2\psi^*(\nabla F(\theta^*))$, 
 	then for any tolerance $\kappa^2\geq \frac{e^2}{1-\alpha} $, if 
 	$ s>3\log  \big(\frac{G(\theta^0)-G(\hat{\theta})}{\kappa^2}\big)\big/\log(1/\alpha) $
 	then
 	$ G(\theta^s)-G(\hat{\theta})\leq \kappa^2, $
 	with probability at least $1-\frac{c_1}{n}$, where $c_1$ is universal positive constant.
 \end{theorem}
 To put Theorem~\ref{main_theorem} in context, some remarks are in order.
 
 \begin{enumerate}
 	\item If we compare with result in standard SVRG (with $u$ strong convexity) \cite{xiao2014proximal}, the difference is that we use modified restricted strongly convex $ \bar{\sigma}$ rather than strongly convex parameter $u$. Indeed the high level idea of the proof is to replace strong convexity by RSC. Set $m\approx C\frac{L}{\bar{\sigma}}$ where $C$ is some universal positive constant, $\beta=\frac{1}{16L} $ as that in\cite{xiao2014proximal} such that $\alpha\in (0,1) $, we have the gradient complexity $O\big( (n+\frac{L}{\bar{\sigma}}   \big) \log \frac{1}{\epsilon})$ when $\epsilon>e^2/(1-\alpha)$  (2m gradients in inner loop and and n gradients for outer loop ). 
 	\item In many statistical models (see corollaries for concrete examples),  we can  choose suitable subspace $M$, step size $\beta$ and $m$ to obtain $\bar{\sigma}$ and $\alpha$ satisfying $\bar{\sigma}>0$, $\alpha<1$. For instance in Lasso, since $\tau_\sigma H^2(\bar{M})\approx \frac{r\log p}{n} $ and $\sigma=1/2$ (suppose the feature vector $x_i$ is sampled from $N(0,I)$ ), when $\theta^*$ is sparse (i.e., $r$ is small) we can set $\bar{\sigma}>0$, e.g., $1/4$, if $\frac{64 r \log p}{n}\leq \frac{1}{4}.$ 
 	\item Smaller $r$ leads to larger $\bar{\sigma}$ , thus smaller $\alpha$ and $\frac{L}{\bar{\sigma}}$, which leads to faster convergence.
 	\item In terms of the tolerance, notice that in cases like sparse regression we can choose $M$ such that $ \theta^*\in M$, and hence the tolerance equals to $\frac{c\tau_\sigma}{Q(\beta,\bar{\sigma},L,m)}H^2(\bar{M})\|\hat{\theta}-\theta^*\|_2^2$. Under above setting in 1 and 2, and combined with the fact that $\tau_\sigma H^2(\bar{M})\approx \frac{r\log p}{n} $ (in Lasso),  we have $e^2 =o(\|\hat{\theta}-\theta^*\|_2^2)$, i.e., the tolerance is dominated by the statistical error of the model.
 \end{enumerate}
 Therefore, Theorem~\ref{main_theorem} indeed states that
 the optimality gap decreases geometrically until it reaches the statistical tolerance. Moreover, this statistical tolerance is dominated by $\|\hat{\theta}-\theta^*\|_2^2$, and thus can be ignored from a statistic perspective when solving formulations such as sparse regression via Lasso. 
  It is instructive to instantiate the above general results to several concrete statistical models, by choosing appropriate  subspace pair $(M,\bar{M} )$ and  check the RSC conditions, which we detail in the following subsections. 
 
 \subsubsection{Sparse regression}
 The first model we consider is Lasso, where $f_i(\theta)=\frac{1}{2} (\langle \theta, x_i \rangle-y_i)^2$ and $\psi(\theta)=\|\theta\|_1$. More concretely, we consider a model where each data point $x_i$ is i.i.d.\  sampled from a zero-mean normal distribution, i.e., $x_i\sim N(0,\Sigma)$. We denote the data matrix by $X\in \mathbb{R}^{n\times p}$   and   the smallest eigenvalue of $\Sigma$ by $\sigma_{\min} (\Sigma)$, and  let $\nu(\Sigma)\triangleq \max_{i=1,...,p}\Sigma_{ii}$.  The observation is generated by $y_i=x_i^T\theta^*+\xi_i$, where $\xi_i$ is the zero mean sub-Gaussian noise with variance $u^2$. We use $X_j\in \mathbb{R}^n$ to denote $j$-th column of $X$. Without loss of generality, we require $X$ is column-normalized, i.e., $\frac{\|X_j\|_2}{\sqrt{n}}\leq 1  \quad \text{for all} \quad j=1,2,...,p$. Here, the constant $1$ is chosen arbitrarily to simplify the exposition, as we can always rescale the data.

 \begin{corollary}[Lasso]\label{cor.lasso}
 	Suppose $\theta^*$ is supported on a subset of cardinality at most r,  $n>c\rho^2\log p$ and  we choose $\lambda$ such that  $\lambda \geq 6u\sqrt{\frac{\log p}{n}}$,
 	then $\bar{\sigma}=\frac{1}{2}\sigma_{\min}(\Sigma)- c_1\nu(\Sigma)\frac{ r\log p}{n} $, $ \alpha=(\frac{1}{\bar{\sigma} \beta(1-4L\beta)m}+\frac{4L\beta (m+1)}{(1-4L\beta)m}), c,c_1$ are some universal positive constants.    
 	For any $$\kappa^2\geq  \frac{c_2}{(1-\alpha)Q(\beta,\bar{\sigma},L,m)}\frac{r\log p}{n}\|\hat{\theta}-\theta^*\|_2^2 $$
 	we have
 	$$  G(\theta^s)-G(\hat{\theta})\leq \kappa^2,$$
 	with probability $1-\frac{c_3}{n}$, for
 	$ s> 3\frac{\log ( \frac{G(\theta_0)-G(\hat{\theta})}{\kappa^2} )}{\log(1/\alpha)}, $ where $c_2$ $c_3$ are universal positive  constants.
 	
 \end{corollary}
 We offer some discussions to put this corollary in context. To achieve statistical consistency for Lasso, it is necessary to have  $\frac{r \log p}{n}=o(1)$ \cite{negahban2009unified}. Under such a condition, we have $\frac{r\log p}{n} \nu (\Sigma)c_1\leq\frac{1}{4}\sigma_{\min}(\Sigma)$ which implies $\bar{\sigma}\geq \frac{1}{4}\sigma_{\min}(\Sigma)$. Thus  $\bar{\sigma}$ is bounded away from zero. Moreover, if we set $m\approx C\frac{L}{\bar{\sigma}}$   following standard practice of SVRG \cite{johnson2013accelerating,xiao2014proximal} and set  $\beta=\frac{1}{16L}$, then  $ \alpha <1 $  which guarantees the convergence of the algorithm. The  requirement of $\lambda$ is commonly used to prove the statistical property of Lasso \cite{negahban2009unified}. Further notice that under this setting, we have $  \frac{c_2}{(1-\alpha)Q(\beta,\bar{\sigma},L,m)}\frac{r\log p}{n}=o(1)$, which implies that the statistical tolerance  is of a  lower order to $\|\hat{\theta}-\theta^*\|_2^2$ which is the statistical error of the {\em optimal solution} of Lasso. Hence it can be ignored from the statistical view.  Combining these together, Corollary~\ref{cor.lasso} states that the objective gap decreases geometrically until it achieves the fundamental statistical limit  of Lasso.

 \subsubsection{Group sparsity model}
 
 In many applications, we need to consider the group sparsity, i.e., a group of coefficients are set to zero simultaneously. We assume features are partitioned into disjoint groups, i.e., $\mathcal{G}=\{G_1,G_2,...,G_{N_{\mathcal{G}}}\}$, and assume $\bar{\gamma}=(\gamma,\gamma,...,\gamma)$. That is, the regularization is $\psi(\theta)= \|\theta\|_{\mathcal{G},\bar{\gamma}} \triangleq\sum_{g=1}^{N_{\mathcal{G}}} \|\theta_{g}\|_\gamma$.  For example, group Lasso corresponds to  $\gamma=2$. Other choice of $\gamma$ may include $\gamma=\infty$, which is suggested in \citet{turlach2005simultaneous}.

 Besides RSC condition, we need the following group counterpart of the column normalization condition: Given a group $G$ of size $q$, and $X_G\in \mathbb{R}^{n\times q}$, we define the associated operator norm $||| X_G  |||_{\gamma\rightarrow2}=\max_{\|\theta\|_\gamma=1} \|X_G\theta\|_2 $, and require that
 $$ \frac{|||X_{G_i} |||_{\gamma\rightarrow 2}}{\sqrt{n}}\leq 1 \quad \text{for all} \quad i=1,2,..., N_{\mathcal{G}}.$$
 Observe that when $G_i$ are all singleton, this condition reduces to column normalization condition.  We assume the data generation model is  $y_i=x_i^T\theta^*+\xi_i$, and $x_i\sim N(0,\Sigma)$.
 
 We discuss the case of $\gamma=2$, i.e., Group Lasso in the following.
 \begin{corollary}
 	Suppose the dimension of $\theta$ is $p$ and each group has $q$ parameters, i.e., $p=qN_\mathcal{G}$, $s_{\mathcal{G}} $ is the cardinality of non-zeros group, $\xi_i$ is zero mean sub-Gaussian noise with variance $u^2 $, $n>c\rho^2\log p$ for some constant c, If we choose  $\lambda\geq 4u (\sqrt{\frac{q}{n}}+\sqrt{\frac{\log N_{\mathcal{G}}}{n}})$,
 	and let
 	\begin{equation*}
 		\begin{split}
 			&\bar{\sigma}=\kappa_1(\Sigma)-c\kappa_2(\Sigma)s_{\mathcal{G}} (\sqrt{\frac{q}{n}}+\sqrt{\frac{3 \log N_{\mathcal{G}}}{n}} )^2;\\
 			&\alpha=(\frac{1}{\bar{\sigma} \beta(1-4L\beta)m}+\frac{4L\beta (m+1)}{(1-4L\beta)m}), 
 		\end{split}
 	\end{equation*}
 	where $\kappa_1$ and $\kappa_2$ are some strictly positive numbers which only depends on $\Sigma$,
 	then for any 
 	$$\kappa^2\geq  \frac{c_2\kappa_2(\Sigma)}{(1-\alpha)Q(\beta,\bar{\sigma,L,\beta})}s_{\mathcal{G}}(\sqrt{\frac{q}{n}}+\sqrt{\frac{3 \log N_{\mathcal{G}}}{n}} )^2 \|\hat{\theta}-\theta^*\|_2^2,$$ 
 	we have
 	$$ G(\theta^s)-G(\hat{\theta})\leq \kappa^2, $$
 	with high probability,
 	for 
 	\begin{equation}
 		\begin{split}
 			s>&3\frac{\log ( \frac{G(\theta_0)-G(\hat{\theta})}{\kappa^2} )}{\log(1/\alpha)} .
 		\end{split}
 	\end{equation}
 \end{corollary}
 Notice that to ensure $\bar{\sigma}\geq 0$, it suffices to have
 $$ s_{\mathcal{G}} \left(\sqrt{\frac{q}{n}}+\sqrt{\frac{3 \log N_{\mathcal{G}}}{n}} \right)^2=o(1)  .$$ This is a mild condition, as it is needed to guarantee the statistical consistency of Group Lasso \cite{negahban2009unified}.  In practice, this condition is not hard to satisfy  when $q$ and $s_{\mathcal{G}}$ are small.
 We can easily adjust $L,\beta, m $ to make $\alpha<1$. Since $ s_{\mathcal{G}} \big(\sqrt{\frac{q}{n}}+\sqrt{\frac{3 \log N_{\mathcal{G}}}{n}} \big)^2=o(1) $ and $Q(\beta,\bar{\sigma},L,m)$ is in the order of $n$ if we set $\beta=\frac{1}{20L}$,$m\approx C\frac{L}{\bar{\sigma}}$  , we have $\kappa^2=o(\|\hat{\theta}-\theta^*\|_2^2).$ Thus, similar as the case of Lasso, the objective gap decreases geometrically up to the scale $o(\|\hat{\theta}-\theta^*\|_2^2)$, i.e., dominated by the  statistical error of the model.
 
 \subsubsection{Extension to Generalized linear model }
 The results  on Lasso and group Lasso are readily extended to generalized linear models, where we consider the model
 $$ \hat{\theta}=\arg\min_{\theta\in \Omega'}\{ \frac{1}{n}\sum_{i=1}^{n} {\Phi(\theta,x_i)-y_i\langle \theta,x_i \rangle}+\lambda \|\theta\|_1   \}, $$
 with $\Omega'=\Omega\cap \mathbb{B}_2 (R)$ and $\Omega=\{ \theta| \|\theta\|_1\leq \rho \}$, where $R$ is a universal constant \cite{loh2013regularized}. This requirement is essential, for instance  for the logistic function , the Hessian function $\Phi''(t)=\frac{\exp(t)}{(1+\exp(t))^2}$ approached to zero as its argument diverges. 
 Notice that when $\Phi(t)={t^2}/{2}$, the problem reduces to Lasso. The RSC condition admit the form 
 $$ \frac{1}{n} \sum_{i=1}^{n} \Phi''(\langle \theta_t,x_i \rangle ) (\langle x_i,\theta-\theta' \rangle )^2\geq \frac{\sigma}{2}\|\theta-\theta'\|_2^2-\tau_\sigma \|\theta-\theta'\|_1, \mbox{for all} \quad \theta,\theta' \in \Omega'$$
 For a board class of log-linear models, the RSC condition holds with $\tau_\sigma=c\frac{\log p}{n}$. Therefore, we obtain same results as those of Lasso, modulus change of  constants. For more details of  RSC conditions in generalized linear model, we refer the readers to \cite{negahban2009unified}.

 \subsection{Results on non-convex $G(\theta)$}
 We define the  following notations.
 \begin{itemize}
 	\item $L_\mu=\max\{\mu, L-\mu\} $
 	\item Modified restricted strongly convex parameter $\bar{\sigma}=\sigma-\mu-64\tau_\sigma r$, where $\tau_\sigma=\tau \log p/n$, $\tau$ is a constant, $r$ is the cardinality of $\theta^*$.
 	\item Contraction factor
 	\begin{equation}\label{contraction factor}
 	\alpha=\frac{8L\beta^2 (m+1) + \frac{2 (1+\beta\mu+4L\mu\beta^2+4L\beta^2 \mu m)}{\bar{\sigma}}}{\beta m (2-8L\beta-\frac{2\mu}{\bar{\sigma}} (1+4L\beta))}
 	\end{equation}
 	\item Statistical tolerance 
 	\begin{equation}
 	e^2=64\tau_\sigma \chi r\|\hat{\theta}-\theta^*\|_2^2,
 	\end{equation}
 	where
 	\begin{equation}\label{chi}
 	\chi= \frac{\frac{2\mu \beta m}{\bar{\sigma}} (1+4L\beta)+ (1+\beta\mu+4L\mu \beta^2+4L\beta^2\mu m) \frac{2}{\bar{\sigma}}}{\beta m (2-8L\beta-\frac{2\mu}{\bar{\sigma}} (1+4L\beta))}.
 	\end{equation}
 	
 \end{itemize}
 
 \begin{theorem}\label{main_theorem_non_convex}
 	In Problem \eqref{obj3}, suppose each $f_i(\theta)$ is $L$ smooth,  $\beta\leq \frac{1}{L_\mu}$,  $\theta^*$ is feasible, $g_{\lambda,\mu}(\cdot)$ satisfies Assumptions in section \ref{section:nonconvex_regularizer}, $F(\theta)$ satisfies RSC with parameter $\sigma$, $\tau_\sigma=\tau \log p/n$, and $\bar{\sigma}>0$, $\alpha\in [0,1]$ by choosing suitable $\beta$ and $m$. Suppose $\hat{\theta}$ is the global optimal, $n>c\rho^2\log p$ for some positive constant c, consider any choice of the regularization parameter $\lambda$ such that $\lambda> \max \{\frac{4}{L_g} \|\nabla F(\theta^*)\|_\infty,16\rho\tau\frac{\log p}{n L_g} \}$, then for any tolerance  $\kappa^2\geq \frac{e^2}{1-\alpha}$ if 
 	$$ s>3\log (\frac{G(\theta^0)-G(\hat{\theta})}{\kappa^2})\big/\log(1/\alpha),$$
 	then
 	$ G(\theta^s)-G(\hat{\theta})\leq \kappa^2, $
 	with probability at least $1-\frac{c_1}{\sqrt{n}}$. 
 \end{theorem}
 We provide some remarks to make the theorem more interpretable.
 \begin{enumerate}
 	\item We require $\sigma\geq \mu+64\tau_\sigma r $ to ensure $\bar{\sigma}>0$. In addition, the non-convex parameter $\mu$ can not be larger than $\bar{\sigma}$. In particular, if $ \frac{\mu}{\bar{\sigma}}\geq \frac{1-4L\beta}{1+4L\beta}$, then $\alpha<0$ and it is not possible to set $\alpha\in (0,1)$ by tunning $m$ and learning rate $\beta$.
 	\item  We consider a concrete case to obtain a sense of the value of different terms we defind. Suppose $n=5000$ and if we set $m\approx 10\frac{L}{\bar{\sigma}}$  which is typical for SVRG, $\beta=\frac{1}{16L}$ and suppose we have $\bar{\sigma}=0.4$, $\mu=0.1$, then we have the contraction factor $\alpha\approx 0.8$. Furthermore, we have
 	$\chi(\beta,\mu,L,m,\sigma)\approx 0.9$, which leads to $ e^2\approx 60  \frac{r\log p}{n} \|\hat{\theta}-\theta^*\|_2^2$. When the model is sparse, the tolerance is dominated by statistical error of the model. 
 \end{enumerate}
 
 \subsubsection{Linear regression with SCAD}
 The first non-convex model we consider is the linear regression with SCAD.  That is, $f_i(\theta)=\frac{1}{2} (\langle \theta, x_i \rangle-y_i)^2$ and $g_{\lambda,\mu}(\cdot)$ is $SCAD(\cdot)$ with parameter $\lambda$ and $\zeta$. The  data $(x_i,y_i)$ are generated in a same way as in the Lasso example.
 \begin{corollary}[Linear regression with SCAD]
 	Suppose we have i.i.d. observations $\{ (x_i,y_i) \}$ ,   $\theta^*$ is supported on a subset of cardinality at most $r$,   $\hat{\theta}$ is the global optimum, $n>c\rho^2\log p$ for some positive constant c, $\bar{\sigma}=\frac{1}{2}\sigma_{\min}(\Sigma)-\mu- c_1\nu(\Sigma)\frac{ r\log p}{n}  $ and $\beta\leq \frac{1}{L_\mu}$ in the algorithm.  We choose  $\lambda $ such that $\lambda\geq \max \{ 12u \sqrt{\frac{\log p}{n}}, 16\rho\tau \frac{\log p}{n} \}$, 
 	Then for any tolerance 
 	$$\kappa^2\geq \frac{c_2\chi   }{ (1-\alpha)}\frac{\tau r\log p}{n}\|\hat{\theta}-\theta^*\|_2^2, $$
 	where $\alpha$  and $\chi$ are defined in  (\ref{contraction factor})  and (\ref{chi}) with $\mu=\frac{1}{\zeta-1}$.
 	if 
 	$$ s>3\log (\frac{G(\theta^0)-G(\hat{\theta})}{\kappa^2})\big/\log(1/\alpha),$$
 	then
 	$ G(\theta^s)-G(\hat{\theta})\leq \kappa^2, $
 	with probability at least $1-\frac{c_3}{\sqrt{n}}$.
 \end{corollary}

 Suppose we have $\zeta=3.7$, $n=5000$, $\bar{\sigma}=1$, $m\approx 10\frac{L}{\bar{\sigma}}$ , $\beta=\frac{1}{40L}$ then $\alpha\approx 0.66$ and $\chi\approx 3$. Notice in this setting we have $ \frac{c_2\chi}{(1-\alpha)} \frac{\tau r\log p}{n} \|\hat{\theta}-\theta^*\|_2^2$=$o(\|\hat{\theta}-\theta^*\|_2^2)$, when the model is sparse.  Thus this corollary asserts that the optimality gap decrease geometrically until it achieves the statistical limit of the model.
 

  \subsubsection{Linear regression with noisy covariate}
  Next we discuss a non-convex M-estimator on linear regression with noisy covariate, termed {\em corrected Lasso} which is proposed by \cite{loh2011high}. Suppose the data are generated according to a  linear model $y_i=x_i^T\theta^*+\xi_i,$ where $\xi_i$ is a random zero-mean sub-Gaussian noise with variance $v^2$. More concretely, we consider a model where each data point $x_i$ is i.i.d.\  sampled from a zero-mean normal distribution, i.e., $x_i\sim N(0,\Sigma)$. We denote the data matrix by $X\in \mathbb{R}^{n\times p}$ , the smallest eigenvalue of $\Sigma$ by $\sigma_{\min} (\Sigma)$ and the largest eigenvalue by $\sigma_{\max} (\Sigma)$  and  let $\nu(\Sigma)\triangleq \max_{i=1,...,p}\Sigma_{ii}$. 
  
  However, $x_i$ are not directly observed. Instead, we observe $z_i$ which is $x_i$ corrupted by addictive noise, i.e.,  $z_i=x_i+w_i$, where $w_i\in \mathbb{R}^p$ is a random vector independent of $x_i$, with zero-mean and known covariance matrix $\Sigma_w$. Define $\hat{\Gamma}=\frac{Z^TZ}{n}-\Sigma_w$ and $\hat{\gamma}=\frac{Z^Ty}{n}$. Then the corrected Lasso is given by 
  $$\hat{\theta}\in \arg \min_{\|\theta\|_1\leq \rho} \frac{1}{2} \theta^T \hat{\Gamma} \theta-\hat{\gamma} \theta+ \lambda \|\theta\|_1. $$
  Equivalently, it solve 
  $$ \min_{\|\theta\|_1\leq \rho}\frac{1}{2n}\sum (y_i-\theta^Tz_i)^2-\frac{1}{2}\theta^T\Sigma_w\theta+\lambda \|\theta\|_1 .$$

  We give the theoretical guarantee for SVRG on corrected Lasso. 
  \begin{corollary}[Corrected Lasso]

  	Suppose we have i.i.d.\ observations $\{ (z_i,y_i) \}$ from the linear model with additive noise, and $\theta^*$ is supported on a subset of cardinality at most $r$, $\Sigma_w=\gamma_w I$ . Let $\hat{\theta}$ denote the global optimal solution, and suppose $n>c\rho^2\log p$ for some positive constant $c$. We choose $\lambda $ such that $\lambda\geq \max \{ c_1 \varphi \sqrt{\frac{\log p}{n}}, 16\rho\tau \frac{\log p}{n} \}$, where $ \varphi=(\sqrt{\sigma_{\max} (\Sigma)}+\sqrt{\gamma_w}) (v+\sqrt{\gamma_w} \|\theta^*\|_2)$.
  	Then for any tolerance 
  	$$\kappa^2\geq \frac{c_1\chi  r  }{ (1-\alpha)}\frac{\tau\log p}{n} $$
  	where 
  	$\bar{\sigma}=\frac{1}{2}\sigma_{\min}(\Sigma)- c_2\frac{ r\log p}{n} $, 
  	$$\alpha=\frac{8L\beta^2 (m+1) + \frac{2 (1+\beta\gamma_w+4L\gamma_w\beta^2+4L\beta^2 \gamma_w m)}{\bar{\sigma}}}{\beta m (2 (1-4L\beta)-\frac{2\gamma_w}{\bar{\sigma}} (1+4L\beta))},$$
  	$$ \chi= \frac{\frac{2\gamma_w \beta m}{\bar{\sigma}} (1+4L\beta)+ (1+\beta\gamma_w+4L\gamma_w \beta^2+4L\beta^2\gamma_w m) \frac{2}{\bar{\sigma}}}{\beta m (2(1-4L\beta)-\frac{2\gamma_w}{\bar{\sigma}} (1+4L\beta))},$$
  	if 
  	$$ s>3\log (\frac{G(\theta^0)-G(\hat{\theta})}{\kappa^2})/\log(1/\alpha),$$
  	then
  	$ G(\theta^s)-G(\hat{\theta})\leq \kappa^2, $
  	with probability at least $1-\frac{c_3}{\sqrt{n}}$, where $c_1,c_2,c_3$ are some positive constant.
  \end{corollary}
  We offer some discussions to interpret corollary. 
  \begin{itemize}
  	\item We can easily extend the result to to more general $\Sigma_w$ where $\Sigma_w\preceq\gamma_w I$.
  	\item The requirement of $\lambda$ is similar with the batch counterpart in \cite{loh2013regularized}.
  	\item Similar with Lasso, since $ \frac{r\log p}{n}=o(1) $, $\bar{\sigma}>0$ is easy to satisfy.
  	\item Concretely, suppose we have $ \bar{\sigma}=0.3$ $\gamma_w=0.1$ , $m\approx 10\frac{L}{\bar{\sigma}}$  and $\beta=\frac{1}{20L}$ , we have 
  	$ \alpha\approx 0.68 $ and $ \chi\approx 1.2.$ Thus we have $e^2=o( \frac{r\log p}{n}) \|\hat{\theta}-\theta^*\|_2^2.$ Again, it indicates the objective gap decreases geometrically up to the scale $o(\|\hat{\theta}-\theta^*\|_2^2)$, i.e., dominated by the  statistical error of the model.
  \end{itemize}
 \section{Experimental results}
 We report some numerical experimental results on  in this subsection. The main objective of the numerical experiments is to validate our theoretic findings~--~that for a class of non-strongly-convex or non-convex optimization problems, SVRG indeed achieves desirable linear convergence. Further more, when the problem is ill-conditioned, SVRG is much better than the batched gradient method. We test SVRG  on synthetic and real datasets and compare the results with those of several other algorithms. Specifically, we implement the following algorithms.
 
 \begin{itemize}
 	\item SVRG: We implement Algorithm \ref{alg:convex}, which is the proximal SVRG proposed in \cite{xiao2014proximal}.
 	\item Composite gradient method: This is  the full proximal gradient method. \citet{agarwal2010fast} established its linear convergence in a setup similar to the convex case we consider (i.e., without strong  convexity).
 	\item SAG: We adapt the stochastic average gradient method \cite{schmidt2013minimizing} to a proximal variant. Note that to the best of our knowledge,  the convergence Prox-SAG has not been established in literature. In particular, it is not known whether this method converges linearly in our setup. Yet, our numerical results  seem to suggest that the algorithm does enjoy linear convergence. 
 	
 	\item Prox-SGD: Proximal stochastic gradient method. It converges sublinearly in our setting. 
 	\item RDA: Regularized dual averaging method \cite{xiao2010dual}.  It converges sublinearly  in our setting. 
 \end{itemize}
 
 For the algorithm with constant learning rate (SAG, SVRG, Composite Gradient), we tune the learning rate from an exponential grid $\{ 2, \frac{2}{2^1},...,\frac{2}{2^{12}} \}$ and chose the one with the best performance.
 Notice we do not include another popular algorithm SDCA \cite{shalev2014accelerated} in our experiments, because the proximal step in SDCA requires  strong convexity of $\psi(\theta)$ to implement.
 \subsection{Synthetic Data}
 \subsubsection{Lasso}
 We first tested solving Lasso on synthetic data. We generate data as follows: $ y_i=x_i^T\theta^*+\xi_i$, where each data point $x_i\in \mathbb{R}^p$ is drawn from normal distribution $N(0,\Sigma)$, and the noise $ \xi_i$ is   drawn from $N(0,1)$. The coefficient $\theta^*$ is sparse with cardinality $r$, where the non-zero coefficient equals to $\pm 1$ generated from the Bernoulli distribution with parameter $0.5$. For the covariance matrix $\Sigma$, we  set the diagonal entries to $1$, and the off-diagonal entries to $b$ (notice when $b\neq 0$, the problem may be ill). The sample size is $n=2500$, and the dimension of problem is $p=5000$. Since $p>n$, the objective function is clearly not strongly convex.
 
 \begin{figure}
 	\begin{subfigure}[b]{0.5\textwidth}
 		\includegraphics[width=\textwidth]{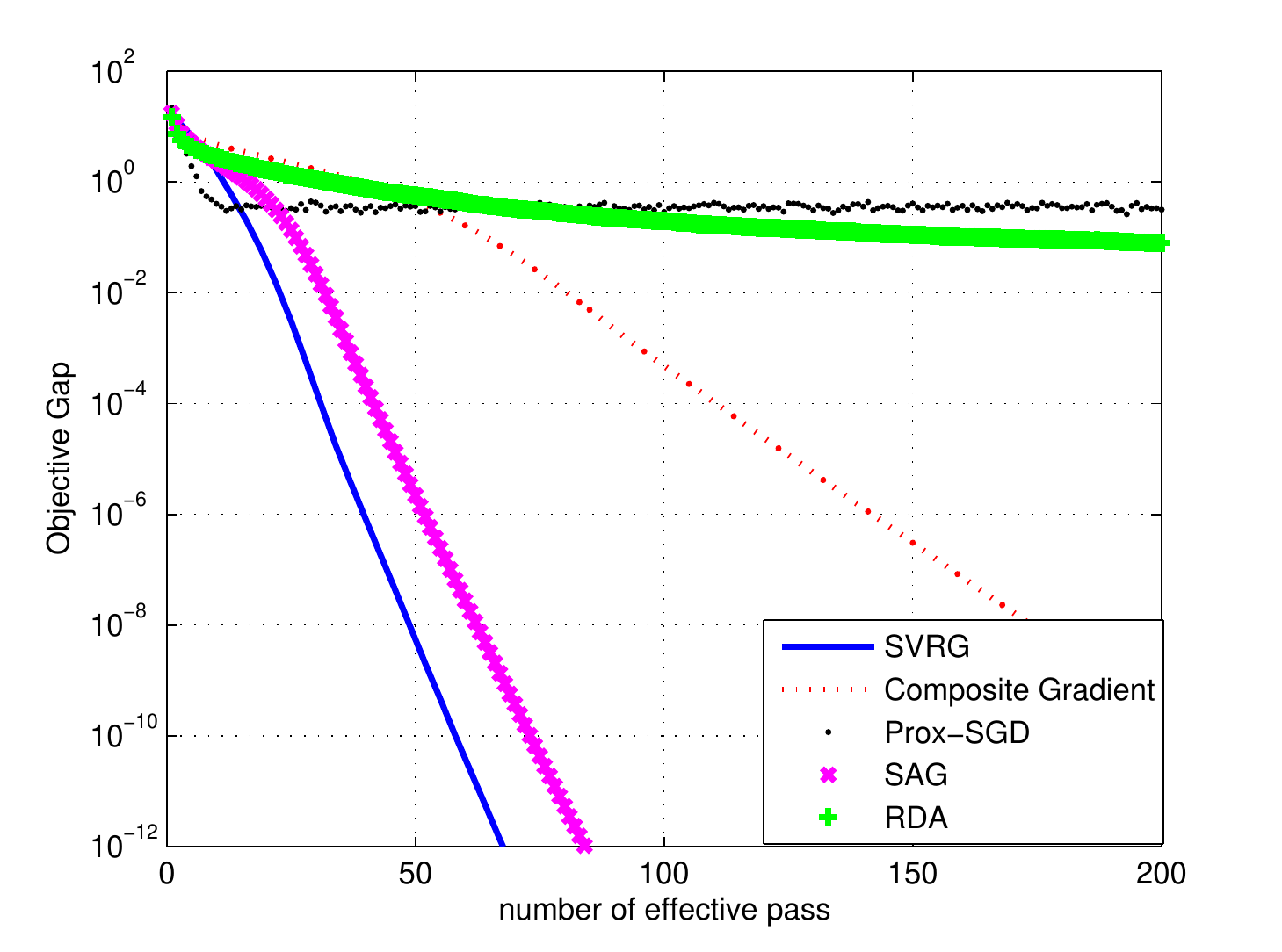}
 		\caption{r=50,b=0}
 	\end{subfigure}
 	\begin{subfigure}[b]{0.5\textwidth}
 		\includegraphics[width=\textwidth]{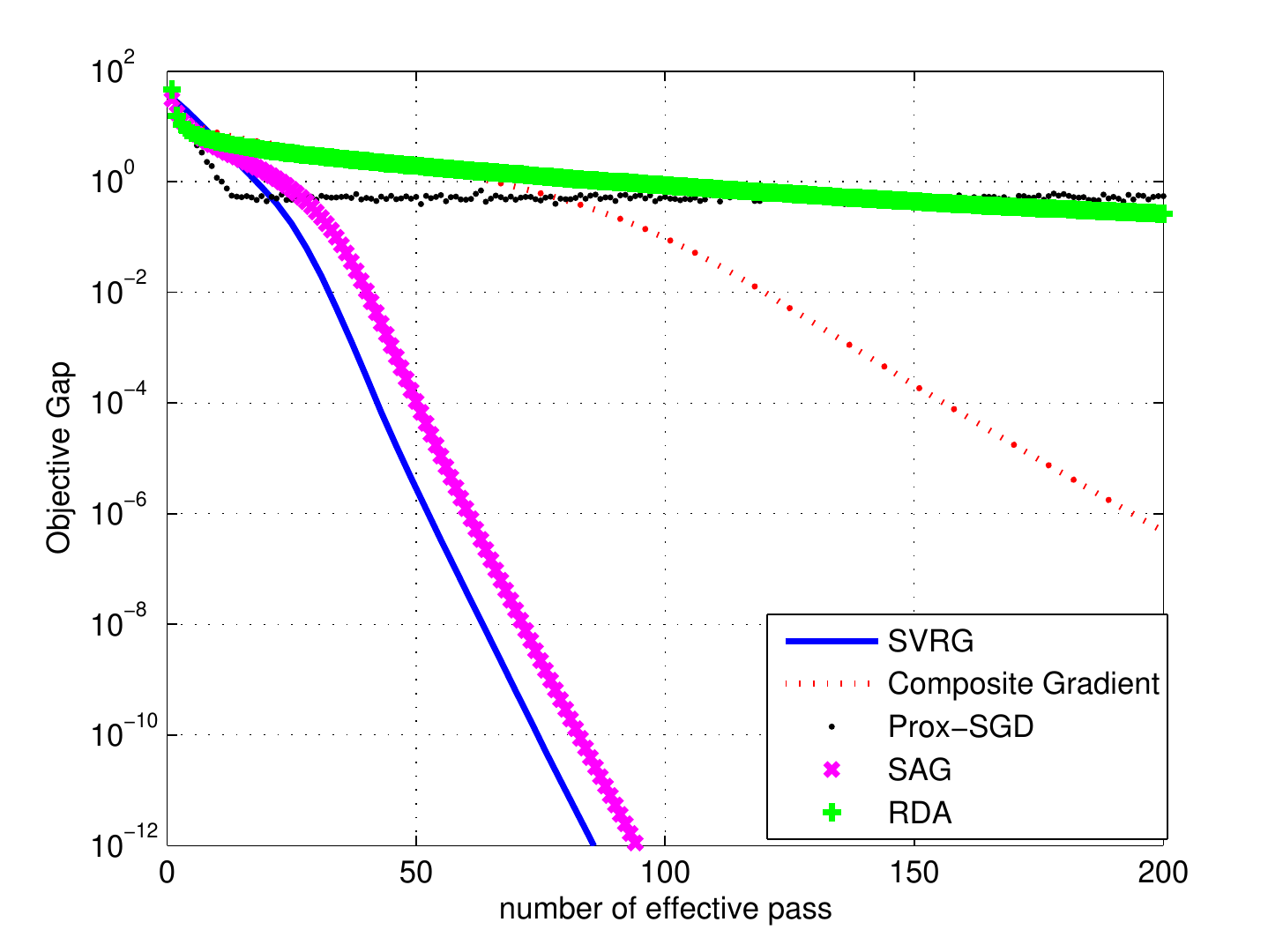}
 		\caption{r=100,b=0}
 	\end{subfigure}
 	\begin{subfigure}[b]{0.5\textwidth}
 		\includegraphics[width=\textwidth]{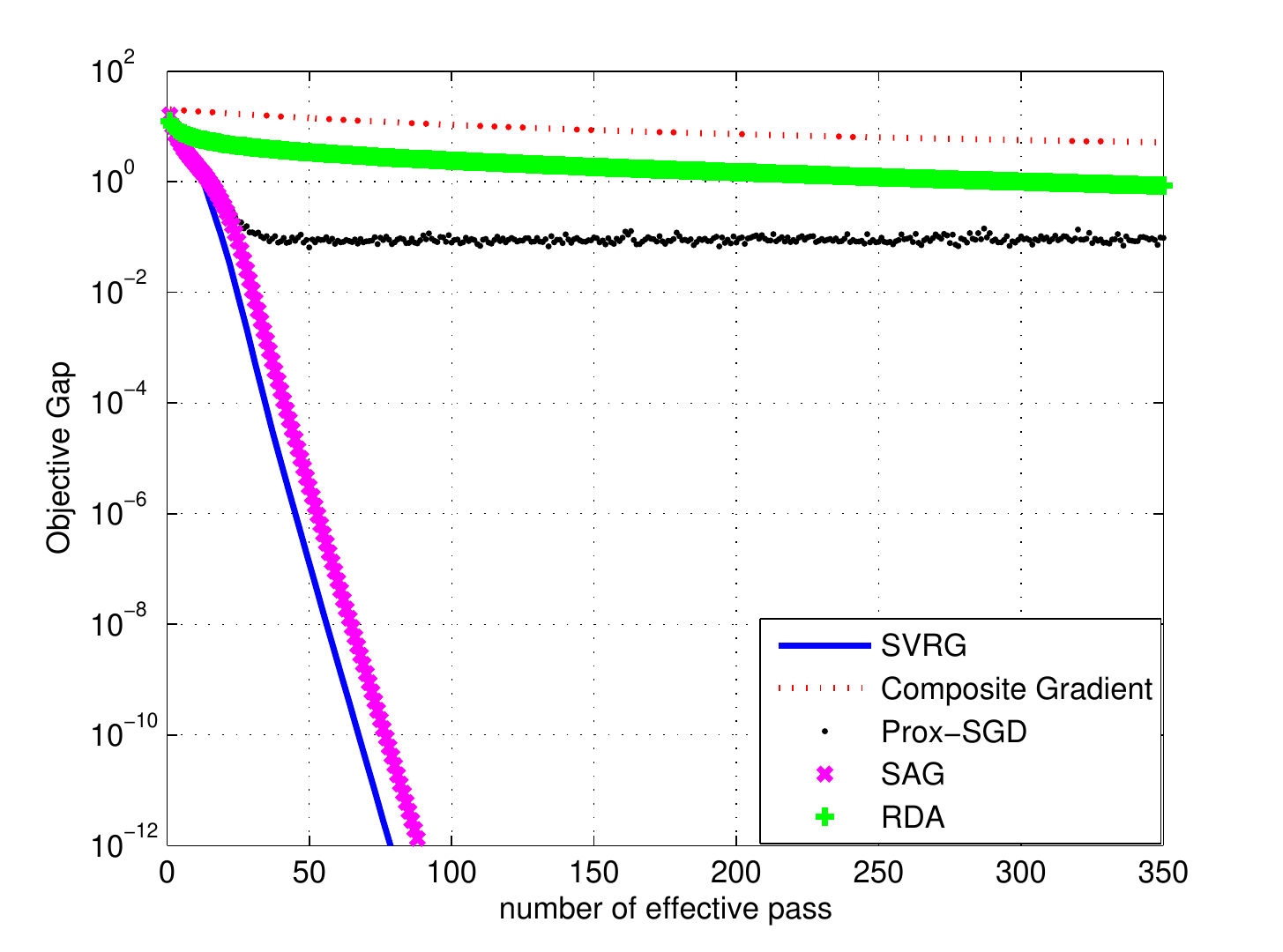}
 		\caption{r=50,b=0.1}
 	\end{subfigure}~~~
 	\begin{subfigure}[b]{0.5\textwidth}
 		\includegraphics[width=\textwidth]{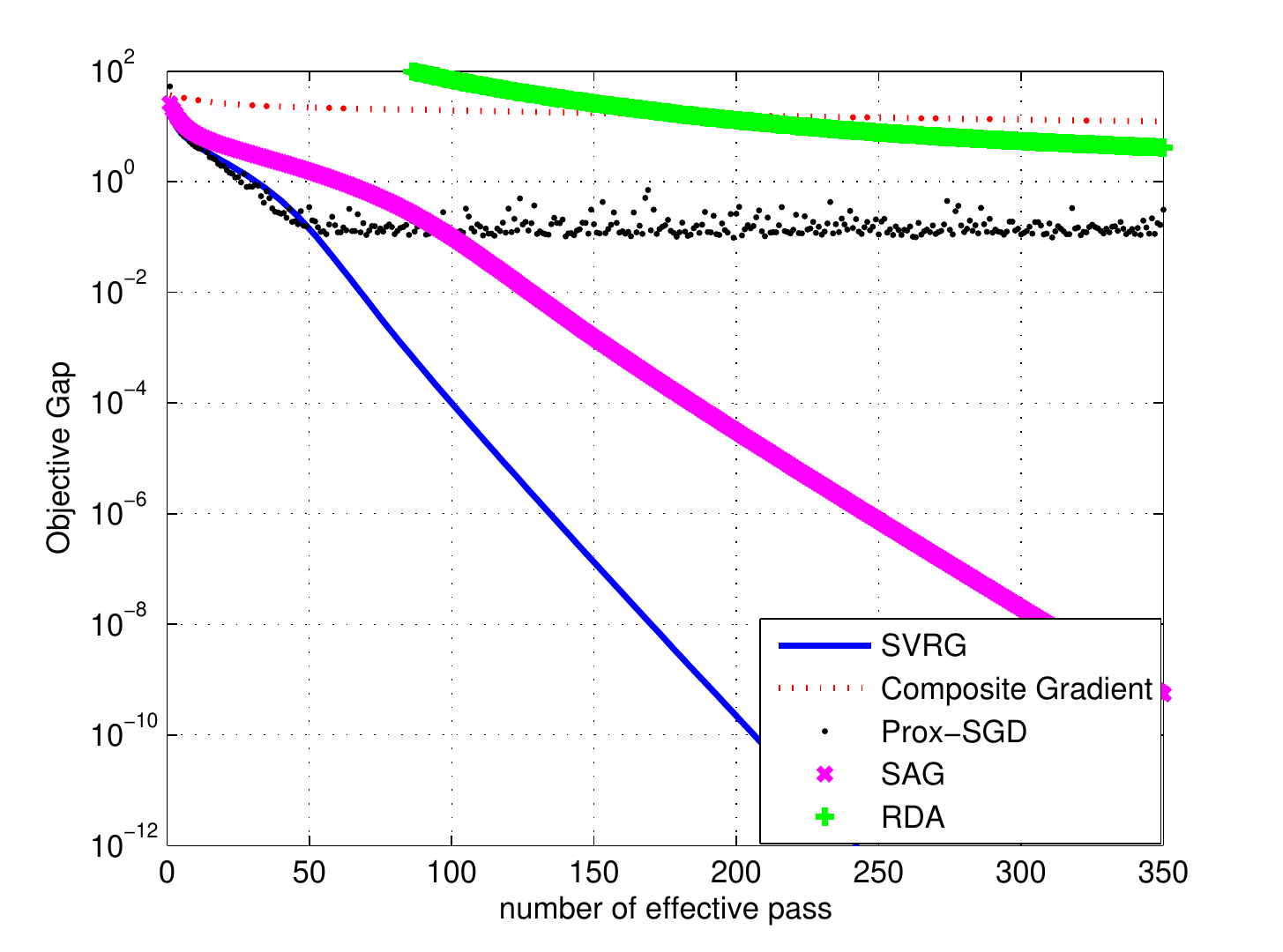}
 		\caption{r=100,b=0.4}
 	\end{subfigure}
 	\caption{Comparison between five algorithms on Lasso. The x-axis is the number of pass over the dataset. y-axis is the objective gap $G(\theta_k)-G(\hat{\theta})$ with a log scale. In figure (a), r=50,b=0 . In figure (b), r=100,b=0. In figure (c), r=50,b=0.1. In figure (d), r=100,b=0.4 . } \label{lasso}
 \end{figure}
 
 In Figure \ref{lasso},  for different values of $r$ and $b$, we report the objective gap $G(\theta_k)-G(\hat{\theta})$ versus the number of passes of the dataset for the algorithms mentioned above. We evaluate $G(\hat{\theta})$ by running SVRG long enough (more than 500 effective passes). Clearly the objective gap of SVRG decreases geometrically, which validates our theoretic findings.  We also observe that when $r$ is larger, SVRG converges slower, compared with  smaller $r$. This agrees with our theorem, as $r$ affects the value of $ \bar{\sigma}$ and hence the contraction factor $ \alpha$. In particular, small $r$ leads to small $\alpha$ thus the algorithm enjoys a faster convergence speed.   The composite gradient method, which uses full information at each iteration, converges linearly in (a) and (b) but with a slower rate. This agrees with the common phenomenon that stochastic variance reduction methods typically converges faster (w.r.t.\ the number of passes of datasets). In (c) and (d), its performance deteriorate significantly due to the large condition number when $b$ is not zero. The optimality gaps of SGD and RDA decrease slowly, indicating lack of linear convergence, due to   large variance of  gradients.  
 We make one interesting observation about SAG: it has a similar performance to that of SVRG in our setting, strongly suggesting that it may be possible to establish linear convergence of SAG under the RSC condition. However, we stress that the goal of the experiments is to validate our analysis of SVRG, rather than comparing SVRG with SAG. 
 
 \subsubsection{Group Lasso}
 We now report results on group sparsity case, in particular the empirical performance of different algorithms to solve Group Lasso. Similar to the above example, we have $p=5000$ and $n=2500$ and each feature is generated from the normal distribution $N(0,\Sigma)$, where $\Sigma_{ii}=1$ and $\Sigma_{ij}=b, i\neq j.$ The cardinality of the non-zero group is $s_{\mathcal{G}}$, and the size of each group is $q$ . In Figure \ref{gp_lasso}, we report results on different settings of  cardinality  $s_{\mathcal{G}}$ and the covariance matrix $\Sigma$ and $q$.
 \begin{figure}
 	\begin{subfigure}[b]{0.5\textwidth}
 		\includegraphics[width=\textwidth]{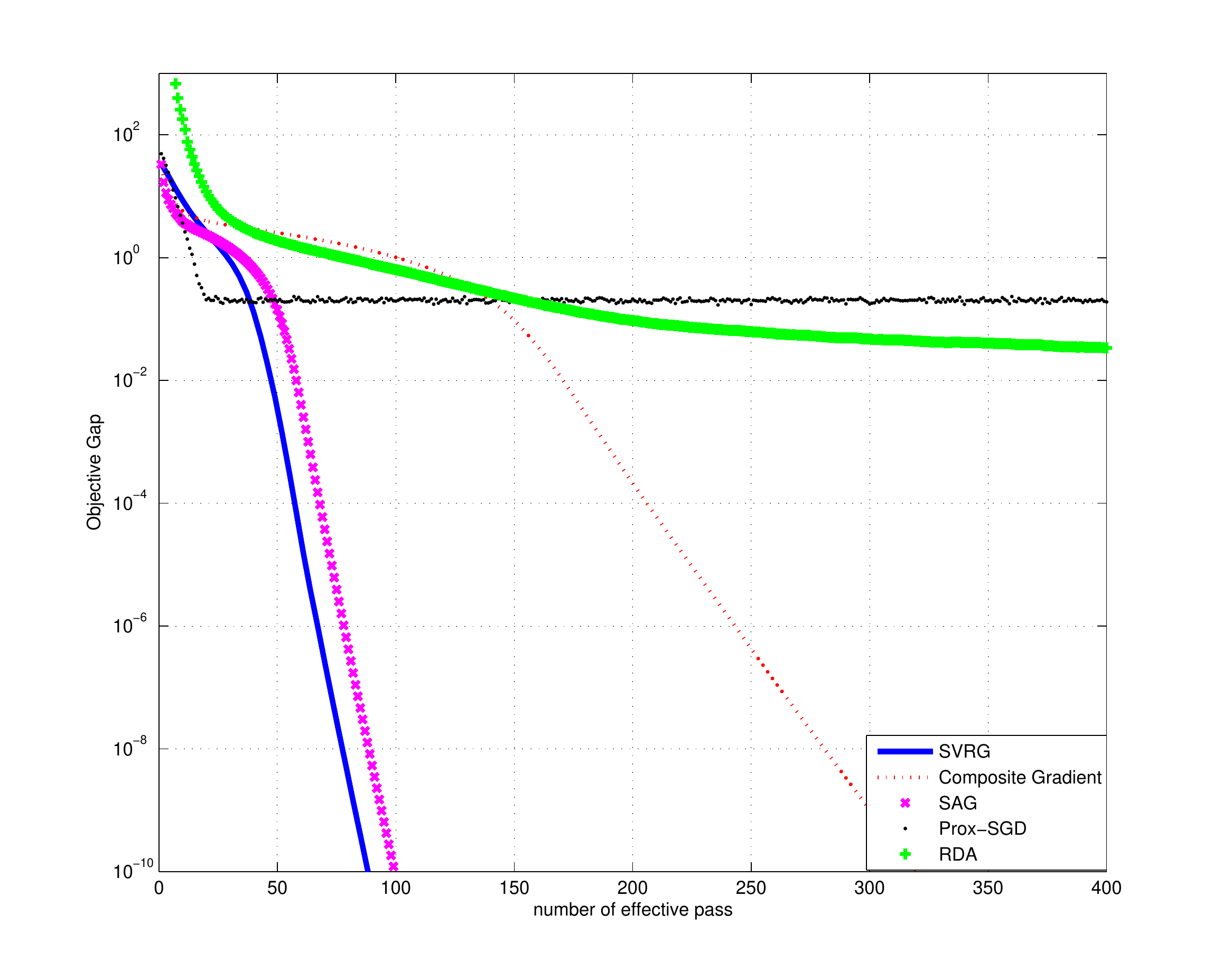}
 		\caption{$s_{\mathcal{G}}$=10,q=10, b=0}
 	\end{subfigure}
 	\begin{subfigure}[b]{0.5\textwidth}
 		\includegraphics[width=\textwidth]{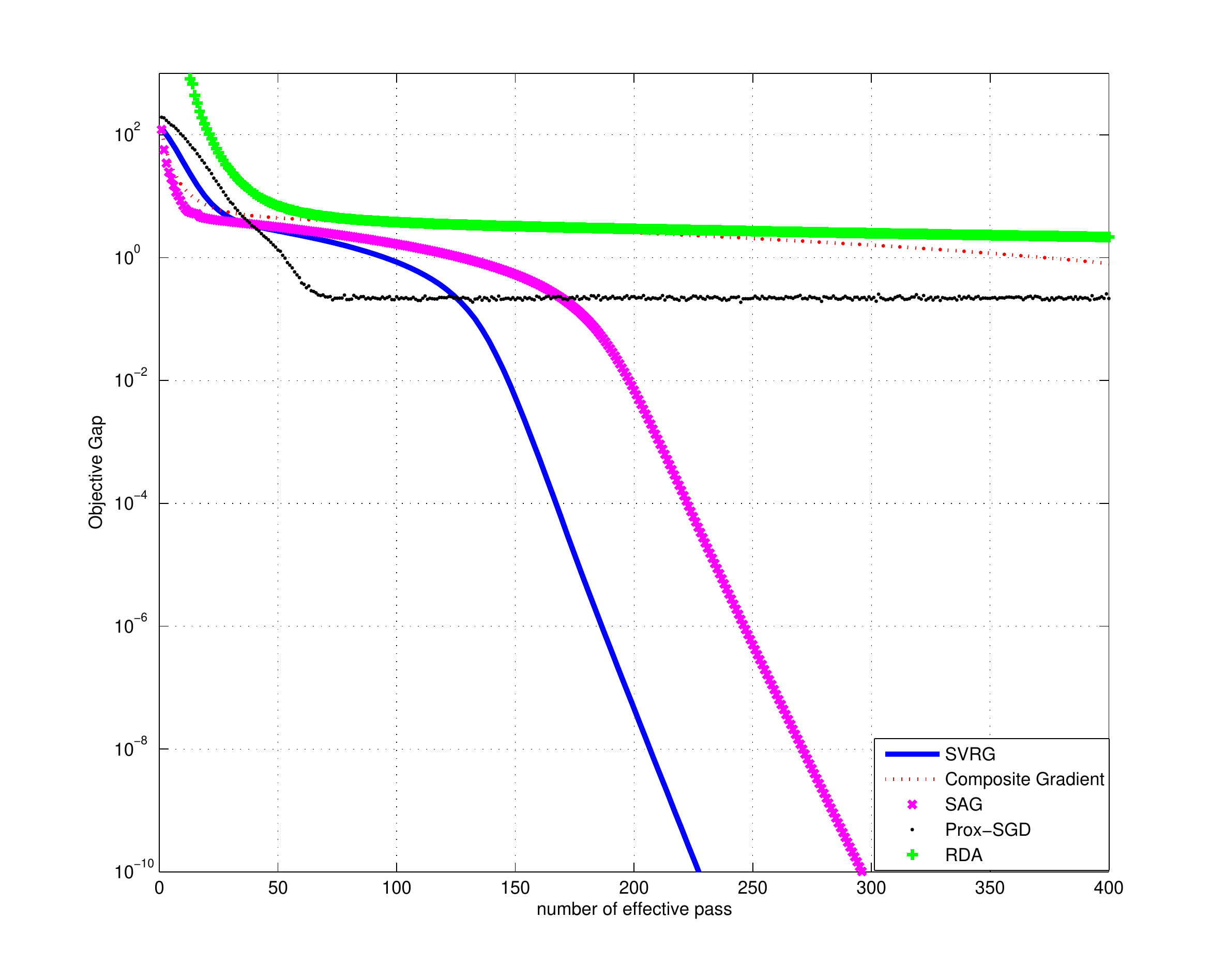}
 		\caption{$s_{\mathcal{G}}$=20, q=20, b=0}
 	\end{subfigure}
 	\begin{subfigure}[b]{0.5\textwidth}
 		\includegraphics[width=\textwidth]{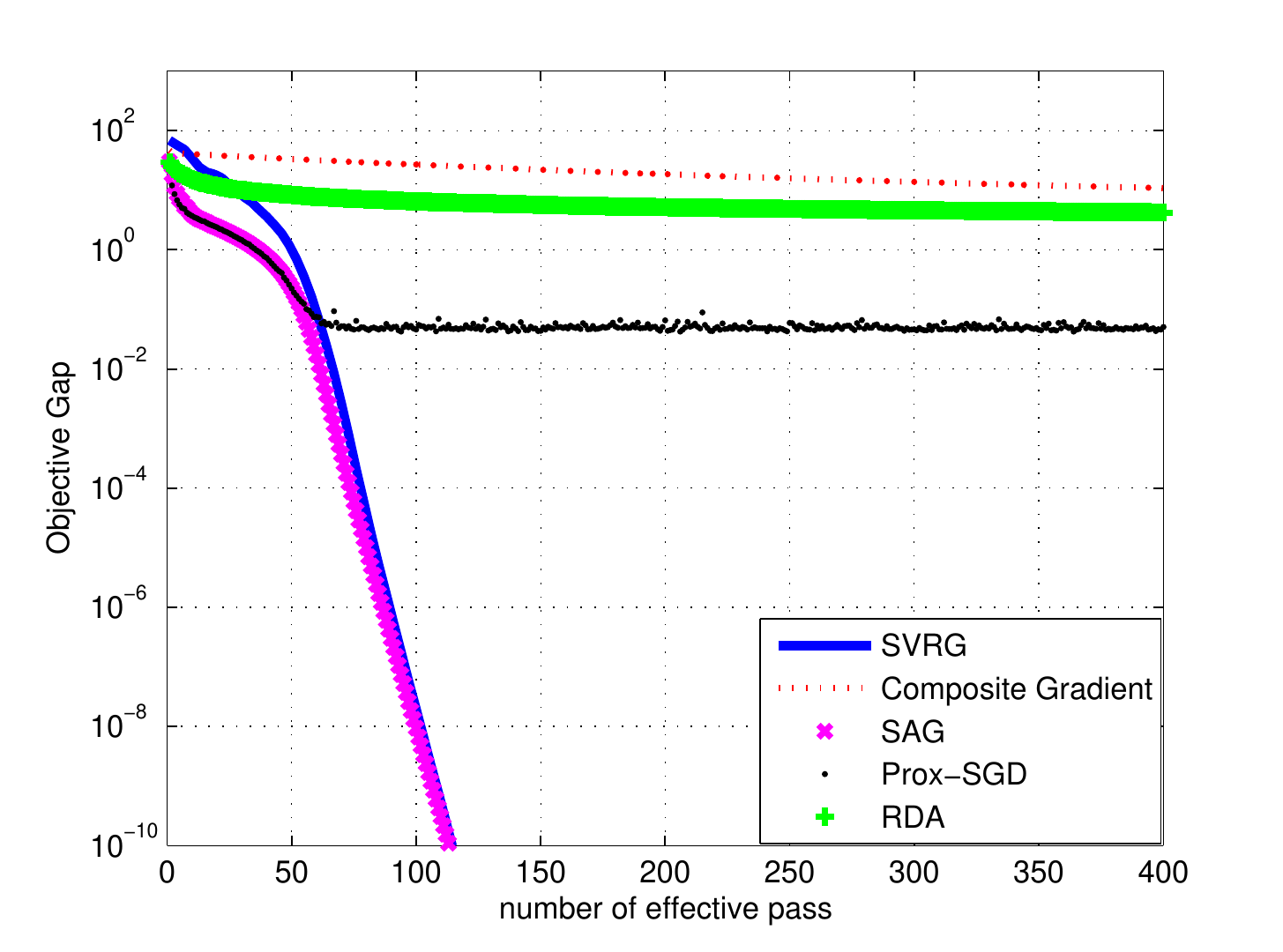}
 		\caption{$s_{\mathcal{G}}$=10, q=10, b=0.1}
 	\end{subfigure}~~~
 	\begin{subfigure}[b]{0.5\textwidth}
 		\includegraphics[width=\textwidth]{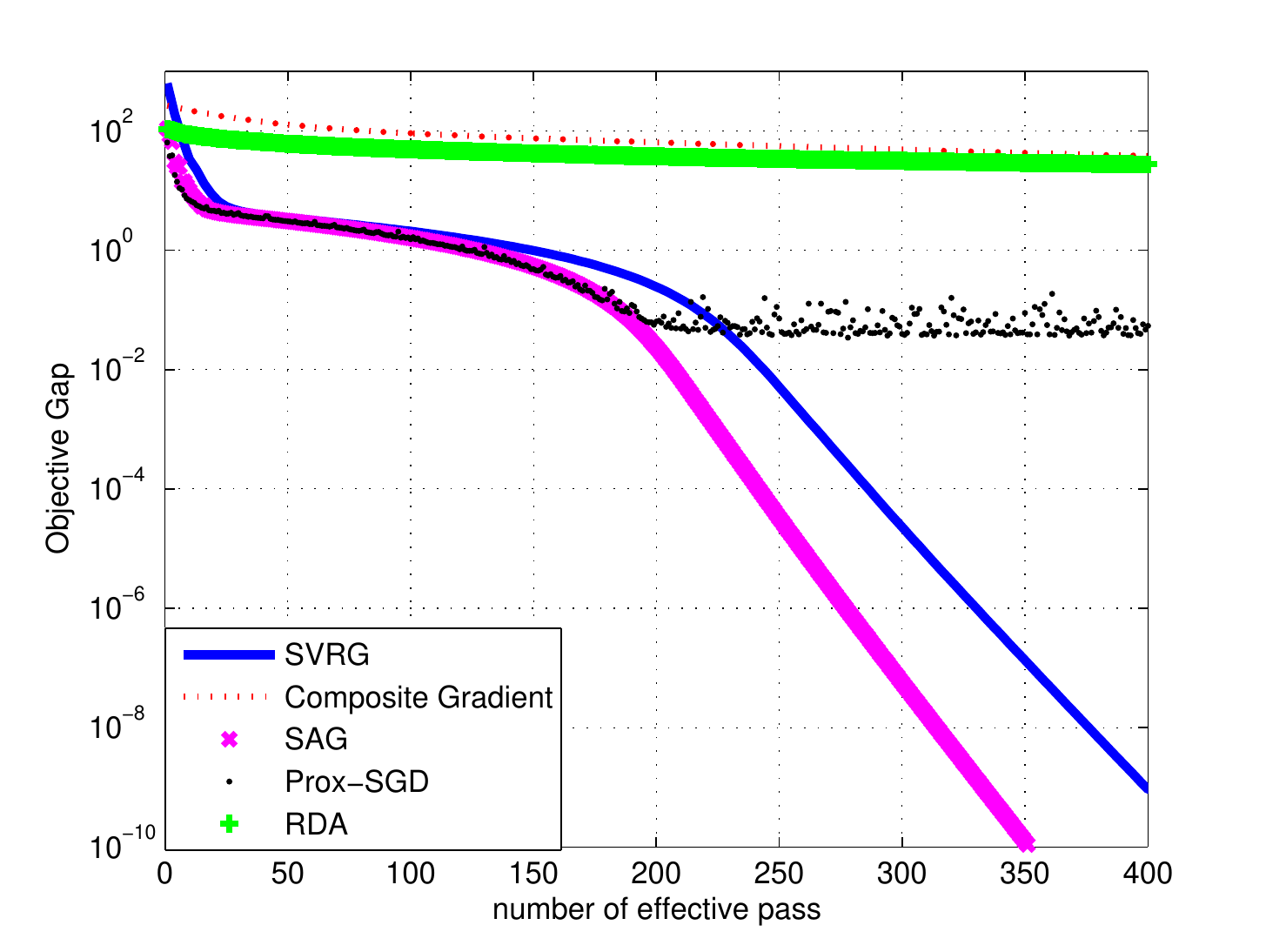}
 		\caption{$s_{\mathcal{G}}$=20, q=20, b=0.4}
 	\end{subfigure}
 	\caption{Comparison between five algorithms on group Lasso. The x-axis is the number of pass over the dataset. y-axis is the objective gap $G(\theta_k)-G(\hat{\theta})$ with  log scale. In figure (a), $s_{\mathcal{G}}$=10, q=10, b=0 . In figure (b), $s_{\mathcal{G}}$=20, q=20, b=0. In figure (c), $s_{\mathcal{G}}$=10, q=10, b=0.1. In figure (d), $s_{\mathcal{G}}$=20, q=20, b=0.4 . } \label{gp_lasso}
 \end{figure}
 In (a), similar to the   Lasso case, SVRG and SAG converge with linear rates, and  have similar performance. On the other hand, SGD and RDA converge slowly  due to the variance of the gradient. In (b), we observe that composite gradient method converge much slower. It is possibly because  the contraction factor of composite gradient method is close to 1 in this setting as the $\theta^*$  becomes less sparse.  In (c) and (d) the composite gradient method does not work due to the large condition number.

 \subsubsection{Corrected Lasso}
 
 We generate data as follows: $ y_i=x_i^T\theta^*+\xi_i$, where each data point $x_i\in \mathbb{R}^p$ is drawn from normal distribution $N(0,I)$, and the noise $ \xi_i$ is   drawn from $N(0,1)$. The coefficient $\theta^*$ is sparse with cardinality $r$, where the non-zero coefficient equals to $\pm 1$ generated from the Bernoulli distribution with parameter $0.5$. We set covariance matrix $\Sigma_w=\gamma_w I$. We choose $\lambda=0.05$ in the formulation. The result is presented in Figure \ref{fig:corrected_lasso}.
 \begin{figure}
 	\begin{subfigure}[b]{0.45\textwidth}
 		\centering
 		\includegraphics[width=\textwidth]{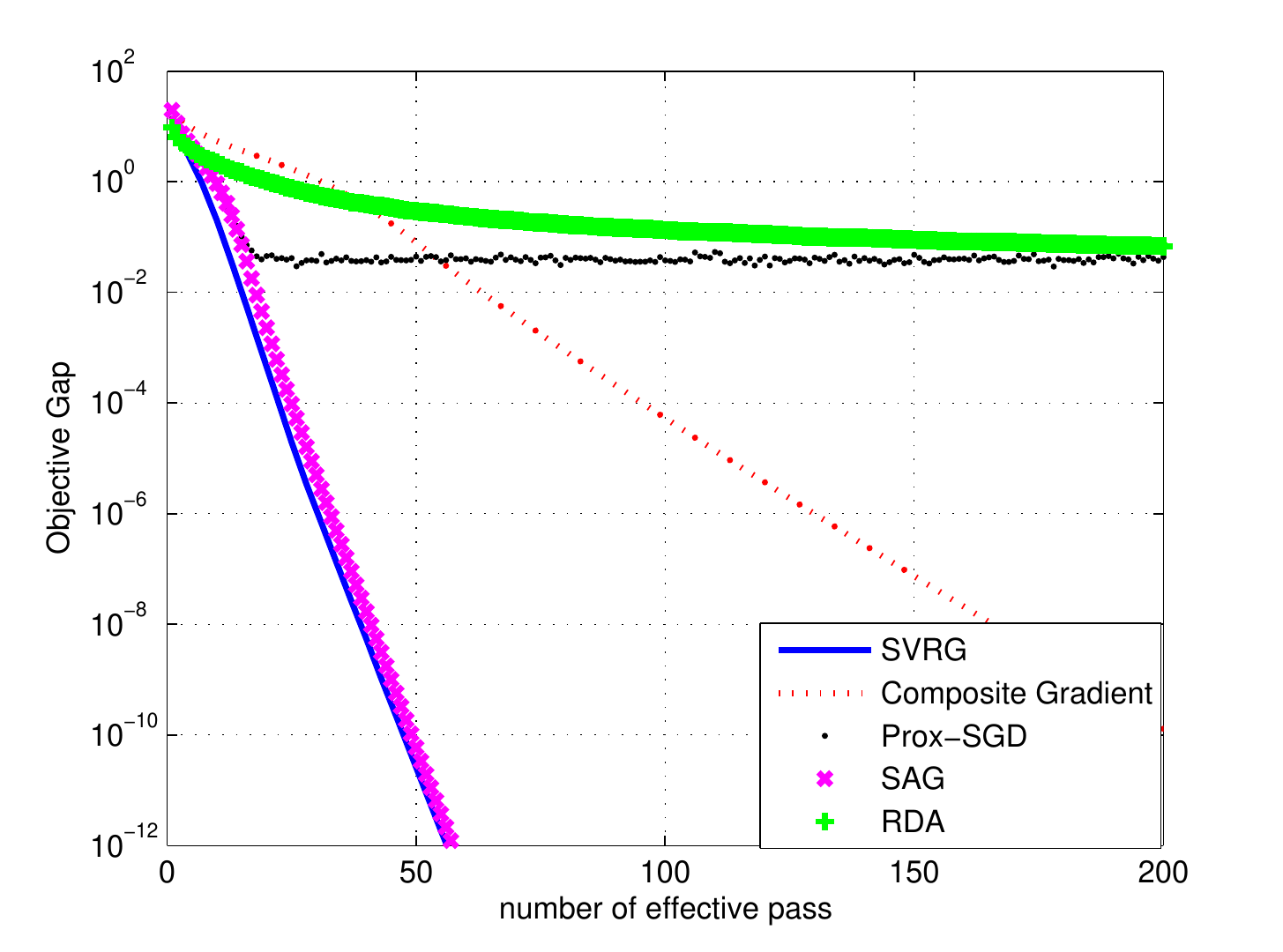}
 		\caption{$n=2500, p=3000, r=50, \gamma_w=0.05$}
 	\end{subfigure}
 	\begin{subfigure}[b]{0.45\textwidth}
 		\centering
 		\includegraphics[width=\textwidth]{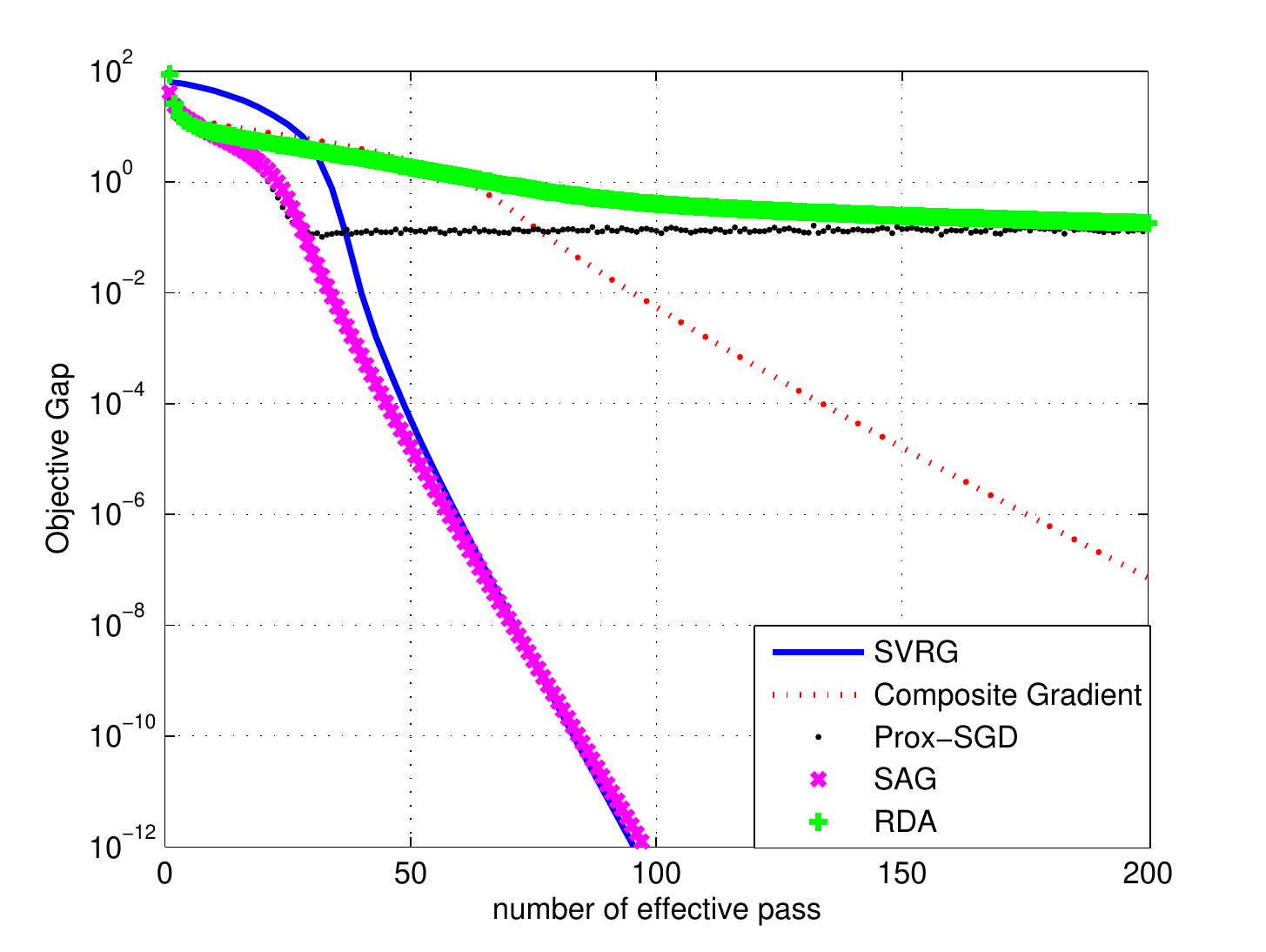}
 		\caption{$n=2500,p=5000, r=100,\gamma_w=0.1$}
 	\end{subfigure}
 	\caption{Results on Corrected Lasso. The x-axis is the number of pass over the dataset. y-axis is the objective gap $G(\theta_k)-G(\hat{\theta})$ with  log scale. We try two different settings. In the first figure $n=2500, p=3000, r=50, \gamma_w=0.05$.in the second figure $n=2500,p=5000, r=100,\gamma_w=0.1$.}\label{fig:corrected_lasso}
 \end{figure}
 
 In both figures (a) and (b), SVRG, SAG and Composite gradient converge linearly. According to our theory, as $\bar{\sigma}$ in figure (a) is larger than that in figure (b), and $\frac{\gamma_w}{\bar{\sigma}}$ in figure (a) is smaller than that in (b),   SVRG should converge faster in the setting of figure (a), which matches our simulation result. SGD and RDA have large optimality gaps. 
 \subsubsection{SCAD}
 The way to generate data is same with Lasso. Here  $x_i\in \mathbb{R}^p$ is drawn from normal distribution $N(0,2I)$ (Here We choose $2I$ to satisfy the requirement of $\bar{\sigma}$ and $\mu$ in our Theorem, although if we choose $N(0, I)$, the algorithm still works. ). $\lambda=0.05$ in the formulation.  We present the result in Figure \ref{fig:SCAD}, for two settings on $n$, $p$, $r$, $\zeta$. Note that $\bar{\sigma}\geq 0.5$ and $\frac{1}{\zeta-1}\leq 0.5$ in both cases, thus our theorem asserts that SVRG  converge linearly under appropriate choices of $\beta$ and $m$.
 
 \begin{figure}
 	\begin{subfigure}[b]{0.45\textwidth}
 		\centering
 		\includegraphics[width=\textwidth]{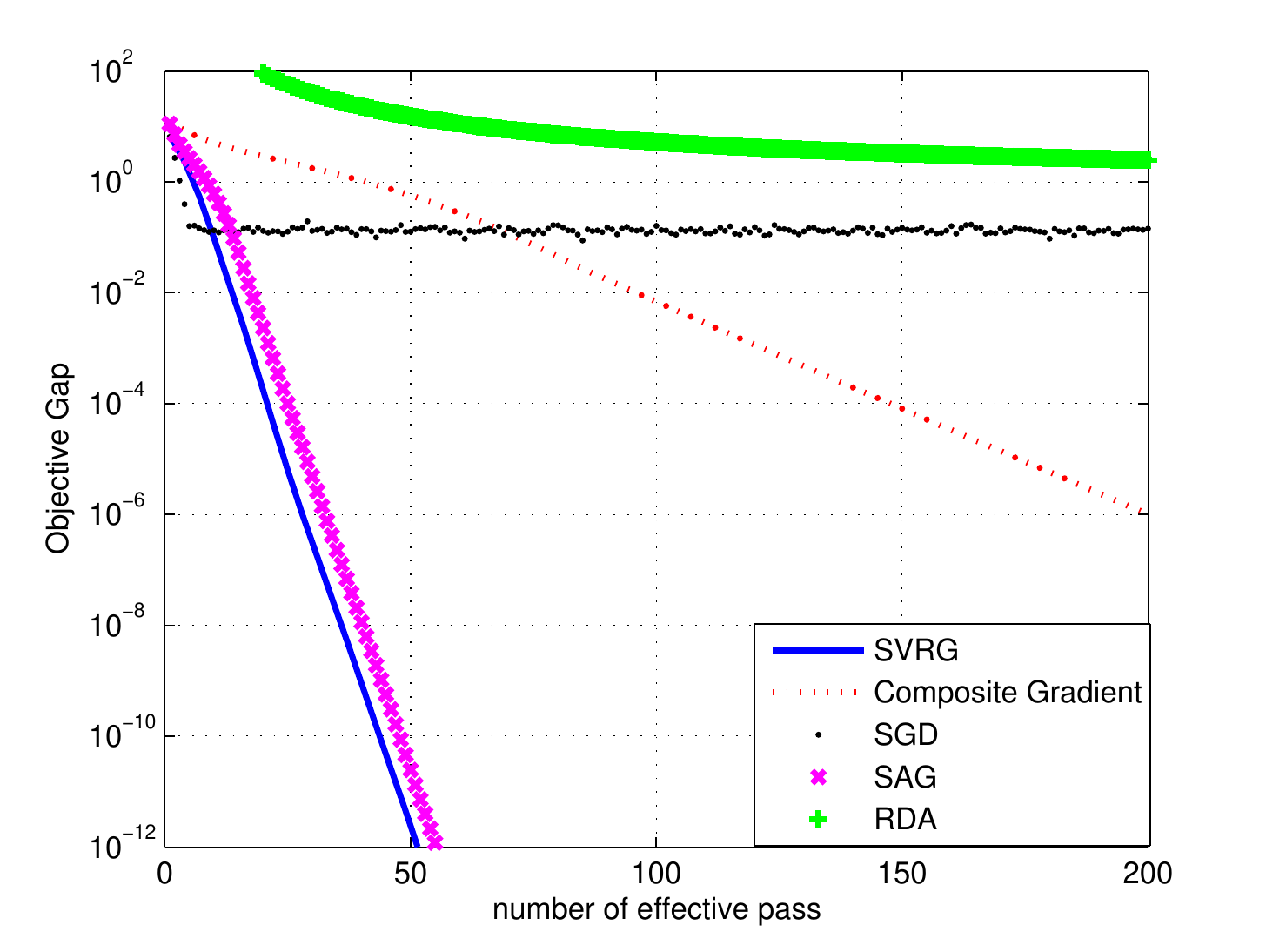}
 		\caption{$n=3000, p=2500, r=30, \zeta=4.5 $ }
 	\end{subfigure}
 	\begin{subfigure}[b]{0.45\textwidth}
 		\centering
 		\includegraphics[width=\textwidth]{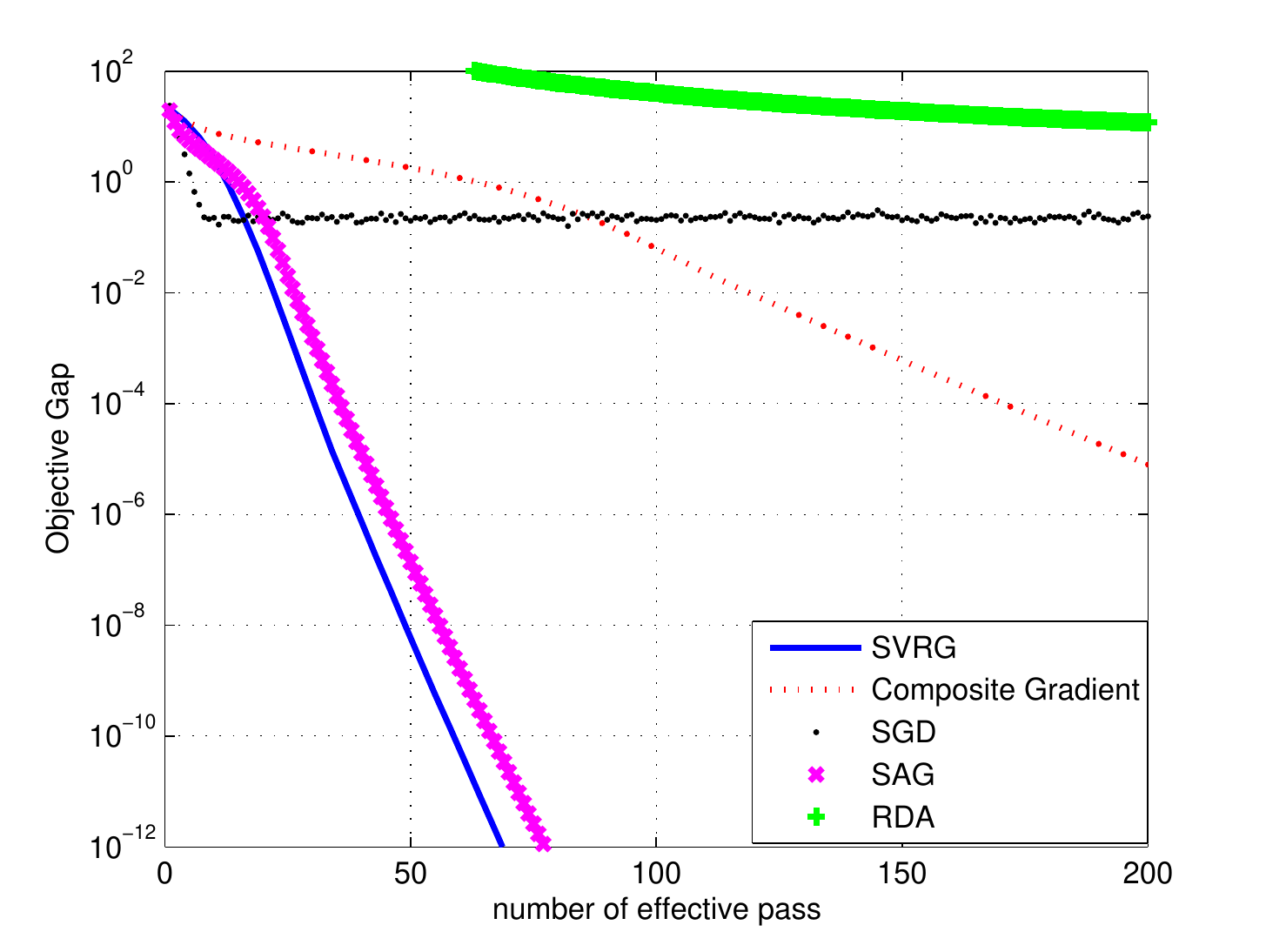}
 		\caption{$n=2500, p=5000, r=50, \zeta=3.7 $  }
 	\end{subfigure}
 	\caption{Results on SCAD. The x-axis is the number of pass over the dataset. y-axis is the objective gap $G(\theta_k)-G(\hat{\theta})$ with  log scale. } \label{fig:SCAD}
 \end{figure} 
 
 We observe from Figure \ref{fig:SCAD} that in both cases, SVRG, SAG converges with linear rates and have similar performance. The composite gradient method also converges linearly but with a slower speed. SGD and RDA have large optimality gaps. 
 
 \subsection{Real data}
 This section presents results of several numerical experiments on real datasets.
 
 \subsubsection{Sparse classification problem}
 The first problem we consider is sparse classification. In particular, we apply logistic regression with $l_1$ regularization on rcv1  ($n=20242,d=47236$) \cite{lewis2004rcv1} and sido0 ($n=12678, d=4932$) \cite{SIDO} datasets for the binary classification problem, i.e.,
 $$ G(\theta)=\frac{1}{n}\sum_{i=1}^{n}\log (1+\exp(-y_i \langle x_i,\theta \rangle)+\lambda \|\theta\|_1. $$
 \begin{figure}
 	\centering
 	\includegraphics[width=0.5\textwidth]{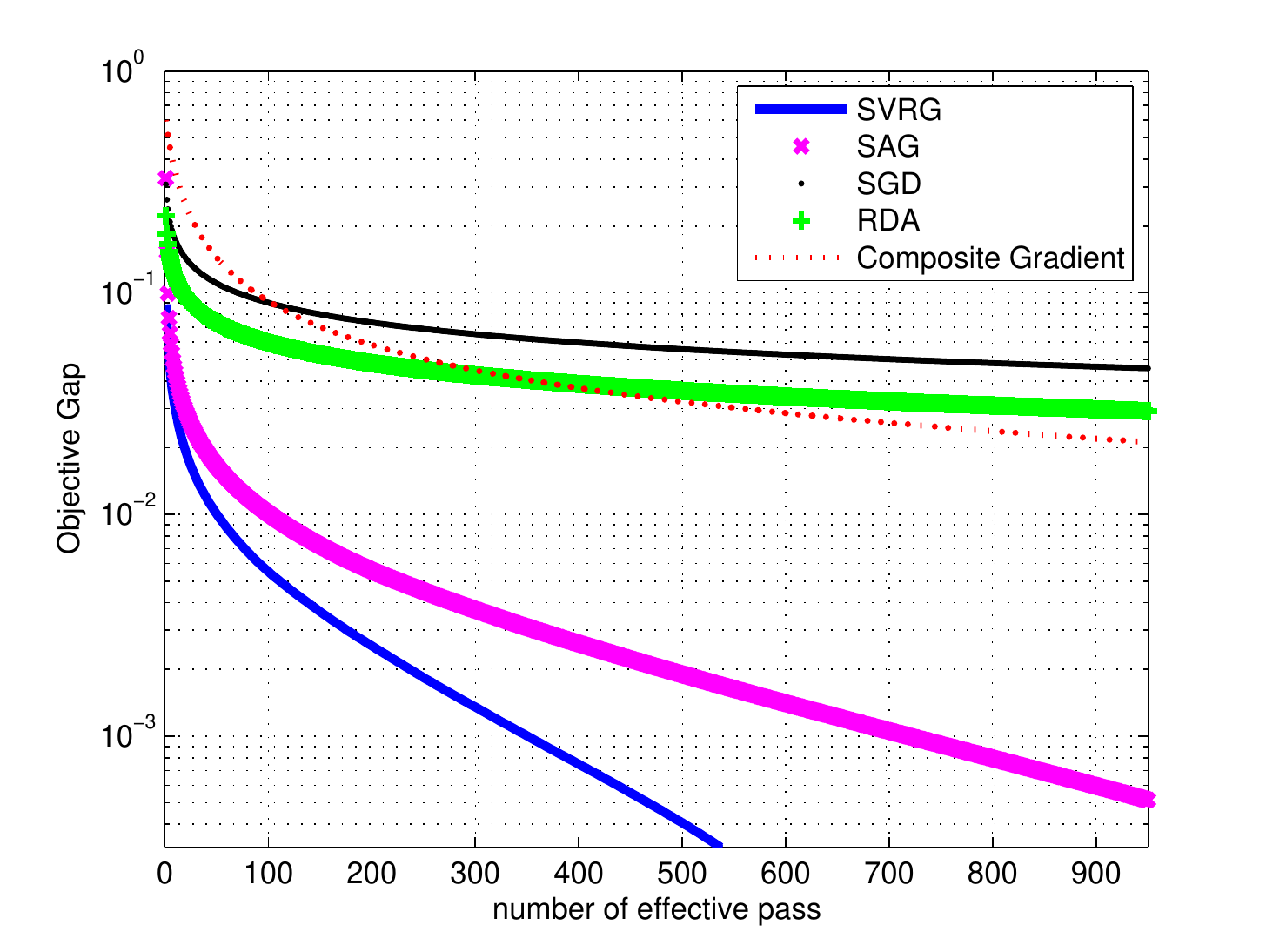}
 	\caption{Different methods on rcv1 dataset.}\label{rcv1}
 \end{figure}
 We choose $\lambda=2\cdot10^{-5}$ in  rcv1 dataset and $\lambda=10^{-4}$ in sido0 dataset suggested in \citet{xiao2014proximal}.
 
 Figure \ref{rcv1} shows  the performance of five algorithms on  rcv1 dataset. The x-axis is the number of passes over the dataset, and the y-axis is the optimality gap in log-scale. In the experiment we choose $m=2n$ for SVRG. Among all five algorithms, SVRG performs best followed by SAG. The composite gradient method does not perform well in this  dataset. RDA and SGD converge slowly and the error of them remains large even after 1000 passes of the entire dataset.

 \begin{figure}[t]
 	\centering
 	\includegraphics[width=0.5\textwidth]{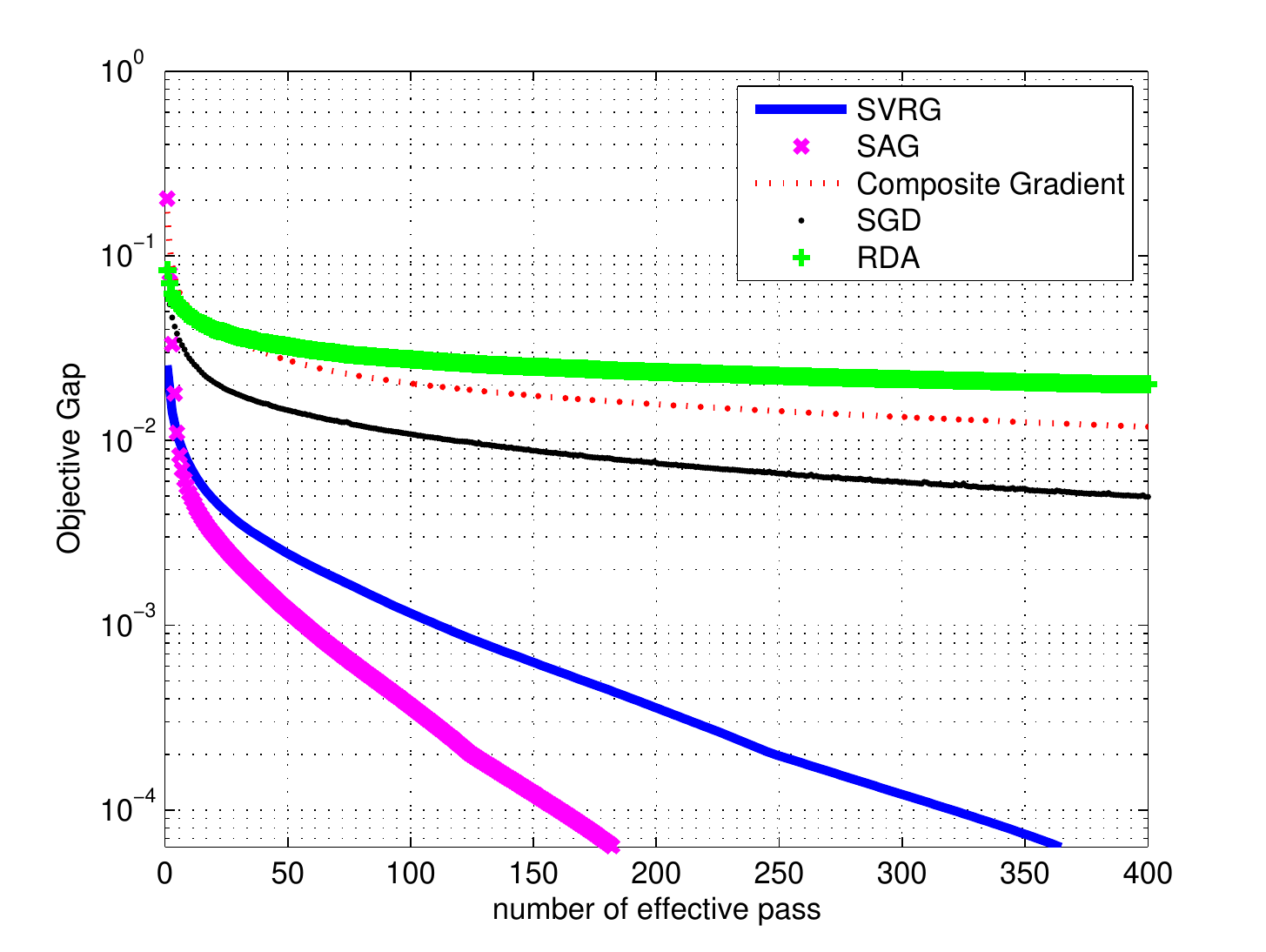}
 	\caption{Different methods on sido0 dataset.}\label{sido0}
 \end{figure}
 Figure \ref{sido0} reports results on sido0 dataset.  On this dataset, SAG outperforms SVRG. We also observe that SGD outperforms  composite gradient. The RDA converges with the slowest rate.
 
 \subsubsection{Group sparse regression}
 
 We consider a group sparse regression problem on the Boston Housing dataset \cite{uci:2013}.  As suggested in \citet{swirszcz2009grouped,xiang2014simultaneous}, to take into account the non-linear relationship between variables and the response, up to third-degree polynomial expansion is applied on each feature. In particular, terms $x$, $x^2$ and $x^3$ are grouped together. We consider group Lasso model on this problem with $\lambda=0.1$. We choose the setting $m=2n$ in SVRG.
 
 \begin{figure}[t]
 	\centering
 	\includegraphics[width=0.5\textwidth]{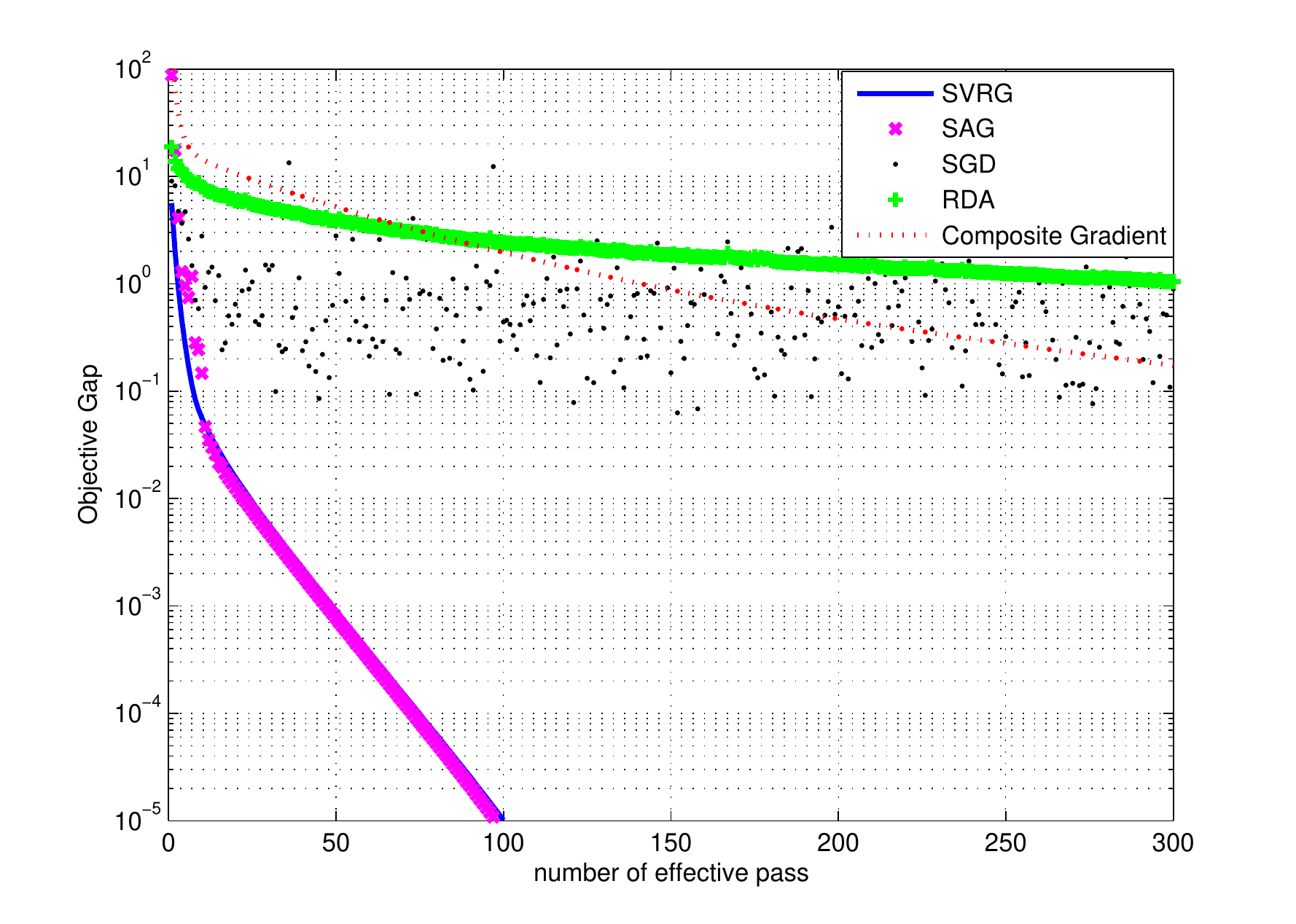}
 	\caption{Five different algorithms on Boston Housing dataset.}\label{boston_housing}
 \end{figure}
 
 In Figure \ref{boston_housing}, we show the objective gap of various algorithms versus the number of passes over the dataset. SVRG and SAG have almost identical performance. SGD fails to converge~--~the optimality gap oscillates between 0.1 and 1. Both the composite gradient method and RDA converge slowly.

 \section{Conclusion and future work}
 In this paper, we analyzed a state-of-art stochastic first order optimization algorithm SVRG where the objective function is not strongly convex, or even non-convex. We established linear convergence of SVRG exploiting the concept of {\em restricted strong convexity}. Our setup naturally includes several important statistical models such as Lasso,  group sparse regression and SCAD, to name a few. We further validated our theoretic findings with numerical experiments on synthetic and real datasets. 
 
 \appendix
 
 \section{Phase transition of linear rate and sub-linear rate in Lasso}\label{app.PT}
 
 We generate data as follows: $ y_i=x_i^T\theta^*+\xi_i$, where each data point $x_i\in \mathbb{R}^p$ is drawn from normal distribution $N(0,I)$, and the noise $ \xi_i$ is   drawn from $N(0,1)$. The coefficient $\theta^*$ is sparse with cardinality $r$, where the non-zero coefficient equals to $\pm 1$ generated from the Bernoulli distribution with probability $0.5$.  The sample size is $n=2500$, and the dimension of problem is $p=5000$.
 
 In Figure~\ref{phase}, we increase $r$ from 500 to 1500 and plot the convergence rate of SVRG. We observe a phase transition from linear convergence to sublinear convergence happening between $r=750$ and $r=1000$. This phenomena is captured  by our theorem: When $r$ is too large, the requirement $\bar{\sigma}=\frac{1}{2}\sigma_{\min}(\Sigma)- c_1\nu(\Sigma)\frac{ r\log p}{n}\geq 0 $ breaks.
 
 \begin{figure}[h]
 	\centering
 	\includegraphics[width=0.45\textwidth]{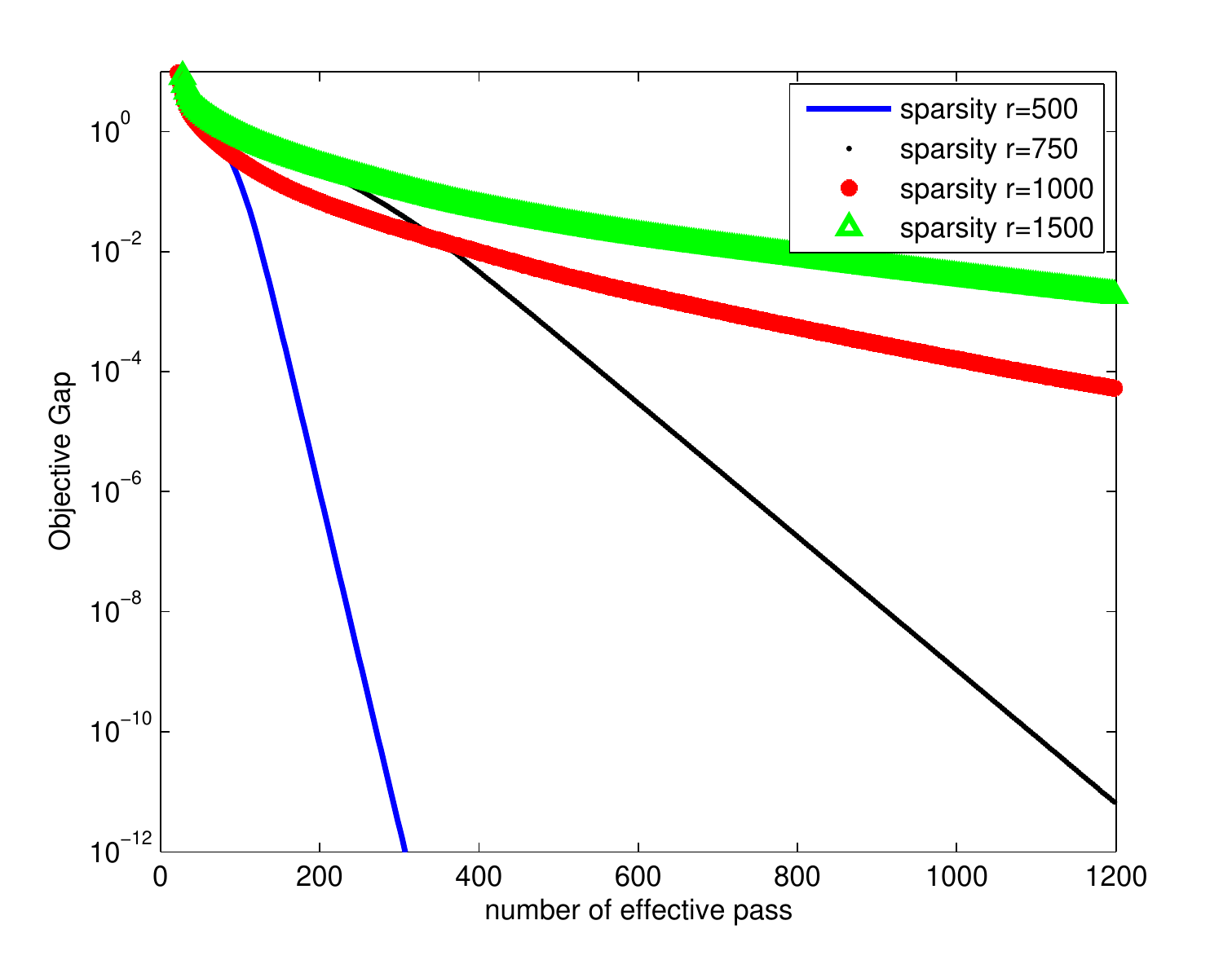}
 	\caption{ Phase transition of linear rate. The x-axis is the number of pass over the dataset. y-axis is the objective gap $G(\theta_k)-G(\hat{\theta})$ with  log scale.}\label{phase}
 \end{figure}



 \section{Proofs}\label{app.proof}
 We provide in this section proofs to all results presented.
 \subsection{SVRG with convex objective function}
 Remind the objective function we aim to optimize is 
 \begin{equation}\label{app.obj2}
 \begin{split}
 \min_{\psi(\theta)\leq \rho} F(\theta)+\lambda\psi(\theta)= \frac{1}{n}\sum_{i=1}^{n} f_i(\theta)+\lambda\psi(\theta).
 \end{split}
 \end{equation}
 We denote 
 $$ \hat{\theta}=\arg\min_{\psi(\theta)\leq \rho} F(\theta)+\lambda\psi(\theta). $$
 The following technical lemma is well-known in SVRG to bound the variance of the modified stochastic gradient $v_k$. It is indeed Corollary 3 in \cite{xiao2014proximal}, which we present here for completeness. 
 \begin{lemma}\label{lemma.variance}
 	Consider   $v_{k-1}$ defined in the algorithm 1. Conditioned on $\theta_{k-1}$, we have $\mathbb{E}v_{k-1}=\nabla F(\theta_{k-1}) $, and 
 	$$ \mathbb{E}\|v_{k-1}-\nabla F(\theta_{k-1})\|_2^2\leq 4L[G(\theta_{k-1})-G(\hat{\theta})+G(\tilde{\theta})-G(\hat{\theta})]. $$
 \end{lemma}

 \begin{lemma}\label{lemma.cone}
 	
 	Suppose that $F(\theta)$ is convex and $\psi(\theta)$ is decomposable with respect to $(M,\bar{M})$, if we choose $\lambda\geq 2\psi^*(\nabla F(\theta^*))$, $\psi(\theta^*)\leq \rho $ ,  define the error $\Delta^*=\hat{\theta}-\theta^*$, then we have the following condition holds, 
 	$$ \psi(\Delta^*_{\bar{M}^\perp})\leq 3 \psi(\Delta^*_{\bar{M}})+4\psi(\theta^*_{M^\perp}), $$
 	which further implies $ \psi (\Delta^*)\leq \psi(\Delta^*_{\bar{M}^\perp})+\psi(\Delta^*_{\bar{M}})\leq 4 \psi(\Delta^*_{\bar{M}})+4\psi(\theta^*_{M^\perp}).$
 \end{lemma}
 
 \begin{proof}
 	Using the optimality of $\hat{\theta}$, we have
 	$$ F(\hat{\theta})+\lambda \psi(\hat\theta)-F(\theta^*)-\lambda\psi(\theta^*)\leq 0.$$
 	
 	So we have
 	$$ \lambda \psi(\theta^*)-\lambda \psi(\hat{\theta})\geq F(\hat{\theta})-F(\theta^*)\geq \langle \nabla F(\theta^*), \hat{\theta}-\theta^*  \rangle\geq -\psi^*(\nabla F(\theta ^*)) \psi (\Delta^*), $$
 	where the second inequality holds from the convexity of $F(\theta)$, and the third holds using Holder inequality.
 	
 	Using triangle inequality, we have
 	
 	$$ \psi(\Delta^*)\leq \psi(\Delta^*_{\bar{M}})+\psi(\Delta^*_{\bar{M}^{\perp}}).$$
 	
 	So 
 	\begin{equation}\label{equ:lemma_cone}
 	\lambda \psi(\theta^*)-\lambda \psi(\hat{\theta})\geq -\psi^*(\nabla F(\theta ^*)) (\psi(\Delta^*_{\bar{M}})+\psi(\Delta^*_{\bar{M}^{\perp}}) )
 	\end{equation}

 	Notice $$\hat{\theta}=\theta^*+\Delta^*= \theta^*_{M}+\theta^*_{M^{\perp}}+\Delta^*_{\bar{M}}+\Delta^*_{\bar{M}^{\perp}},$$
 	which leads to
 	\begin{equation}
 	\begin{split}
 	\psi  (\hat{\theta})-\psi(\theta^*)& \overset{(a)}{\geq}  \psi (\theta^*_{M}+\Delta^*_{\bar{M}^{\perp}})-\psi(\theta^*_{M^{\perp}})-\psi(\Delta^*_{\bar{M}})-\psi(\theta^*)\\
 	& \overset{(b)}{= }\psi (\theta^*_{M})+\psi (\Delta^*_{\bar{M}^{\perp}})-\psi(\theta^*_{M^{\perp}})-\psi(\Delta^*_{\bar{M}})-\psi(\theta^*)\\
 	& \overset{(c)}{\geq} \psi (\theta^*_{M})+\psi (\Delta^*_{\bar{M}^{\perp}})-\psi(\theta^*_{M^{\perp}})-\psi(\Delta^*_{\bar{M}})-\psi(\theta^*_{M})-\psi(\theta^*_{M^{\perp}})\\
 	&\geq \psi (\Delta^*_{\bar{M}^{\perp}})-2\psi(\theta^*_{M^{\perp}})-\psi(\Delta^*_{\bar{M}}),
 	\end{split}
 	\end{equation}
 	where $(a)$ and $(c)$ holds from the triangle inequality, $(b)$ uses the decomposability of $\psi(\cdot)$.

 	Substitute  left hand side of \ref{equ:lemma_cone} by above result and use the assumption that $\lambda\geq 2\psi^*(\nabla F(\theta^*))$, we have 
 	$$-\frac{\lambda}{2} (\psi(\Delta^*_{\bar{M}})+\psi(\Delta^*_{\bar{M}^{\perp}}) )+\lambda (\psi (\Delta^*_{\bar{M}^{\perp}})-2\psi(\theta^*_{M^{\perp}})-\psi(\Delta^*_{\bar{M}}))\leq 0$$
 	which implies 
 	$$ \psi(\Delta^*_{\bar{M}^\perp})\leq 3  \psi(\Delta^*_{\bar{M}})+4\psi(\theta^*_{M^{\perp}}).$$
 \end{proof}

 \begin{lemma}\label{lemma.cone_optimization}
 	$F(\theta)$ is convex and $\psi(\theta)$ is decomposable with respect to $(M,\bar{M})$, if we choose $\lambda\geq 2\psi^*(\nabla F(\theta^*))$, $\psi(\theta^*)\leq \rho $ and suppose there exist a time step $S>0$ and a given tolerance $\epsilon'$ such that for all $s>S$, $ G(\theta^s)-G(\hat{\theta})\leq \epsilon' $ holds,  then for the error $\Delta^s=\theta^s-\theta^* $ we have
 	
 	$$ \psi(\Delta^s_{\bar{M}^\perp})\leq 3 \psi(\Delta^s_{\bar{M}})+4\psi(\theta^*_{M^\perp})+ 2\min\{\frac{\epsilon'}{\lambda},\rho\}, $$
 	which implies $$ \psi(\Delta^s)\leq 4 \psi(\Delta^s_{\bar{M}})+4\psi(\theta^*_{M^\perp})+ 2\min\{\frac{\epsilon'}{\lambda},\rho\}. $$
 \end{lemma}
 \begin{proof}
 	First notice $ G(\theta^s)-G(\theta^*)\leq \epsilon' $ holds by assumption since $G(\theta^*)\geq G(\hat{\theta}).$
 	So we have
 	$$ F(\theta^s)+\lambda \psi(\theta^s)-F(\theta^*)-\lambda\psi(\theta^*)\leq \epsilon'.$$
 	Follow same steps in the proof of Lemma \ref{lemma.cone}, we have
 	
 	$$ \psi(\Delta^s_{\bar{M}^\perp})\leq 3 \psi(\Delta^s_{\bar{M}})+4\psi(\theta^*_{M^\perp})+ 2\frac{\epsilon'}{\lambda}. $$
 	
 	Notice  $\Delta^s=\Delta^s_{\bar{M}^\perp}+\Delta^s_{\bar{M}} $ so $\psi(\Delta^s_{\bar{M}^\perp})\leq \psi(\Delta^s_{\bar{M}})+\psi(\Delta^s) $ using the triangle inequality.\\
 	Then use the fact that $\psi(\Delta^s)\leq \psi (\theta^*)+\psi(\theta^s)\leq 2\rho $, we establish
 	$$ \psi(\Delta^s_{\bar{M}^\perp})\leq 3 \psi(\Delta^s_{\bar{M}})+4\psi(\theta^*_{M^\perp})+ 2\min\{\frac{\epsilon'}{\lambda},\rho\}. $$
 	The second statement follows immediately from $\psi(\Delta^s) \leq \psi(\Delta^s_{\bar{M}^\perp})+\psi(\Delta^s_{\bar{M}}) $.
 \end{proof}
 Using the above two lemmas we now prove modified restricted convexity on $ G(\theta^s)-G(\hat{\theta}) $.
 
 \begin{lemma}\label{lemma.RSC_cone}
 	Under the same assumptions of Lemma \ref{lemma.cone_optimization}, we have
 	\begin{equation}\label{RSC_cone}
 	G(\theta^s)-G(\hat{\theta})\geq \left(\frac{\sigma}{2}-32\tau_\sigma H^2(\bar{M} )\right)\|\hat{\Delta}^s \|_2^2- \epsilon^2(\Delta^*,M,\bar{M}),
 	\end{equation}
 	where $ \epsilon^2(\Delta^*,M,\bar{M})=2\tau_{\sigma} (\delta_{stat}+\delta)^2 $, $\delta=2\min\{\frac{\epsilon'}{\lambda},\rho\} $, and $\delta_{stat}= H(\bar{M})\|\Delta^*\|_2+8\psi(\theta^*_{M^\perp}) $.
 \end{lemma} 
 
 \begin{proof}
 	At the beginning of the proof, we show a simple fact on $\hat{\Delta}^s=\theta^s-\hat{\theta}$. Notice the conclusion in Lemma \ref{lemma.cone_optimization} is on  $\Delta^s$, we need transfer it to $\hat{\Delta}^s$.
 	\begin{equation}
 	\begin{split}
 	\psi(\hat{\Delta}^s)&\leq \psi (\Delta^s)+\psi(\Delta^*)\\
 	&\leq  4 \psi(\Delta^s_{\bar{M}})+4\psi(\theta^*_{M^\perp})+ 2\min\{\frac{\epsilon'}{\lambda},\rho\}+4 \psi(\Delta^*_{\bar{M}})+4\psi(\theta^*_{M^\perp})\\
 	& \leq 4 H(\bar{M})\|\Delta^s\|_2+4H(\bar{M})\|\Delta^*\|_2+8\psi(\theta^*_{M^\perp})+2\min\{\frac{\epsilon'}{\lambda},\rho\},
 	\end{split}
 	\end{equation}
 	where the first inequality holds from the triangle inequality, the second inequality uses Lemma \ref{lemma.cone} and \ref{lemma.cone_optimization}, the third holds because of the definition of subspace compatibility. 
 	
 	We now use the above result to rewrite the RSC condition.
 	We know $$ F(\theta^s)-F(\hat{\theta})-\langle \nabla F(\hat{\theta}), \hat{\Delta}^s  \rangle\geq \frac{\sigma}{2} \|\hat{\Delta}^s \|_2^2-\tau_{\sigma}\psi^2(\hat{\Delta}^s) $$
 	which implies $ G(\theta^s)-G(\hat{\theta})\geq \frac{\sigma}{2} \|\hat{\Delta}^s \|_2^2-\tau_{\sigma}\psi^2(\hat{\Delta}^s)$, by using the fact that $\hat{\theta}$ is the optimal solution to the problem \ref{obj2} and $\phi(\cdot)$ is convex.
 	Notice that
 	\begin{equation}
 	\begin{split}
 	\psi(\hat{\Delta}^s)&\leq 4 H(\bar{M})\|\Delta^s\|_2+4H(\bar{M})\|\Delta^*\|_2+8\psi(\theta^*_{M^\perp})+2\min\{\frac{\epsilon'}{\lambda},\rho\}\\
 	&\leq  4 H(\bar{M})\|\hat{\Delta}^s\|_2+8H(\bar{M})\|\Delta^*\|_2+8\psi(\theta^*_{M^\perp})+2\min\{\frac{\epsilon'}{\lambda},\rho\},
 	\end{split}
 	\end{equation}
 	where the second inequality uses the triangle inequality.
 	Now use the inequality $ (a+b)^2\leq 2 a^2+2b^2 $, we upper bound  $\psi^2 (\hat{\Delta}^s)$ with
 	$$\psi^2 (\hat{\Delta}^s)\leq 32H^2(\bar{M})\|\hat{\Delta}^s\|_2^2+2\left(8H(\bar{M})\|\Delta^*\|_2+8\psi(\theta^*_{M^\perp})+2\min\{\frac{\epsilon'}{\lambda},\rho\}\right)^2.$$
 	Substitute  this upper bound into RSC, we have
 	$$ G(\theta^s)-G(\hat{\theta})\geq \left(\frac{\sigma}{2}-32\tau_\sigma H^2(\bar{M} )\right)\|\hat{\Delta}^s \|_2^2-2\tau_\sigma\left(8H(\bar{M})\|\Delta^*\|_2+8\psi(\theta^*_{M^\perp})+2\min\{\frac{\epsilon'}{\lambda},\rho\}\right)^2.$$
 	
 	Notice that by $\delta=2\min\{\frac{\epsilon'}{\lambda},\rho\} $, $\delta_{stat}= H(\bar{M})\|\Delta^*\|_2+8\psi(\theta^*_{M^\perp}),$  and $ \epsilon^2(\Delta^*,M,\bar{M})=2\tau_{\sigma} (\delta_{stat}+\delta)^2 $, we have
 	$$\epsilon^2(\Delta^*,M,\bar{M})=2\tau_\sigma\left(8H(\bar{M})\|\Delta^*\|_2+8\psi(\theta^*_{M^\perp})+2\min\{\frac{\epsilon'}{\lambda},\rho\}\right)^2.$$
 	We thus conclude
 	\begin{equation}\label{RSC_modified}
 	G(\theta^s)-G(\hat{\theta})\geq \left(\frac{\sigma}{2}-32\tau_\sigma H^2(\bar{M} )\right)\|\hat{\Delta}^s \|_2^2- \epsilon^2(\Delta^*,M,\bar{M}).
 	\end{equation}
 \end{proof}
 
 The next lemma is a simple extension of a standard property proximal operator with a constraint $\Omega$.
 \begin{lemma}\label{prox}
 	Define $prox_{h,\Omega}(x)=\arg\min_{z\in \Omega
 	} h(z)+\frac{1}{2}\|z-x\|_2^2$, where $\Omega$ is a convex compact set, then $\|prox_{h,\Omega}(x)-prox_{h,\Omega}(y)\|_2\leq \|x-y\|_2 .$
 \end{lemma}
 \begin{proof}
 	Define $u=prox_{h,\Omega}(x)$ and $ v=prox_{h,\Omega}(y)$. Using optimality of $u$ we have
 	$$\langle\partial h(u)+u-x, \tilde{u}-u \rangle \geq 0 \quad\forall \tilde{u}\in \Omega. $$
 	We choose $\tilde{u}=v$ and obtain $\langle\partial h(u)+u-x, v-u \rangle\geq 0.$
 	Similarly we have $\langle\partial h(v)+v-y, u-v \rangle\geq 0.$
 	Summing up these two inequalities leads to
 	$$ \langle v-y-u+x,u-v \rangle\geq \langle \partial h(u)-\partial h(v), u-v \rangle \geq 0,$$
 	where the second inequality is due to convexity of $h(\cdot)$. This leads to
 	\begin{eqnarray*}
 		\|prox_{h,\Omega} (x)-prox_{h,\Omega} (y)\|_2^2&\leq& \langle prox_{h,\Omega} (x)-prox_{h,\Omega} (y),x-y\rangle \\
 		&\stackrel{(a)}{\leq}& \|prox_{h,\Omega} (x)-prox_{h,\Omega} (y)\|_2 \|x-y\|_2,
 	\end{eqnarray*}
 	which implies the lemma. Here (a) holds by Cauchy-Schwarz inequality.
 \end{proof}
 
 \begin{lemma}\label{lemma.key_lemma}
 	Under the same assumption of Lemma \ref{lemma.cone_optimization},  and suppose that 
 	$\bar{\sigma}=\sigma-64\tau_\sigma H^2(\bar{M})>0$,
 	$\alpha=\frac{1}{\beta (1-4L\beta)m \bar{\sigma}}+\frac{4L\beta (m+1)}{(1-4L\beta)m} <1$ 
 	we have 
 	$$ \mathbb{E} G(\theta^{s+1})-G(\hat{\theta})\leq \alpha ( G(\theta^{s})-G(\hat{\theta}))+ \frac{1}{\bar{\sigma} \beta(1-4L\beta)m}{\epsilon} ^2 ( \Delta^*,M,\bar{M}) .$$
 \end{lemma}
 Notice this lemma states that the optimization error decrease by a fraction $\alpha$   plus some additional term related to $\epsilon ^2 ( \Delta^*,M,\bar{M})$.
 We emphasize that we make no assumption of strong convexity.
 \begin{proof}
 	In the following  proof we denote $\psi_\lambda(\theta)\triangleq\lambda \psi(\theta)$.
 	In  stage $s+1$, we have
 	\begin{equation}\label{equ.xu.prooflemma6}
 	\begin{split}
 	& G(\theta_{k+1})-G(\hat{\theta})\\
 	\leq & F(\theta_k)-F(\hat{\theta})+\langle \nabla F(\theta_k), \theta_{k+1}-\theta_k \rangle+\frac{L}{2}\|\theta_{k+1}-\theta_{k}\|_2^2+\psi_\lambda(\theta_{k+1})-\psi_\lambda(\hat{\theta})\\
 	\leq & \langle \nabla F(\theta_k),\theta_k-\hat{\theta} \rangle+\langle \nabla F(\theta_k), \theta_{k+1}-\theta_k \rangle+ \langle \partial \psi_\lambda(\theta_{k+1}),\theta_{k+1}-\hat{\theta} \rangle+\frac{L}{2}\|\theta_{k+1}-\theta_{k}\|_2^2\\
 	= &\langle \nabla F(\theta_k), \theta_{k+1}-\hat{\theta} \rangle+\langle \partial \psi_\lambda(\theta_{k+1}),\theta_{k+1}-\hat{\theta} \rangle+\frac{L}{2}\|\theta_{k+1}-\theta_{k}\|_2^2,
 	\end{split}
 	\end{equation}
 	where we use smoothness assumption of $F(\theta)$ in the first inequality, and the second inequality holds from the convexity of $F(\theta)$ and $\psi(\theta)$.
 	Now using optimality condition of $\theta_{k+1}$, we obtain
 	$$ \langle \theta_{k+1}-(\theta_k-\beta v_k)+\beta \partial \psi_\lambda(\theta_{k+1}),\theta-\theta_{k+1}\rangle\geq 0,  ~~ \forall \theta\in \{\theta|\psi(\theta)\leq \eta\}.$$
 	We choose $\theta=\hat{\theta}$ and get
 	\begin{eqnarray*}
 		&& \langle \theta_{k+1}-(\theta_k-\beta v_k)+\beta \partial \psi_\lambda(\theta_{k+1}),\hat{\theta}-\theta_{k+1}\rangle\geq 0,   \\
 		\Longrightarrow && \langle \partial \psi_\lambda(\theta_{k+1}),\theta_{k+1}-\hat{\theta}\rangle \leq \langle \frac{1}{\beta}(\theta_{k+1}-\theta_k)  + v_k, \hat{\theta}-\theta_{k+1}\rangle.
 	\end{eqnarray*}
 	Substitute the above inequality to Equation~\eqref{equ.xu.prooflemma6}, we have
 	\begin{equation}\label{G}
 	\begin{split}
 	G(\theta_{k+1})-G(\hat{\theta}) &\leq \langle \nabla F(\theta_k),\theta_{k+1}-\hat{\theta}\rangle+  \langle \frac{1}{\beta}(\theta_{k+1}-\theta_k)  + v_k, \hat{\theta}-\theta_{k+1}\rangle+\frac{L}{2}\|\theta_{k+1}-\theta_k\|_2^2 \\
 	&=\langle \nabla F(\theta_k)-v_k, \theta_{k+1}-\hat{\theta} \rangle +\frac{1}{\beta}\langle \theta_{k+1}-\theta_{k},\hat{\theta}-\theta_{k}    \rangle +(\frac{L}{2}-\frac{1}{\beta}) \|\theta_{k+1}-\theta_k\|_2^2\\
 	& \leq \langle \nabla F(\theta_k)-v_k, \theta_{k+1}-\hat{\theta} \rangle +\frac{1}{\beta}\langle \theta_{k+1}-\theta_{k},\hat{\theta}-\theta_{k}    \rangle-\frac{1}{2\beta}\|\theta_{k+1}-\theta_k\|_2^2,
 	\end{split}
 	\end{equation}
 	where the last inequality holds using the fact that $0\leq \beta\leq \frac{1}{L}$. 
 	
 	Define $\bar{\theta}_{k+1}=prox_{\beta \psi_{\lambda}, \Omega} (\theta_k-\beta \nabla F(\theta_k))$
 	\begin{equation}
 	\begin{split}
 	&\|\theta_{k+1}-\hat{\theta}\|_2^2 \\
 	=&\|\theta_k-\hat{\theta}\|_2^2+2\langle \theta_{k+1}-\theta_k, \theta_k-\hat{\theta} \rangle+\|\theta_{k+1}-\theta_{k}\|_2^2\\
 	\leq &\|\theta_k-\hat{\theta}\|_2^2+2\beta\langle \nabla F(\theta_k)-v_k, \theta_{k+1}-\hat{\theta} \rangle+2\beta (G(\hat{\theta})-G(\theta_{k+1}) )\\
 	\leq & \|\theta_k-\hat{\theta}\|_2^2+2\beta (G(\hat{\theta})-G(\theta_{k+1}))\\&\qquad +2\beta \langle \nabla F(\theta_k)-v_k,\theta_{k+1}-\bar{\theta}_{k+1} \rangle
 	+2\beta \langle \nabla F(\theta_k)-v_k, \bar{\theta}_{k+1}-\hat{\theta} \rangle\\
 	\leq & \|\theta_k-\hat{\theta}\|_2^2+2\beta (G(\hat{\theta})-G(\theta_{k+1}))+2\beta \|\nabla F(\theta_k)-v_k\|_2 \|\theta_{k+1}-\bar{\theta}_{k+1}\|_2\\ &\qquad+2\beta \langle \nabla F(\theta_k)-v_k, \bar{\theta}_{k+1}-\hat{\theta} \rangle\\
 	\leq & \|\theta_k-\hat{\theta}\|_2^2+2\beta (G(\hat{\theta})-G(\theta_{k+1}))+2\beta^2\|\nabla F(\theta_k)-v_k \|_2^2+2\beta \langle \nabla F(\theta_k)-v_k, \bar{\theta}_{k+1}-\hat{\theta} \rangle
 	\end{split}
 	\end{equation}
 	where the first inequality uses Equation \eqref{G}, and the last inequality uses the below equation implied by Lemma \ref{prox}
 	$$\|\theta_{k+1}-\bar{\theta}_{k+1}\|_2\leq \|\theta_k-\beta v_k-(\theta_k-\beta \nabla F(\theta_{k-1}))\|_2.$$
 	
 	Now, take expectation at both side w.r.t. $i_k$, and apply Lemma \ref{lemma.variance} \label{key} on $\mathbb{E}\|\nabla F(\theta_k)-v_k \|_2^2  $ and notice the fact that $\mathbb{E} \langle \nabla F(\theta_k)-v_k, \bar{\theta}_{k+1}-\hat{\theta} \rangle=0 $, we have
 	$$ \mathbb{E}\|\theta_{k+1}-\hat{\theta}\|_2^2\leq \|\theta_k-\hat{\theta}\|_2^2+\beta (G(\hat{\theta})-\mathbb{E} G(\theta_{k+1})) +8L\beta^2 ( G(\theta_{k})-G(\hat{\theta})+G(\tilde{\theta})-G(\hat{\theta}) ).$$
 	In stage $s$, sum up the above inequality at both side and take expectation, and use the fact that $\tilde{\theta}=\theta^{s} $ we have 
 	$$ 2\beta (1-4L\beta) \sum_{k=1}^{m} (\mathbb{E} G(\theta_k)-G(\hat{\theta}))\leq \|\theta^{s}-\hat{\theta}\|_2^2+8L\beta^2 (m+1) [G(\theta^{s})-G(\hat{\theta})].  $$
 	One the left hand side, we use the convexity of $G(\theta)$, i.e., 
 	$$ m(\mathbb{E} G(\theta^{s+1})-G(\hat{\theta}))\leq \sum_{k=1}^{m} (\mathbb{E} G(\theta_k)-G(\hat{\theta})). $$
 	We apply Equation \eqref{RSC_cone} on the right hand side, i.e., 
 	$$  G(\theta^{s})-G(\hat{\theta})\geq (\frac{\sigma}{2}-32\tau_\sigma H^2(\bar{M} ))\|\hat{\Delta}^{s} \|_2^2- \epsilon^2(\Delta^*,M,\bar{M}). $$

 	Recall the definition  $\bar{\sigma}=\sigma-64\tau_\sigma H^2(\bar{M})$, we have
 	$$ G(\theta^{s})-G(\hat{\theta})\geq \frac{\bar{\sigma}}{2} \|\hat{\Delta}^{s} \|_2^2- \epsilon^2(\Delta^*,M,\bar{M}). $$
 	Hence we have
 	$$2\beta (1-4L\beta)m (\mathbb{E}G(\theta^{s+1})-G(\hat{\theta}))\leq \frac{G(\theta^{s})-G(\hat{\theta})+\epsilon ^2 ( \Delta^*,M,\bar{M})}{ \bar{\sigma}/2}+ 8L\beta^2 (m+1) [G(\tilde{\theta}^{s})-G(\hat{\theta})].$$
 	
 	Rearrange terms in the above inequality we have
 	\begin{equation}
 	\begin{split}
 	\mathbb{E}G(\theta^{s+1})-G(\hat{\theta})&\leq (\frac{1}{\beta (1-4L\beta)m \bar{\sigma}}+\frac{4L\beta (m+1)}{(1-4L\beta)m}) [G(\theta^{s})-G(\hat{\theta})]\\
 	&\qquad+\frac{\epsilon ^2 ( \Delta^*,M,\bar{M})}{\bar{\sigma}\beta(1-4L\beta)m }
 	\end{split}
 	\end{equation}
 	Remind the definition of $\alpha$, this leads to
 	$$ \mathbb{E} G(\theta^{s+1})-G(\hat{\theta})\leq \alpha ( G(\theta^{s})-G(\hat{\theta}))+ \frac{1}{\bar{\sigma} \beta(1-4L\beta)m}{\epsilon} ^2 ( \Delta^*,M,\bar{M}).$$
 \end{proof}
 We can iteratively apply above inequality from time step S to s, we have
 \begin{equation}
 \begin{split}
 \mathbb{E}G(\theta^s) -G(\hat{\theta})\leq 
 &\alpha^{s-S} (G(\theta^S) -G(\hat{\theta}))+ \frac{\sum_{i=0}^{s-S}\alpha^i}{ \bar{\sigma} \beta(1-4L\beta)m}{\epsilon} ^2 ( \Delta^*,M,\bar{M})\\
 \leq&\alpha^{s-S} (G(\theta^S) -G(\hat{\theta}))+ \frac{1}{ (1-\alpha)\bar{\sigma} \beta(1-4L\beta)m}{\epsilon} ^2 ( \Delta^*,M,\bar{M}).
 \end{split}
 \end{equation}

 \begin{proof}[Proof of Theorem 1]
 	From a high level, we prove our main theorem using the argument of induction, particularly, we divide the stages into several disjoint intervals that is
 	$$\{ [S_0,S_1 ),[S_1,S_2), [S_2,S_3 ),...\}$$ with $S_0=0$. Corresponding to these intervals, we have a sequence of tolerance $\{\epsilon_0, \epsilon_1,\epsilon_2,...\} $. At the end of each interval $[S_{i-1}, S_{i})$, we can prove that the optimization error decrease to $\epsilon_i$ and $\{\epsilon_0, \epsilon_1,\epsilon_2,...\}$ is a decreasing sequence. In particular we choose $ \epsilon_{i+1}/\epsilon_{i}=\frac{1}{4} $ for $i=1,2,...,$. We also construct an increasing sequence $ k_i$ with $ k_{i+1}=2k_i $ when we apply the Markov inequality.  We apply Lemma \ref{lemma.key_lemma} recursively until $\delta_i$  is close to the statistical error $\delta_{stat}$. Notice in the following proof we can always assume $\delta_i\geq \delta_{stat}$, otherwise we already have our conclusion.
 	
 	We start the analysis from the first interval. Recall the notation
 	$$ \epsilon^2(\Delta^*,M,\bar{M})=2\tau_{\sigma} (\delta_{stat}+\delta_i)^2, \,\, \delta_i=2\min\{\frac{\epsilon_i'}{\lambda},\rho\}, \,\, \delta_{stat}= H(\bar{M})\|\Delta^*\|_2+8\psi(\theta^*_{M^\perp}).$$ 
 	In the first interval it is safe to choose  $\delta_0= 2\rho $, since $\min\{\epsilon_0/\lambda, \rho\}\leq \rho.$

 	Remind  $Q(\beta,\bar{\sigma},L,m)=\bar{\sigma} \beta(1-4L\beta)m$.
 	We apply Lemma \ref{lemma.key_lemma} in this interval to obtain
 	\begin{equation}
 	\begin{split}
 	\mathbb{E}G(\theta^{S_1}) -G(\hat{\theta}) & \leq \alpha^{S_1-S_0} (G(\theta^{S_0}) -G(\hat{\theta}))+ \frac{2}{ (1-\alpha)Q(\beta,\bar{\sigma},L,m)} \tau_\sigma (\delta_{stat}+2\rho)^2 \\
 	& \leq \alpha^{S_1-S_0} (G(\theta^{S_0}) -G(\hat{\theta}))+ \frac{4}{ (1-\alpha)Q(\beta,\bar{\sigma},L,m)} \tau_\sigma (\delta_{stat}^2+4\rho^2).
 	\end{split}
 	\end{equation}

 	
 	
 	Now we can choose 
 	$$\epsilon_1=\frac{8}{ (1-\alpha)Q(\beta,\bar{\sigma},L,m)} \tau_\sigma (\delta_{stat}^2+4\rho^2).$$
 	So we can  choose 
 	$$S_1-S_0= \lceil   \log(\frac{2 (G(\theta^{S_0})-G(\hat{\theta}))}{\epsilon_{1}}/\log( 1/\alpha))         \rceil, $$
 	such that $$\mathbb{E} G(\theta^{S_1})-G(\hat{\theta})\leq \epsilon_1 .$$
 	
 	Then we use the Markov inequality to get
 	$$ G(\theta^{S_1})-G(\hat{\theta})\leq k_1\epsilon_1\equiv\epsilon_1',$$
 	with probability $1-\frac{1}{k_1}$.

 	Now we look at the second interval, by a similar argument we obtain 
 	\begin{equation}
 	\begin{split}
 	\mathbb{E} (G(\theta^s))-G(\hat{\theta})\leq & \alpha^{s-S_1} \mathbb{E}(G(\theta^{S_1})-G(\hat{\theta}))+\frac{4\tau_\sigma}{ (1-\alpha)Q(\beta,\bar{\sigma},L,m)}  (\delta_{stat}^2+\delta_1^2)\\
 	\leq &\alpha^{s-S_1} \mathbb{E}(G(\theta^{S_1})-G(\hat{\theta}))+\frac{8\tau_\sigma}{ (1-\alpha)Q(\beta,\bar{\sigma},L,m)}  \delta_1^2
 	\end{split}
 	\end{equation}
 	where $\delta_1=\frac{2\epsilon_1'}{\lambda}=\frac{2k_1\epsilon_1}{\lambda}.$
 	
 	We need
 	$$\frac{8}{ (1-\alpha)Q(\beta,\bar{\sigma},L,m)} \tau_\sigma \delta_1^2\leq\frac{\epsilon_1}{8}, $$
 	which is satisfied if we choose
 	$$ k_1=\sqrt{\frac{(1-\alpha) \lambda^2 Q(\beta,\bar{\sigma},L,m)}{256\tau_\sigma \epsilon_1}}. $$
 	Then we can choose $S_2-S_1=\lceil \log 8/\log(1/\alpha) \rceil$, such that
 	$$ \mathbb{E}(G(\theta^{S_2})-G(\hat{\theta})) \leq \frac{\epsilon_1}{8}+\frac{\epsilon_1}{8}\leq \frac{\epsilon_1}{4}=\epsilon_2.$$
 	In general, for the  $i+1^{th}$ time interval, since $\mathbb{E}G(\theta^{S_i})-G(\hat{\theta})\leq \epsilon_i$, we have
 	\begin{equation}
 	\begin{split}
 	\mathbb{E}G(\theta^{s}) -G(\hat{\theta}) & \leq \alpha^{s-S_i} \mathbb{E}(G(\theta^{S_i}) -G(\hat{\theta}))+ \frac{4}{ (1-\alpha)Q(\beta,\bar{\sigma},L,m)} \tau_\sigma (\delta^2_{stat}+\delta_i^2) \\
 	& \leq \alpha^{s-S_i} \mathbb{E}(G(\theta^{S_i}) -G(\hat{\theta}))+ \frac{8}{ (1-\alpha)Q(\beta,\bar{\sigma},L,m)} \tau_\sigma \max(\delta^2_{stat},\delta_i^2)\\
 	&\leq  \alpha^{s-S_i} \mathbb{E}(G(\theta^{S_i}) -G(\hat{\theta}))+ \frac{8}{ (1-\alpha)Q(\beta,\bar{\sigma},L,m)} \tau_\sigma \delta_i^2
 	\end{split}
 	\end{equation} 
 	with probability $1-\frac{1}{k_i}$, where we choose $k_i=2^{i-1}k_1  $. $\delta_i=2\min \{\epsilon_i'/\lambda,\rho\}.$ 

 	We need following condition holds
 	$$\frac{8}{ (1-\alpha)Q(\beta,\bar{\sigma},L,m)} \tau_\sigma \delta_i^2=\frac{8}{ (1-\alpha)Q(\beta,\bar{\sigma},L,m)} \tau_\sigma (2k_i\epsilon_i/\lambda)^2\leq \frac{\epsilon_i}{8} $$
 	which means
 	$$ \frac{32 k_i^2\epsilon_i\tau_\sigma}{\lambda^2 (1-\alpha)Q(\beta,\bar{\sigma},L,m) }\leq \frac{1}{8} .$$
 	Since $k_i=2k_{i-1}$ and $\epsilon_i=\frac{1}{4}\epsilon_{i-1} $, we just need $ \frac{32 k_1^2\epsilon_1\tau_\sigma}{\lambda^2 (1-\alpha)Q(\beta,\bar{\sigma},L,m) }\leq \frac{1}{8} $ holds.
 	It is satisfied by our choice $ k_1=\sqrt{\frac{(1-\alpha) \lambda^2 Q(\beta,\bar{\sigma},L,m)}{256\tau_\sigma \epsilon_1}}. $
 	Thus we can choose $S_{i+1}-S_i=\log (8)/\log (1/\alpha) $, such that
 	$$ \mathbb{E} G(\theta^{S_{i+1}})-G(\hat{\theta})\leq \frac{\epsilon_i}{4}=\epsilon_i ' .$$ 
 	If $\kappa^2\geq \frac{8}{ (1-\alpha)Q(\beta,\bar{\sigma},L,m)} \tau_\sigma (\delta_{stat}^2+4\rho^2)$, the total number of steps to achieve the tolerance $\kappa^2$ is 
 	$$  \log\left(\frac{ (G(\theta^{0})-G(\hat{\theta}))}{\kappa^2}/\log( 1/\alpha)\right).    $$
 	If  $\kappa^2\leq \frac{8}{ (1-\alpha)Q(\beta,\bar{\sigma},L,m)} \tau_\sigma (\delta_{stat}^2+4\rho^2), $
 	the total number of steps to achieve the tolerance $\kappa^2$ is  
 	\begin{equation}
 	\begin{split}
 	&\lceil \frac{\log 8}{\log (1/\alpha)}\rceil  \log_4 (\frac{G(\theta^0)-G(\hat{\theta})}{\kappa^2})+ \lceil   \log(\frac{ (G(\theta^{0})-G(\hat{\theta}))}{\frac{4}{ (1-\alpha)Q(\beta,\bar{\sigma},L,m)} \tau_\sigma (\delta_{stat}^2+4\rho^2)}/\log( 1/\alpha))         \rceil\\
 	\leq & 3 \log (\frac{G(\theta^0)-G(\hat{\theta})}{\kappa^2})/\log(1/\alpha),
 	\end{split}
 	\end{equation}
 	with probability at least $ 1- \frac{\log 8}{\log (1/\alpha)}\sum_{i=1}^{S_\kappa} \frac{1}{k_i}\geq 1-2\frac{\log 8}{k_1\log (1/\alpha)}$, where we use the union bound on each step and the fact $k_i=2^{i-1}k_1.$ Since we choose $m=2n$ and remind the assumption that $n>c\rho^2 \log p $, we know  $\frac{1}{k_1}\simeq \frac{c_1}{n}$ for some constant $c_1$.
 \end{proof}

 %
 %
 %

 \subsection{SVRG with non-convex objective function}
 We start with two technical lemmas. The following lemma is Lemma 5 extract from\cite{loh2013regularized}, we present here for the completeness.
 \begin{lemma}\label{lemma.non_convex_norm}
 	For any vector $ \theta\in R^p$, let $A$ denote the index set of its $r$ largest elements in magnitude, under assumption on $g_{\lambda,\mu}$ in Section \ref{section:nonconvex_regularizer} , we have 	
 	$$ g_{\lambda,\mu}(\theta_A)-g_{\lambda,\mu} (\theta_{A^{c}})\leq \lambda L_g (\|\theta_A\|_1-\|\theta_{A^c}\|_1).$$ 	
 	Moreover, for an arbitrary vector $\theta\in R^p$, we have	
 	$$ g_{\lambda,\mu} (\theta^*)-g_{\lambda,\mu} (\theta)\leq \lambda L_g (\|\nu_A\|_1-\|\nu_{A^c}\|_1), $$
 	where $\nu=\theta-\theta^*$ and $\theta^*$ is r sparse.
 \end{lemma}	
 
 The following lemma is well known on smooth function, we extract it from Lemma 1 in\cite{xiao2013proximal}.
 \begin{lemma}\label{lemma.smooth} Suppose each $f_i(\theta)$ is $L$ smooth and convex then we have
 	
 	$$\|\nabla f_i(\theta)-\nabla f_i(\hat{\theta})\|_2^2\leq 2L ( f_i(\theta)-f_i(\hat{\theta}) -\langle \nabla f_i(\hat{\theta}),\theta-\hat{\theta} \rangle ).$$
 	So we have
 	$$\frac{1}{n}\sum_{i=1}^{n}\|\nabla f_i(\theta)-\nabla f_i(\hat{\theta})\|_2^2\leq 2L (F(\theta)-F(\hat{\theta})-\langle \nabla F(\hat{\theta}),\theta-\hat{\theta} \rangle)$$
 \end{lemma}
 
 The next lemma is a non-convex counterpart of Lemma~\ref{lemma.cone} and Lemma \ref{lemma.cone_optimization}
 \begin{lemma}\label{lemma.non_convex_cone_optimizatin}
 	Suppose $g_{\lambda,\mu}(\cdot)$ satisfies the assumptions in section \ref{section:nonconvex_regularizer}, $\lambda L_g\geq8\rho\tau \frac{\log p }{n}$, $\lambda\geq \frac{4}{L_g} \|\nabla F (\theta^*)\|_\infty$, $\theta^*$ is feasible,   and there exists a pair $(\epsilon',S)$ such that
 	$$G(\theta^s)-G(\hat{\theta})\leq \epsilon', \forall s\geq S.$$
 	Then for any iteration $s\geq S$, we have	
 	$$ \|\theta^s-\hat{\theta}\|_1\leq 4\sqrt{r} \|\theta^s-\hat{\theta}\|_2+8\sqrt{r} \|\theta^*-\hat{\theta}\|_2+2\min (\frac{\epsilon'}{\lambda L_g}, \rho).$$
 \end{lemma}
 \begin{proof}
 	Fix an arbitrary feasible $\theta$,
 	Define $\Delta=\theta-\theta^* $.
 	Since we know $G(\hat{\theta})\leq G(\theta^*) $ so we have $ G(\theta)\leq G(\theta^*)+\epsilon'$
 	, which implies 
 	$$ F(\theta^*+\Delta)+g_{\lambda,\mu} (\theta^*+\Delta)\leq F(\theta^*)+g_{\lambda,\mu}(\theta^*) +\epsilon' .$$
 	Subtract $\langle \nabla F(\theta^*),\Delta \rangle$ and use the RSC condition we have
 	\begin{equation}
 	\begin{split}
 	&\frac{\sigma}{2} \|\Delta\|_2^2-\tau\frac{\log p}{n} \|\Delta\|_1^2+g_{\lambda,\mu} (\theta^*+\Delta)-g_{\lambda,\mu}(\theta^*)\\
 	\leq &\epsilon'-\langle \nabla F(\theta^*),\Delta \rangle\\
 	\leq &\epsilon'+\|\nabla F(\theta^*)\|_\infty \|\Delta\|_1
 	\end{split}
 	\end{equation}
 	where the last inequality holds from Holder's inequality.
 	Rearrange terms and use the fact that 
 	$\|\Delta\|_1\leq 2\rho $ (by feasiblity of $\theta$ and $\theta^*$) and the assumptions $\lambda L_g\geq8\rho\tau \frac{\log p }{n}$ , $\lambda\geq \frac{4}{L_g} \|\nabla F (\theta^*)\|_\infty$, we obtain
 	$$ \epsilon'+\frac{1}{2} \lambda L_{g} \|\Delta\|_1+g_{\lambda,\mu}(\theta^*)-g_{\lambda,\mu}(\theta^*+\Delta)\geq \frac{\sigma}{2}\|\Delta\|_2^2\geq 0. $$
 	
 	By Lemma \ref{lemma.non_convex_norm}, we have
 	$$ g_{\lambda,\mu} (\theta^*)-g_{\lambda,\mu} (\theta)\leq \lambda L_g (\|\Delta_A\|_1-\|\Delta_{A^c}\|_1) ,$$
 	where $A$ indexes the top $r$ components of $\Delta$ in magnitude.
 	So we have 
 	$$ \frac{3\lambda L_g}{2} \|\Delta_A\|_1-\frac{\lambda L_g}{2} \|\Delta_{A^c}\|_1+\epsilon'\geq 0, $$
 	and consequently 
 	$$ \|\Delta\|_1\leq \|\Delta_A\|_1+\|\Delta_{A^c}\|_1\leq 4\|\Delta_A\|_1+\frac{2\epsilon'}{\lambda L_g} \leq 4\sqrt{r} \|\Delta\|_2 +\frac{2\epsilon'}{\lambda L_g}.$$
 	Combining this with $  \|\Delta\|_1\leq 2\rho$ leads to
 	$$ \|\Delta\|_1\leq 4\sqrt{r} \|\Delta\|_2 +2\min \{\frac{\epsilon'}{\lambda L_g}, \rho\}.$$
 	Since this holds for any feasible $\theta$,  we have $\|\theta^s-\theta^*\|_1\leq 4\sqrt{r} \|\theta^s-\theta^*\|_2+2\min \{\frac{\epsilon'}{\lambda L_g}, \rho\}.$
 	
 	Notice $G(\theta^*)-G(\hat{\theta})\leq 0$, so following same derivation as above and set $\epsilon'=0$ we have
 	$\|\hat{\theta}-\theta^*\|_1\leq 4\sqrt{r} \|\hat{\theta}-\theta^*\|_2.$
 	
 	Combining the two, we have 
 	$$\|\theta^s-\hat{\theta}\|_1\leq \|\theta^s-\theta^*\|_1+\|\theta^*-\hat{\theta}\|_1\leq  4\sqrt{r} \|\theta^s-\hat{\theta}\|_2+8\sqrt{r} \|\theta^*-\hat{\theta}\|_2+2\min (\frac{\epsilon'}{\lambda L_g}, \rho). $$
 \end{proof}
 
 Now we provide a counterpart of Lemma \ref{lemma.RSC_cone} in the non-convex case. Notice the main difference with the convex case is the coefficient before $ \|\theta^s-\hat{\theta}\|_2^2. $
 \begin{lemma}\label{lemma.non_convex_RSC_cone}
 	Under the same assumption of Lemma \ref{lemma.non_convex_cone_optimizatin}, we  have
 	$$ G(\theta^s)-G(\hat{\theta})\geq (\frac{\sigma-\mu}{2}-32r\tau_\sigma)\|\theta^s-\hat{\theta}\|_2^2-\epsilon^2 (\Delta^*,r ), $$
 	where $\Delta^*=\hat{\theta}-\theta^*$, and $ \epsilon^2 (\Delta^*,r )=2\tau_\sigma (8\sqrt{r}\|\hat{\theta}-\theta^*\|_2+2\min(\frac{\epsilon'}{\lambda L_g},\rho))^2$.
 \end{lemma}
 \begin{proof} We have the following:
 	\begin{equation}\label{equ.xu.proof-lemma8}
 	\begin{split}
 	& G(\theta^s)-G(\hat{\theta})\\
 	= & F(\theta^s)-F(\hat{\theta})-\frac{\mu}{2}\|\theta^s\|_2^2+\frac{\mu}{2}\|\hat{\theta}\|_2^2+ \lambda g_\lambda(\theta^s)-\lambda g_\lambda(\hat{\theta})\\
 	\geq & \langle \nabla F(\hat{\theta}),\theta^s-\hat{\theta} \rangle+\frac{\sigma}{2}\|\theta^s-\hat{\theta}\|_2^2- \langle u\hat{\theta}, \theta^s-\hat{\theta} \rangle-\frac{\mu}{2} \|\theta^s-\hat{\theta}\|_2^2\\ &\qquad+ \lambda g_\lambda(\theta^s)-\lambda g_\lambda(\hat{\theta})-\tau \frac{\log p}{n} \|\theta^s-\hat{\theta}\|_1^2\\
 	\geq &  \langle \nabla F(\hat{\theta}),\theta^s-\hat{\theta} \rangle+\frac{\sigma}{2}\|\theta^s-\hat{\theta}\|_2^2- \langle u\hat{\theta}, \theta^s-\hat{\theta} \rangle-\frac{\mu}{2} \|\theta^s-\hat{\theta}\|_2^2\\&\qquad+ \lambda \langle \partial g_\lambda(\hat{\theta}),\theta^s-\hat{\theta}\rangle-\tau \frac{\log p}{n} \|\theta^s-\hat{\theta}\|_1^2 \\
 	= & \frac{\sigma-\mu}{2}\|\theta^s-\hat{\theta}\|_2^2-\tau\frac{\log p}{n}\|\theta^s-\hat{\theta}\|_1^2,
 	\end{split}
 	\end{equation}
 	where the first inequality uses the RSC condition, the second inequality uses the convexity of $g_\lambda(\theta)$, and the last equality holds from the optimality condition of $\hat{\theta}$.

 	By Lemma \ref{lemma.non_convex_cone_optimizatin} we have 
 	\begin{equation}
 	\begin{split}
 	\|\theta^s-\hat{\theta}\|_1^2\leq  &(4\sqrt{r} \|\theta^s-\hat{\theta}\|_2+8\sqrt{r} \|\theta^*-\hat{\theta}\|_2+2\min (\frac{\epsilon'}{\lambda L_g}, \rho))^2 \\
 	\leq &32 r \|\theta^s-\hat{\theta}\|_2^2+2 (8\sqrt{r}\|\hat{\theta}-\theta^*\|_2+2\min(\frac{\epsilon'}{\lambda L_g},\rho))^2.
 	\end{split}
 	\end{equation}
 	Substitute this into Equation~\eqref{equ.xu.proof-lemma8} we obtain
 	$$ G(\theta^s)-G(\hat{\theta})\geq (\frac{\sigma-\mu}{2}-32r\tau_\sigma)\|\theta^s-\hat{\theta}\|_2^2-2\tau_\sigma (8\sqrt{r}\|\hat{\theta}-\theta^*\|_2+2\min(\frac{\epsilon'}{\lambda L_g},\rho))^2.$$
 \end{proof}
 
 We are now ready to prove the main theorem for non-convex case, i.e., Theorem 3.
 \begin{proof}[Proof of Theorem 3]
 	Define $F_{u}(\theta)=F(\theta)-\frac{\mu}{2} \|\theta\|_2^2$. It is easy to check that $F_\mu (\theta)$ is smooth with parameter $L_\mu=\max(L-\mu, \mu)$.   Use the smoothness of $F_\mu (\theta)$, we have
 	\begin{equation}
 	\begin{split}
 	& F_\mu(\theta_{k+1})+\lambda g_\lambda(\theta_{k+1})-F_\mu(\hat{\theta})-\lambda g_\lambda(\hat{\theta})\\
 	\leq & F_\mu(\theta_k) -F_\mu(\hat{\theta})-\langle \nabla F_\mu (\theta_k),\theta_{k+1}-\theta_k\rangle+\frac{L_\mu}{2} \|\theta_{k+1}-\theta_k\|_2^2+\lambda g_\lambda(\theta_{k+1})-\lambda g_\lambda(\hat{\theta})\\
 	=& F(\theta_k)-F(\hat{\theta})-\frac{\mu}{2}\|\theta_k\|_2^2+\frac{\mu}{2} \|\hat{\theta}\|_2^2+\langle \nabla F_{\mu}(\theta_k),\theta_{k+1}-\theta_k \rangle\\
 	&\qquad+\frac{L_\mu}{2} \|\theta_{k+1}-\theta_k\|_2^2+\lambda g_\lambda(\theta_{k+1})-\lambda g_\lambda(\hat{\theta})\\
 	\leq & \langle \nabla F(\theta_k),\theta_k-\hat{\theta} \rangle+\langle \nabla F_\mu(\theta_k),\theta_{k+1}-\theta_k\rangle
 	-\frac{\mu}{2}\|\theta_k\|_2^2+\frac{\mu}{2}\|\hat{\theta}\|_2^2\\
 	&\qquad+\frac{L_\mu}{2}\|\theta_{k+1}-\theta_k\|_2^2 + \lambda g_\lambda(\theta_{k+1})-\lambda g_\lambda(\hat{\theta})\\
 	= & \langle \nabla F_\mu(\theta_k),\theta_k-\hat{\theta} \rangle+\langle \nabla F_\mu(\theta_k),\theta_{k+1}-\theta_k\rangle+\frac{L_\mu}{2}\|\theta_{k+1}-\theta_k\|_2^2\\&\qquad+\frac{\mu}{2}\|\theta_k-\hat{\theta}\|_2^2 + \lambda g_\lambda(\theta_{k+1})-\lambda g_\lambda(\hat{\theta})\\
 	\leq & \langle \nabla F_\mu(\theta_k),\theta_{k+1}-\hat{\theta} \rangle+\frac{L_\mu}{2}\|\theta_{k+1}-\theta_k\|_2^2+\frac{\mu}{2}\|\theta_k-\hat{\theta}\|_2^2 + \lambda \langle \partial g_\lambda(\theta_{k+1}),\theta_{k+1}-\hat{\theta} \rangle,
 	\end{split}
 	\end{equation}
 	where the second inequality uses the convexity of $F(\theta). $
 	Then we use the optimality of $\theta_{k+1}$ and recall $g_{\lambda}(\cdot)$ is convex then have
 	$$ \langle \partial g_\lambda(\theta_{k+1}),\theta_{k+1}-\hat{\theta} \rangle\leq \langle \frac{1}{\beta} (\theta_{k+1}-\theta_k)+v_k ,\hat{\theta}-\theta_{k+1}\rangle. $$
 	Using this result we have
 	\begin{equation}
 	\begin{split}
 	& F_\mu(\theta_{k+1})+\lambda g_\lambda(\theta_{k+1})-F_\mu(\hat{\theta})-\lambda g_\lambda(\hat{\theta})\\
 	\leq & \langle \nabla F_{\mu} (\theta_k),\theta_{k+1}-\hat{\theta} \rangle +\langle \frac{1}{\beta} (\theta_{k+1}-\theta_k)+v_k,\hat{\theta}-\theta_{k+1} \rangle +\frac{L_\mu}{2}\|\theta_{k+1}-\theta_k\|_2^2+\frac{\mu}{2}\|\theta_k-\hat{\theta}\|_2^2\\
 	= & \langle \nabla F_{\mu} (\theta_k)-v_k, \theta_{k+1}-\hat{\theta} \rangle +\frac{1}{\beta}\langle  (\theta_{k+1}-\theta_k),\hat{\theta}-\theta_{k} \rangle -\frac{1}{\beta} \|\theta_{k+1}-\theta_k\|_2^2\\
 	&\qquad+\frac{L_\mu}{2}\|\theta_{k+1}-\theta_k\|_2^2+\frac{\mu}{2}\|\theta_k-\hat{\theta}\|_2^2\\
 	\leq &\langle \nabla F_{\mu} (\theta_k)-v_k, \theta_{k+1}-\hat{\theta} \rangle +\frac{1}{\beta}\langle  \theta_{k+1}-\theta_k,\hat{\theta}-\theta_{k} \rangle-\frac{1}{2\beta}\|\theta_{k+1}-\theta_k\|_2^2+\frac{\mu}{2}\|\theta_k-\hat{\theta}\|_2^2,
 	\end{split}
 	\end{equation}
 	where the last inequality uses the fact that $\beta\leq \frac{1}{L_\mu}$.
 	Rearranging terms, we obtain
 	\begin{eqnarray*}
 		&& 2\langle \theta_{k+1}-\theta_k,\theta_k-\hat{\theta} \rangle+\|\theta_{k+1}-\theta_k\|_2^2 \\
 		&&\leq 2\beta \langle \nabla F_\mu (\theta_k)-v_k,\theta_{k+1}-\hat{\theta} \rangle+\beta\mu \|\theta_k-\hat{\theta}\|_2^2+2\beta (G(\hat{\theta})-G(\theta_{k+1})).\end{eqnarray*}

 	Define $\bar{\theta}_{k+1}=prox_{\beta \lambda g, \Omega} (\theta_k-\beta \nabla F_\mu(\theta_k))$.   Similarly as the convex case, we have 
 	\begin{equation}\label{eq: mid_non_convex_proof}
 	\begin{split}
 	&\|\theta_{k+1}-\hat{\theta}\|_2^2 \\
 	=&\|\theta_k-\hat{\theta}\|_2^2+2\langle \theta_{k+1}-\theta_k, \theta_k-\hat{\theta} \rangle+\|\theta_{k+1}-\theta_{k}\|_2^2\\
 	\leq& \|\theta_k-\hat{\theta}\|_2^2+2\beta \langle \nabla F_\mu (\theta_k)-v_k,\theta_{k+1}-\hat{\theta} \rangle+\beta \mu \|\theta_k-\hat{\theta}\|_2^2+2\beta (G(\hat{\theta})-G(\theta_{k+1}))\\
 	\leq & \|\theta_k-\hat{\theta}\|_2^2+2\beta (G(\hat{\theta})-G(\theta_{k+1})) +2\beta \langle \nabla F_{\mu}(\theta_k)-v_k,\theta_{k+1}-\bar{\theta}_{k+1} \rangle
 	\\&\quad +2\beta \langle \nabla F_{\mu}(\theta_k)-v_k, \bar{\theta}_{k+1}-\hat{\theta} \rangle
 	+\beta\mu \|\theta_{k}-\hat{\theta}\|_2^2\\  
 	\leq & (1+\beta \mu) \|\theta_k-\hat{\theta}\|_2^2+2\beta (G(\hat{\theta})-G(\theta_{k+1}))+2\beta^2\|\nabla F_\mu(\theta_k)-v_k \|_2^2+2\beta \langle \nabla F_\mu(\theta_k)-v_k, \bar{\theta}_{k+1}-\hat{\theta} \rangle.
 	\end{split}
 	\end{equation}
 	Now we need to bound $E \|\nabla F_\mu (\theta_k)-v_k\|_2^2 $.
 	\begin{equation*}
 	\begin{split}
 	\nabla F_\mu (\theta_k)-v_k =&\nabla F(\theta_k)-\mu\theta_k- (f_{i_{k+1}}(\theta_k)-\mu \theta_k-f_{i_{k+1}} (\tilde{\theta})+\mu \tilde{\theta}+\nabla F(\tilde{\theta})-\mu\tilde{\theta})\\
 	=& \nabla F (\theta_k)-\nabla F (\tilde{\theta})+f_{i_{k+1}}(\tilde{\theta})-f_{i_{k+1}}(\theta_k) 
 	\end{split}
 	\end{equation*}
 	Notice $F(\theta)$ and $f_{i}(\theta)$ are convex, so we can use Lemma \ref{lemma.smooth} to bound  $ E\| \nabla F (\theta_k)-\nabla F (\tilde{\theta})+f_{i_{k+1}}(\tilde{\theta})-f_{i_{k+1}}(\theta_k)   \|_2^2.$
 	
 	In particular condition on $\theta_{k}$, and take the expectation with respect to $i_{k+1}$, we have
 	
 	$ \mathbb{E} \nabla f_{i_{k+1}}(\theta_k)=\nabla F(\theta_k)$, $ \mathbb{E} \nabla f_{i_{k+1}}(\tilde{\theta})=\nabla F(\tilde{\theta})$ and 
 	
 	\begin{equation}
 	\begin{split}
 	&\mathbb{E}\| \nabla F (\theta_k)-\nabla F (\tilde{\theta})+f_{i_{k+1}}(\tilde{\theta})-f_{i_{k+1}}(\theta_k)   \|_2^2\\
 	=&\mathbb{ E} \| f_{i_{k+1}}(\tilde{\theta})-f_{i_{k+1}}(\theta_k) \|_2^2-\|\nabla F (\theta_k)-\nabla F (\tilde{\theta})\|_2^2\\
 	\leq &  \mathbb{E} \| f_{i_{k+1}}(\tilde{\theta})-f_{i_{k+1}}(\theta_k) \|_2^2\\
 	\leq & 2 \mathbb{E} \| f_{i_{k+1}}(\tilde{\theta})-f_{i_{k+1}}(\hat{\theta
 	}) \|_2^2+2 \mathbb{E} \| f_{i_{k+1}}(\theta_k)-f_{i_{k+1}} (\hat{\theta})\|_2^2\\
 	\leq & 4 L [F(\tilde{\theta})-F(\hat{\theta})-\langle \nabla F(\hat{\theta}),\tilde{\theta}-\hat{\theta} \rangle]+4L[F(\theta_k)-F(\hat{\theta})-\langle \nabla F(\hat{\theta}),\theta_k-\hat{\theta} \rangle].
 	\end{split}
 	\end{equation}
 	
 	Now substitute corresponding terms in \ref{eq: mid_non_convex_proof}, we obtain   
 	\begin{equation}\label{non_convex_optimization}
 	\begin{split}
 	\mathbb{E}\|\theta_{k+1}-\hat{\theta}\|_2^2 \leq & (1+\beta \mu) \|\theta_k-\hat{\theta}\|_2^2 +2\beta (G(\hat{\theta})-\mathbb{E}G(\theta_{k+1}))+8L\beta^2 [F(\tilde{\theta})-F(\hat{\theta})-\langle \nabla F(\hat{\theta}),\tilde{\theta}-\hat{\theta} \rangle]\\
 	+ & 8L\beta^2 [F(\theta_k)-F(\hat{\theta})-\langle \nabla F(\hat{\theta}),\theta_k-\hat{\theta} \rangle].
 	\end{split}
 	\end{equation}
 	
 	Notice
 	\begin{equation*}
 	\begin{split}
 	&F(\tilde{\theta})-F(\hat{\theta})-\langle \nabla F(\hat{\theta}),\tilde{\theta}-\hat{\theta} \rangle\\
 	=& F(\tilde{\theta})-F(\hat{\theta})+\langle \lambda \partial g_\lambda(\hat{\theta})-\mu\hat{\theta},\tilde{\theta}-\hat{\theta} \rangle\\
 	=& F_\mu (\tilde{\theta})-F_\mu(\hat{\theta})+\frac{\mu}{2}\|\tilde{\theta}-\hat{\theta}\|_2^2 +\langle \lambda \partial g_\lambda(\hat{\theta}),\tilde{\theta}-\hat{\theta}\rangle\\
 	\leq & F_\mu (\tilde{\theta})+\lambda g_\lambda(\tilde{\theta})-F_\mu(\hat{\theta})-\lambda g_\lambda(\hat{\theta})+\frac{\mu}{2}\|\tilde{\theta}-\hat{\theta}\|_2^2 \\
 	= & G(\tilde{\theta})-G(\hat{\theta})+\frac{\mu}{2} \|\theta_k-\hat{\theta}\|_2^2.
 	\end{split}
 	\end{equation*}
 	Similarly we have
 	$$F(\theta_k)-F(\hat{\theta})-\langle \nabla F(\hat{\theta}),\theta_k-\hat{\theta} \rangle\leq  G(\theta_k)-G(\hat{\theta})+\frac{\mu}{2}\|\theta_k-\hat{\theta}\|_2^2. $$
 	Substitute these into corresponding terms in \ref{non_convex_optimization}, we have
 	\begin{eqnarray*}
 		\mathbb{E}\|\theta_{k+1}-\hat{\theta}\|_2^2&\leq& (1+\beta\mu+4L\mu\beta^2)\|\theta_k-\hat{\theta}\|_2^2+4L\beta^2\mu \|\tilde{\theta}-\hat{\theta}\|_2^2\\
 		&&\quad+2\beta (G(\hat{\theta})-E G(\theta_{k+1}))+8L\beta^2 [G(\tilde{\theta})-G(\hat{\theta})+G(\theta_k)-G(\hat{\theta})].
 	\end{eqnarray*}
 	
 	Now summation over $k$  and take expectation, and notice in the algorithm   we chose $\theta^{s+1}$ randomly rather than average, we have 
 	\begin{equation}
 	\begin{split}
 	&\mathbb{E}\sum_{k=0}^{m-1} \|\theta_{k+1}-\hat{\theta}\|_2^2\leq (1+\beta\mu+4L\mu\beta^2) \|\theta_0-\hat{\theta}\|_2^2+\sum_{k=1}^{m} (1+\beta\mu+4L\mu\beta^2) \mathbb{E}\|\theta_k-\hat{\theta}\|_2^2\\
 	-&(1+\beta\mu +4L\mu\beta^2)\mathbb{E} \|\theta_m-\hat{\theta}\|_2^2+4L\beta^2\mu m\|\tilde{\theta}-\hat{\theta}\|_2^2+2\beta(1-4L\beta)\sum_{k=1}^{m-1}(G(\hat{\theta})-\mathbb{E}G(\theta_{k+1}))\\
 	+&8L\beta^2m[G(\tilde{\theta})-G(\hat{\theta})]+8L\beta^2 (G(\theta_0)-G(\hat{\theta}))-8L\beta^2 (\mathbb{E}G(\theta_m)-G(\hat{\theta})).
 	\end{split}
 	\end{equation}
 	Using the fact that $G(\theta_m)-G(\hat{\theta})\geq 0$, $\theta^0=\theta^s$, $\sum_{k=1}^{m}\|\theta_k-\hat{\theta}\|_2^2=m E\|\theta^{s+1}-\hat{\theta}\|_2^2$  and rearrange terms we have
 	\begin{equation}\label{equ.xu.proof-nconvex1}
 	\begin{split}
 	& 2\beta (1-4L\beta) m (\mathbb{E}G(\theta^{s+1})-G(\hat{\theta}))-\mu\beta (1+4L\beta) m \mathbb{E}\|\theta^{s+1}-\hat{\theta}\|_2^2\\
 	\leq & (1+\beta\mu+4L\mu \beta^2+4L\beta^2\mu m) \|\theta^s-\hat{\theta}\|_2^2+8L\beta^2 (m+1) [G(\theta^s)-G(\hat{\theta})].
 	\end{split}
 	\end{equation}
 	The remainder of the proof follows a similar line to that of the convex case, modulus some difference in coefficients. We divide the stage into several disjoint intervals that is $\{ [S_0,S_1 ),[S_1,S_2), [S_2,S_3 ),...\}$ with $S_0=0$. Corresponding to these intervals , we have a sequence of tolerance $\{\epsilon_0, \epsilon_1,\epsilon_2,...\}$, where $\epsilon_{i+1}=\epsilon_i/4$ and the value of $\epsilon_1$ will be specified below.

 	Apply Lemma \ref{lemma.non_convex_RSC_cone} and recall the definition  $\bar{\sigma}=\sigma-\mu-64\tau_\sigma r$ to Equation~\eqref{equ.xu.proof-nconvex1}, we obtain
 	\begin{equation}
 	\begin{split}
 	& 2\beta (1-4L\beta) m (\mathbb{E}G(\theta^{s+1})-G(\hat{\theta}))-\mu\beta (1+4L\beta) m \frac{2}{\bar{\sigma}} \mathbb{E}( G(\theta^{s+1})-G(\hat{\theta})+\epsilon^2 (\Delta^*,r) )\\
 	\leq & (1+\beta\mu+4L\mu \beta^2+4L\beta^2\mu m) \frac{2}{\bar{\sigma}} ( G(\theta^{s})-G(\hat{\theta})+\epsilon^2 (\Delta^*,r) )   +8L\beta^2 (m+1) [G(\theta^s)-G(\hat{\theta})].
 	\end{split}
 	\end{equation}
 	Rearrange the terms we have 
 	\begin{equation}
 	\begin{split}
 	&\beta m (2-8L\beta-\frac{2\mu}{\bar{\sigma}} (1+4L\beta)) \mathbb{E}(G(\theta^{s+1})-G(\hat{\theta}))\\
 	\leq & [8L\beta^2 (m+1) + \frac{2 (1+\beta\mu+4L\mu\beta^2+4L\beta^2 \mu m)}{\bar{\sigma}}] (G(\theta^s)-G(\hat{\theta}))\\
 	+&[\frac{2\mu \beta m}{\bar{\sigma}} (1+4L\beta)+ (1+\beta\mu+4L\mu \beta^2+4L\beta^2\mu m) \frac{2}{\bar{\sigma}}]\epsilon^2(\Delta^*,r).
 	\end{split}
 	\end{equation}
 	This is equivalent to
 	\begin{eqnarray}
 	&&\mathbb{E} (G(\theta^{s+1})-G(\hat{\theta}))\leq \alpha (G(\theta^s)-G(\hat{\theta}))+ \chi(\beta,\mu,L,m,\sigma)\epsilon^2(\Delta^*,r),\\
 	&&\mbox{where}\quad  \alpha\triangleq \frac{8L\beta^2 (m+1) + \frac{2 (1+\beta\mu+4L\mu\beta^2+4L\beta^2 \mu m)}{\bar{\sigma}}}{\beta m (2-8L\beta-\frac{2\mu}{\bar{\sigma}} (1+4L\beta))},\nonumber\\
 	&&\mbox{and}\quad \chi\triangleq \frac{\frac{2\mu \beta m}{\bar{\sigma}} (1+4L\beta)+ (1+\beta\mu+4L\mu \beta^2+4L\beta^2\mu m) \frac{2}{\bar{\sigma}}}{\beta m (2-8L\beta-\frac{2\mu}{\bar{\sigma}} (1+4L\beta))}.\nonumber
 	\end{eqnarray}
 	For the first interval, it is safe to choose  $$\epsilon^2 (\Delta^*,r)=2\tau_\sigma (\delta_{stat}+2\rho)^2,$$
 	which leads to
 	\begin{equation}
 	\begin{split}
 	\mathbb{E}G(\theta^{S_1}) -G(\hat{\theta}) & \leq \alpha^{S_1-S_0} (G(\theta^{S_0}) -G(\hat{\theta}))+ \frac{2\chi}{ (1-\alpha)} \tau_\sigma (\delta_{stat}+2\rho)^2 \\
 	& \leq \alpha^{S_1-S_0} (G(\theta^{S_0}) -G(\hat{\theta}))+ \frac{4 \chi}{ (1-\alpha)} \tau_\sigma (\delta_{stat}^2+4\rho^2).
 	\end{split}
 	\end{equation}
 	Now we can choose 
 	$$\epsilon_1=\frac{8 \chi}{ (1-\alpha)} \tau_\sigma (\delta_{stat}^2+4\rho^2).$$
 	So it is enough to choose
 	$$S_1-S_0= \lceil   \log(\frac{2 (G(\theta^{S_0})-G(\hat{\theta}))}{\epsilon_{1}}/\log( 1/\alpha)) \rceil .$$
 	Then by Markov inequality we have
 	$$ G(\theta^{S_1})-G(\hat{\theta})\leq k_1\epsilon_1\equiv\epsilon_1',$$
 	with probability $1-\frac{1}{k_1}$. The value of $k_1$ will be specified below.
 	
 	Next we turn to the second interval, a similar derivation leads to  
 	\begin{equation}
 	\begin{split}
 	\mathbb{E} (G(\theta^s))-G(\hat{\theta})\leq & \alpha^{s-S_1} \mathbb{E}(G(\theta^{S_1})-G(\hat{\theta}))+\frac{4\chi\tau_\sigma}{ (1-\alpha)}  (\delta_{stat}^2+\delta_1^2)\\
 	\leq &\alpha^{s-S_1} \mathbb{E}(G(\theta^{S_1})-G(\hat{\theta}))+\frac{8 \chi\tau_\sigma}{ (1-\alpha)}  \delta_1^2,
 	\end{split}
 	\end{equation}
 	where $\delta_1=\frac{2\epsilon_1'}{\lambda L_g}=\frac{2k_1\epsilon_1}{\lambda L_g}.$
 	
 	We need
 	$$\frac{8 \chi}{ (1-\alpha)} \tau_\sigma \delta_1^2\leq\frac{\epsilon_1}{8}. $$
 	So it suffices to choose
 	$$ k_1=\sqrt{\frac{(1-\alpha) (\lambda L_g)^2 }{256 \chi\tau_\sigma \epsilon_1}}. $$
 	Then we can choose $S_2-S_1=\lceil \log 8/\log(1/\alpha) \rceil$, such that
 	$$ \mathbb{E}(G(\theta^{S_2})-G(\hat{\theta}) )\leq \frac{\epsilon_1}{8}+\frac{\epsilon_1}{8}\leq \frac{\epsilon_1}{4}=\epsilon_2. $$
 	Now we analyze the $i+1^{th}$ time interval, since $\mathbb{E}G(\theta^{S_i})-G(\hat{\theta})\leq \epsilon_i$, we have
 	\begin{equation}
 	\begin{split}
 	\mathbb{E}G(\theta^{s}) -G(\hat{\theta}) & \leq \alpha^{s-S_i} \mathbb{E}(G(\theta^{S_i}) -G(\hat{\theta}))+ \frac{ 4\chi\tau_\sigma}{ (1-\alpha)} \tau_\sigma (\delta^2_{stat}+\delta_i^2) \\
 	& \leq \alpha^{s-S_i} \mathbb{E}(G(\theta^{S_i}) -G(\hat{\theta}))+ \frac{8\chi\tau_\sigma}{ (1-\alpha)} \tau_\sigma \max(\delta^2_{stat},\delta_i^2)\\
 	&\leq  \alpha^{s-S_i} \mathbb{E}(G(\theta^{S_i}) -G(\hat{\theta}))+ \frac{8 \chi\tau_\sigma}{ (1-\alpha)} \tau_\sigma \delta_i^2
 	\end{split}
 	\end{equation} 
 	with probability $1-\frac{1}{k_i}$, where we choose $k_i=2^{i-1}k_1  $, and $\delta_i=2\min \{\epsilon_i'/(\lambda L_g),\rho\}.$ 

 	We need following condition to hold
 	$$\frac{8 \chi}{ (1-\alpha)\tau_\sigma} \tau_\sigma \delta_i^2=\frac{8 \chi}{ (1-\alpha)} \tau_\sigma (2k_i\epsilon_i/\lambda)^2\leq \frac{\epsilon_i}{8} $$
 	which is equivalent to
 	$$ \frac{32\chi k_i^2\epsilon_i\tau_\sigma}{(\lambda L_g)^2 (1-\alpha) }\leq \frac{1}{8} .$$
 	Since $k_i=2k_{i-1}$ and $\epsilon_i=\frac{1}{4}\epsilon_{i-1} $, the inequality holds when $ \frac{32\chi k_1^2\epsilon_1\tau_\sigma}{(\lambda L_g)^2 (1-\alpha) }\leq \frac{1}{8} $, which
 	is satisfied by our choice of  $ k_1=\sqrt{\frac{(1-\alpha) (\lambda L_g) ^2 }{256 \chi\tau_\sigma \epsilon_1}}. $
 	
 	Thus we set  $S_{i+1}-S_i=\log (8)/\log (1/\alpha) $, such that
 	$$ \mathbb{E} G(\theta^{S_{i+1}})-G(\hat{\theta})\leq \frac{\epsilon_i}{4}=\epsilon_i ' .$$ 
 	Since we have $ \epsilon_{i+1}/\epsilon_{i}=4$, so the total number of steps to achieve the tolerance is,
 	$$S_\kappa =\lceil \frac{\log 8}{\log (1/\alpha)}\rceil  \log_4 (\frac{G(\theta^0)-G(\hat{\theta})}{\kappa^2})+ \lceil   \log(\frac{2 (G(\theta^{0})-G(\hat{\theta}))}{\frac{8\chi}{ (1-\alpha)} \tau_\sigma (\delta_{stat}^2+4\rho^2)}/\log( 1/\alpha))         \rceil ,$$
 	
 	with probability at least $ 1- \frac{\log 8}{\log (1/\alpha)}\sum_{i=1}^{S_\kappa} \frac{1}{k_i}\geq 1-2\frac{\log 8}{k_1\log (1/\alpha)}$, where we use the fact $k_i=2^{i-1}k_1
 	$. Since we choose $m=2n$ and remind the assumption that $n>c\rho^2 \log p $, we know   $\frac{1}{k_1} \simeq \frac{c_1}{\sqrt{n}}$ for some constant $c_1$.
 \end{proof}

 \subsection{Proof of corollaries}
 We now prove the corollaries instantiating our main theorems to different statistical estimators. 
 \begin{proof}[Proof of Corollary on Lasso]
 	We begin the proof, by presenting the below lemma of the RSC,  proved in \cite{raskutti2010restricted}, and we then use it in the case of Lasso.
 	\begin{lemma}
 		if each data point $x_i$ is i.i.d.\ random sampled from the distribution $N(0,\Sigma) $, then there are some universal constants $c_0$ and $c_1$ such that
 		$$ \frac{\|X\Delta\|_2^2}{n}\geq \frac{1}{2}\|\Sigma^{1/2}\Delta\|_2^2-c_1\nu(\Sigma)\frac{\log p}{n} \|\Delta\|_1^2, \quad \mbox{for all } \Delta \in \mathbb{R}^p ,$$
 		with probability at least $1-\exp(-c_0n)$. Here, $X$ is the data matrix where each row is data point $x_i $.
 	\end{lemma}
 	
 	Since $\theta^*$ is support on a subset $S$ with cardinality r,  we choose $$ \bar{M}(S):=\{ \theta\in \mathbb{R}^p | \theta_j=0 \text{~for all ~} j\notin S   \}. $$ It is straightforward to choose $M(S)=\bar{M}(S)$ and notice that $\theta^*\in M(S)$.
 	In Lasso formulation, $F(\theta)=\frac{1}{2n}\|y-X\theta\|_2^2$, and hence it is easy to verify that
 	$$ F(\theta+\Delta)-F(\theta)-\langle \nabla F(\theta),\Delta \rangle\geq \frac{1}{2n} \|X\Delta\|_2^2\geq \frac{1}{4}\|\Sigma^{1/2}\Delta\|_2^2-\frac{c_1}{2}\nu(\Sigma)\frac{\log p}{n} \|\Delta\|_1^2  .$$
 	Also, $\psi(\cdot)$ is $\|\cdot\|_1$ in Lasso, and hence $H(\bar{M})=\sup_{\theta\in \bar{M}\backslash \{0\}} \frac{\|\theta\|_1}{\|\theta\|_2}=\sqrt{r}$. Thus we have 
 	$$\bar{\sigma}=\frac{1}{2}\sigma_{\min} (\Sigma)-64c_1 \nu(\Sigma)\frac{ r\log p}{n} .$$

 	On the other hand,	the statistical tolerance is 	
 	\begin{equation}
 	\begin{split}
 	e^2&=\frac{8\tau_\sigma}{Q(\beta,\bar{\sigma,L,\beta})}(8H(\bar{M})\|\Delta^*\|_2+8\psi(\theta^*_{M^\perp}))^2\\
 	&=\frac{c_2}{Q(\beta,\bar{\sigma,L,\beta})}\frac{r\log p}{n}\|\Delta^*\|_2^2,
 	\end{split}
 	\end{equation}
 	where we use the fact that $\theta^*\in M(S)$, which implies  $\psi(\theta^*_{M^{\perp}})=0$.\\
 	
 	Recall we require $\lambda$ to satisfy $\lambda\geq 2\psi^*(\nabla F(\theta^*))$. In Lasso we have $\psi^*(\cdot)=\|\cdot\|_\infty$. Using the fact that $y_i=x_i^T\theta^*+\xi$, we have $ \lambda\geq \frac{2}{n}\|X^T\xi\|_\infty $. Then we apply the tail bound on the Gaussian variable and use union bound to obtain that
 	$$\frac{2}{n}\|X^T\xi\|_\infty\leq 6u\sqrt{\frac{\log p}{n}} $$ holds
 	with probability at least $1-\exp(-3\log p)$.
 \end{proof}
 
 \begin{proof}[Proof of Corollary on Group Lasso]
 	We use the following fact on the RSC condition of Group Lasso \cite{negahban2009unified}\cite{negahban2012supplement}: if each data point $x_i$ is i.i.d.\ randomly sampled from the distribution $N(0,\Sigma) $, then  there exists strictly positive constant  $(\kappa_1,\kappa_2)$ which only depends on $\Sigma$ such that , 
 	$$ \frac{ \|X\Delta\|_2^2}{n} \geq \kappa_1(\Sigma) \|\Delta\|^2_2-\kappa_2(\Sigma) (\sqrt{\frac{q}{n}}+\sqrt{\frac{3 \log N_{\mathcal{G}}}{n}} )^2 \|\Delta\|^2_{\mathcal{G},2}, \quad \mbox{for all } \Delta \in \mathbb{R}^p ,$$
 	with probability at least $1-c_3\exp (-c_4n).$ 
 	
 	Remind we define the subspace
 	$$  \bar{M}(S_{\mathcal{G}})= M(S_{\mathcal{G}})=\{\theta|\theta_{G_i}=0 \text{~for all~} i\notin S_{\mathcal{G}}\} $$
 	where $S_{\mathcal{G}}$ corresponds to non-zero group of $\theta^*$.
 	
 	The subspace compatibility can be computed by $$ H(\bar{M})=\sup_{\theta\in \bar{M}\backslash \{0\}} \frac{\|\theta\|_{\mathcal{G},2}}{\|\theta\|_2}=\sqrt{s_{\mathcal{G}}}.$$
 	Thus, the modified RSC parameter 
 	$$\bar{\sigma}=\kappa_1(\Sigma)-c\kappa_2(\Sigma)s_{\mathcal{G}} (\sqrt{\frac{q}{n}}+\sqrt{\frac{3 \log N_{\mathcal{G}}}{n}} )^2  .$$
 	
 	We then bound the value of $\lambda$. 
 	As the regularizer in Group Lasso is $\ell_{1,2}$ grouped norm, its dual norm  is $(\infty,2)$ grouped norm.
 	So it suffices to have any $\lambda$ such that $$\lambda\geq 2 \max_{i=1,...,N_{\mathcal{G}}} \|\frac{1}{n}(X^T\xi)_{G_i}\|_2.$$
 	
 	Using Lemma 5 in \cite{negahban2009unified}, we know 
 	$$\max_{i=1,...,N_{\mathcal{G}}} \|\frac{1}{n}(X^T\xi)_{G_i}\|_2\leq 2\mu (\sqrt{\frac{q}{n}}+\sqrt{\frac{\log N_{\mathcal{G}}}{n}}) $$
 	with probability at least $1-2\exp (-2\log N_{\mathcal{G}})$.
 	Thus it suffices to choose $\lambda= 4\mu (\sqrt{\frac{q}{n}}+\sqrt{\frac{\log N_{\mathcal{G}}}{n}})$.
 	
 	%

 	The statistical tolerance is given by,
 	\begin{equation}
 	\begin{split}
 	e^2&=\frac{8\tau_\sigma}{Q(\beta,\bar{\sigma,L,\beta})}(8H(\bar{M})\|\Delta^*\|_2+8\psi(\theta^*_{M^\perp}))^2\\
 	&=\frac{c_2 \kappa_2(\Sigma)}{Q(\beta,\bar{\sigma,L,\beta})}s_{\mathcal{G}} (\sqrt{\frac{q}{n}}+\sqrt{\frac{3 \log N_{\mathcal{G}}}{n}} )^2   \|\Delta^*\|_2^2,
 	\end{split}
 	\end{equation}
 	where we use the fact $\psi(\theta^*_{M^\perp})=0$.
 \end{proof}

 \begin{proof}[Proof of Corollary on  SCAD]
 	
 	The proof is very similar to that of Lasso. In the proof of results for Lasso, we established
 	
 	$$ \|\nabla F(\theta^*)\|_\infty=\frac{1}{n}\|X^T\xi \|_\infty\leq 3 u \sqrt{\frac{\log p}{n}} $$ 
 	and the RSC condition 
 	$$ \frac{\|X\Delta\|_2^2}{n}\geq \frac{1}{2}\|\Sigma^{1/2}\Delta\|_2^2-c_1\nu(\Sigma)\frac{\log p}{n} \|\Delta\|_1^2.$$
 	Recall that $\mu=\frac{1}{\zeta-1}$ and $L_g=1$, we establish the corollary. 
 \end{proof}

 \begin{proof}[Proof of corollary on Corrected Lasso]
 	First notice
 	$$ \|\nabla F(\theta^*)\|_\infty=\|\hat{\Gamma}\theta^*-\hat{\gamma}\|_{\infty}=\|\hat{\gamma}-\Sigma\theta^* +(\Sigma-\hat{\Gamma})\theta^*\|_{\infty} \leq \|\hat{\gamma}-\Sigma\theta^*\|_\infty+\|(\Sigma-\hat{\Gamma})\theta^*\|_{\infty}.$$
 	As shown in literature (Lemma 2 in \cite{loh2011high}), both terms on the right hand side can be bounded by
 	$c_1 \varphi \sqrt{\frac{\log p}{n}}$, where $\varphi\triangleq(\sqrt{\sigma_{\max} (\Sigma)}+\sqrt{\gamma_w}) (v+\sqrt{\gamma_w} \|\theta^*\|_2)$, with high probability. 
 	
 	To obtain the RSC condition,  we apply Lemma 12 in \cite{loh2011high}, to get
 	$$ \frac{1}{n}\Delta^T\hat{\Gamma}\Delta \geq \frac{\sigma_{\min} (\Sigma)}{2} \|\Delta\|_2^2-\frac{c\log p}{n}\|\Delta\|_1^2,$$
 	with high probability.
 	
 	Combine these together, we establish the corollary.
 \end{proof}

\bibliography{SVRG}
\bibliographystyle{plainnat}

\end{document}